\documentclass[journal]{IEEEtran}

\usepackage{amsmath,amscd}
\usepackage{bbm}

\usepackage{dsfont}
\usepackage{bm}
\usepackage{amsthm,array}
\usepackage{amsmath}
\usepackage{amsfonts,latexsym}
\usepackage{graphicx,wrapfig}
\usepackage{subfig}
\usepackage{times}
\usepackage{psfrag,epsfig}
\usepackage{verbatim}
\usepackage{tabularx}
\usepackage[hidelinks]{hyperref}
\usepackage{color}
\usepackage{soul}
\usepackage{indentfirst}
\usepackage{inputenc}
\usepackage[capitalise]{cleveref}
\usepackage{courier}
\usepackage{authblk}
\usepackage{mathtools}
\usepackage{mathtools}
\usepackage{cite}
\usepackage{amssymb}
\usepackage{graphicx}
\usepackage{caption}
\usepackage{dblfloatfix}
\usepackage{lineno}
\usepackage{multirow}
\usepackage{algorithm,algpseudocode}
\usepackage{ltablex}
\usepackage{relsize}

\usepackage{enumitem}
\usepackage{xtab}
\usepackage{xfrac}
\usepackage{amssymb}             
\usepackage{soul}
\usepackage{flushend}

\xentrystretch{-0.07}

\DeclareMathOperator*{\rank}{rank}
\DeclareMathOperator*{\trace}{\mathrm{trace}}

\makeatletter
\def\mathcenterto#1#2{\mathclap{\phantom{#1}\mathclap{#2}}\phantom{#1}}
\let\old@widetilde\widetilde
\def\widetildeto#1#2{\mathcenterto{#2}{\old@widetilde{\mathcenterto{#1}{#2\,}}}}
\let\old@widehat\widehat
\def\widehatto#1#2{\mathcenterto{#2}{\old@widehat{\mathcenterto{#1}{#2\,}}}}
\def\mathcenterto#1#2{\mathclap{\phantom{#1}\mathclap{#2}}\phantom{#1}}
\let\old@underline\underline
\def\underlineto#1#2{\mathcenterto{#2}{\old@underline{\mathcenterto{#1}{#2\,}}}}
\makeatother

\def\0{\boldsymbol{0}}
\def\1{\boldsymbol{1}}
\def\a{\boldsymbol{a}}

\def\D{\boldsymbol{\text{D}}}

\def\S{\mathbb{S}}

\def\I{\mathbf{I}}
\def\Zero{\mathbf{O}}
\def\d{\text{d}}

\def\i{\mathbf{i}}

\def\rot{z} 
\def\tran{c} 
\def\lmk{s} 
\def\tM{\widetildeto{Q}{M}}

\def\transpose{\top} 

\def\R{\mathbb{R}}

\def\C{\mathbb{C}}

\def\H{\mathbb{H}}
\def\P{\mathbb{P}}

\def\VV{\mathcal{V}}

\def\se3{\mathfrak{se}(3)}

\def\diag{\mathrm{diag}}
\def\ddiag{\mathrm{ddiag}}

\def\ux{{\underline{x}}}
\def\tx{{\tilde{x}}}
\def\hx{{\hat{x}}}

\def\EE{\mathcal{E}}
\def\aEE{\overrightarrow{\EE}}
\def\aEEp{\overrightarrow{\EE'}}
\def\aGG{\overrightarrow{G}}

\long\def\answer#1{}

\long\def\comment#1{}

\newcommand{\ea}{{et al.~}}

\newcommand{\innprod}[2]{\langle {#1},{#2}\rangle}

\newcommand{\ssf}[1]{{\small \sf #1}}

\newcommand{\tentimes}[2]{{#1}\times 10^{#2}}

\theoremstyle{definition}

\newtheorem{prop}{Proposition}

\newtheorem{lemma}{Lemma}

\crefname{prop}{Proposition}{Propositions}

\theoremstyle{remark}
\newtheorem*{remark*}{{Remark}}

\newcolumntype{P}[1]{>{\centering\arraybackslash}p{#1}}
\newcolumntype{M}[1]{>{\centering\arraybackslash}m{#1}}

\newcommand{\RNum}[1]{\uppercase\expandafter{\romannumeral #1\relax}}

\hypersetup{
    bookmarks=true,         
    unicode=false,          
    pdftoolbar=true,        
    pdfmenubar=true,        
    pdffitwindow=false,     
    pdfstartview={FitH},    
    pdftitle={CPL-SLAM: Efficient and Certifiably Correct Planar Graph-Based SLAM Using the Complex Number Representation},    
    pdfauthor={Taosha Fan},     
    pdfsubject={robotics},   
    pdfcreator={Taosha Fan},   
    pdfproducer={Taosha Fan}, 
    pdfkeywords={SLAM} {optimization} {robotics}, 
    pdfnewwindow=true,      
    colorlinks=true,       
    linkcolor=blue,          
    citecolor=red,        
    filecolor=magenta,      
    urlcolor=cyan           
}

\IEEEoverridecommandlockouts                              

\title{\LARGE \bf  CPL-SLAM: Efficient and Certifiably Correct Planar Graph-Based SLAM Using the Complex Number Representation}

\author{Taosha Fan, Hanlin Wang, Michael Rubenstein and Todd Murphey
	\thanks{Taosha Fan and Todd Murphey are with the Department of Mechanical Engineering, Northwestern University, Evanston, IL 60201, USA. E-mail: {\tt taosha.fan@u.northwestern.edu, t-murphey@northwestern.edu}\\
	\indent Hanlin Wang is with the Department of  Computer Science, Northwestern Univerity, Evanston, IL 60201, USA. E-mail: {\tt hanlinwang@u.northwestern.edu}\\
	\indent Michael Rubenstein is with the Departments of Computer Science and Mechanical Engineering, Northwestern Univerity, Evanston, IL 60201, USA. E-mail: {\tt rubenstein@northwestern.edu}
}
}

\begin{document}
\maketitle
\thispagestyle{empty}
\pagestyle{empty}

\begin{abstract}
In this paper, we consider the problem of planar graph-based simultaneous localization and mapping (SLAM) that involves both poses of the autonomous agent and positions of observed landmarks. We present CPL-SLAM, an efficient and certifiably correct algorithm to solve planar graph-based SLAM using the complex number representation. We formulate and simplify planar graph-based SLAM as the maximum likelihood estimation (MLE) on the product of unit complex numbers, and relax this nonconvex quadratic complex optimization problem to convex complex semidefinite programming (SDP). Furthermore, we simplify the corresponding complex semidefinite programming to Riemannian staircase optimization (RSO) on the complex oblique manifold that can be solved with the Riemannian trust region (RTR) method. In addition, we prove that the SDP relaxation and RSO simplification are tight as long as the noise magnitude is below a certain threshold. The efficacy of this work is validated through  applications of CPL-SLAM and comparisons with existing state-of-the-art methods on planar graph-based SLAM, which indicates that our proposed algorithm is capable of solving planar graph-based SLAM  certifiably, and is more efficient in numerical computation and more robust to measurement noise than existing state-of-the-art methods. The C++ code for CPL-SLAM is available at \url{https://github.com/MurpheyLab/CPL-SLAM}.
\end{abstract}

\section{Introduction}
Simultaneous localization and mapping (SLAM) estimates poses of an autonomous agent and positions of observed landmarks from noisy measurements \cite{Cadena16tro-SLAMsurvey,thrun2005probabilistic,taketomi2017visual}.  For an autonomous agent, the ability to construct a map of the environment and concurrently estimate its location within the map is essential to navigation and exploration in unknown scenarios, such as autonomous driving \cite{geiger2012we}, disaster response \cite{kleiner2006rfid}, underwater exploration \cite{kinsey2006survey}, precision agriculture \cite{dong2017agriculture}, floor plan building \cite{vysotska2016exploiting}, virtual and augmented reality \cite{polvi2016slidar}, to name a few. An intuitive way to formulate SLAM problems is to use a graph whose vertices are associated with either poses of the autonomous agent or positions of observed landmarks and whose edges are associated with the available measurements \cite{grisetti2010tutorial}. In graph-based SLAM, the estimation problem is usually addressed as a difficult nonconvex optimization problem that involves up to thousands of variables and constraints, and the procedure of solving the optimization problem greatly affects the overall performance of estimation. Even though a number of optimization methods have been developed \cite{bertsekas1999nonlinear,Cadena16tro-SLAMsurvey,grisetti2010tutorial}, it is generally NP-hard to solve a nonconvex optimization problem globally \cite{bertsekas1999nonlinear}, and it is common to get stuck at local minima in solving graph-based SLAM, which results in bad estimates.

In robotics, most graph-based SLAM techniques rely on local search methods for nonlinear optimization to  estimate poses of the autonomous agent and positions of observed landmarks. Lu and Milios \cite{lu1997globally} formulate SLAM as pose graph optimization (PGO) and use iterative nonlinear optimization methods to solve PGO. Duckett and Frese \ea \cite{duckett2002fast,frese2005multilevel} exploit the sparsity of graph-based SLAM and propose relaxations for the resulting nonlinear optimization problem. Olson \ea \cite{olson2006fast} propose a stochastic gradient descent on an alternative state space representation of graph-based SLAM that has good stability and scalability. Grisetti \ea \cite{grisetti2007tree} extend Olson's work by presenting a novel tree parametrization that improves the convergence of stochastic gradient descent. Fan and Murphey \cite{fan2019proximal} propose an accelerated proximal method for pose graph optimization. Dellaert and Kaess \ea \cite{dellaert2006square,kaess2008isam,kaess2012isam2,rosen2014rise} propose incremental smoothing algorithms that enable online updates of large-scale  graph-based SLAM with nonlinear optimization. Huang and Wang \ea \cite{huang2010far,wang2013structure} study the least-square structure of graph-based SLAM and indicate the possibility of reducing the nonlinearity and nonconvexity of SLAM. K{\"u}mmerle et al. \cite{kummerle2011g} propose $\mathrm{g^2o}$ framework that solves graph-based SLAM using Gauss-Newton method. Carlone \ea \cite{carlone2014angular,carlone2014fast} propose approximations for planar pose graph optimization that reduce the risks of getting stuck at local minima. Khosoussi \ea \cite{khosoussi2015exploiting} exploit the separable structure of SLAM problems using variable projection and propose algorithms to improve the efficiency of Gauss-Newton methods. However, all of the aforementioned nonlinear optimization techniques are local search methods, and as a result, there are no guarantees of the correctness for the resulting solutions.

To address the issues of local minima in nonlinear optimization, several efforts have been made to relax graph-based SLAM as convex optimization problems. Liu \ea \cite{liu2012convex} propose a suboptimal convex relaxation to solve SLAM
problems. Carlone \ea \cite{carlone2015duality,carlone2016planar,carlone2015lagrangian} propose a tight semidefinite
relaxation and analyze its optimality using Lagrangian duality. Briales \ea \cite{briales2016fast} present a fast method for the optimality verification of 3D PGO based on \cite{carlone2015lagrangian}. A further breakthrough of the semidefinite relaxation of PGO is made in \cite{rosen2016se} that results in a fast and certifiable algorithm for pose graph optimization. Rosen \ea propose SE-Sync to solve the semidefinite relaxation of PGO using Riemannian staircase optimization on the Stiefel manifold that is orders of magnitude faster than interior point methods \cite{rosen2016se}. Furthermore, it is shown in  \cite{rosen2016se} that the semidefinite relaxation of PGO is tight as long as the magnitude of the measurement noise is below a certain threshold. Similar to SE-Sync, Briales \ea \cite{briales2017cartan} propose Cartan-Sync that uses the Cartan motion group and introduce a novel preconditioner to accelerate the algorithm. Mangelson \ea \cite{mangelson2018guaranteed} formulate planar pose graph and landmark SLAM using sparse-bounded sum of squares programming that is guaranteed to find the globally optimal solution regardless of the noise level.

In fields other than robotics, problems such as angular and rotation synchronization that share a similar mathematical formulation with graph-based SLAM have been extensively studied. Singer \ea \cite{singer2011angular,singer2011three} propose semidefinite relaxations to solve angular and rotation synchronization by finding the eigenvectors that correspond to the greatest eigenvalues. Bandeira \ea \cite{bandeira2017tightness} prove the tightness of the semidefinite relaxation of angular synchronization and show that the Riemannian staircase optimization is significantly more scalable to solve the resulting problem. Boumal \cite{boumal2016nonconvex} proposes the generalized power method that can recover the globally optimal solution to angular synchronization. Eriksson \ea \cite{eriksson2018rotation} explore the role of strong duality in rotation averaging, which has important applications in computer vision.

In applied mathematics, it is common to use unit complex numbers in synchronization problems over $SO(2)$ \cite{bandeira2017tightness,boumal2016nonconvex}. In robotics, it is not new to use the complex number representation in planar robot localization and mapping problems, either. Betke \ea \cite{betke1997mobile} use the complex number to represent positions of landmarks to localize a mobile robot with bearing measurements. Carlone \ea \cite{carlone2016planar} use the complex number representation of $SO(2)$ and $SE(2)$ to verify the optimality of planar PGO and the tightness of semidefinite relaxations, and the analysis is much clearer and simpler than that using the matrix representation, and to our knowledge, this is the first implementation of the complex number representation in planar PGO.

In general, a certifiably correct algorithm for an optimization problem not only finds a solution to the problem, but also is capable of certifying the global optimality of the resulting solution \cite{bandeira2016note}. For many estimation problems, it is usually intractable to attain a globally optimal solution and we have to either solve these problems using local search methods, or relax them to a more reasonable formulation. As a result, the certifiable correctness of the algorithm is important for estimation problems in which globally optimal solutions are preferred. Even though a number of optimization methods are proposed to planar graph-based SLAM, to our knowledge, only \cite{rosen2016se,carlone2015lagrangian,carlone2016planar,mangelson2018guaranteed,carlone2015duality,briales2017cartan,briales2016fast} are certifiably correct.

In this paper, we consider the problem of planar graph-based SLAM that involves both poses of the autonomous agent and positions of observed landmarks. We present CPL-SLAM, which means the \textbf{C}om\textbf{PL}ex number \textbf{S}imultaneous \textbf{L}ocalization \textbf{A}nd \textbf{M}apping, an efficient and certifiably correct algorithm to solve planar graph-based SLAM using the complex number representation.

This paper is built upon the works of \cite{rosen2016se,bandeira2017tightness,carlone2016planar} that use the complex number representation, the semidefinite relaxation and the Riemannian staircase optimization \cite{boumal2016non} to efficiently and certifiably correctly solve large-scale estimation problems. In \cite{carlone2016planar}, Carlone \ea were first to formulate planar PGO using the complex number representation of $SE(2)$, in which the optimality and tightness of the semidefinite relaxation are studied; in CPL-SLAM, we use the same representation of $SE(2)$ as the one in \cite{carlone2016planar}. In \cite{rosen2016se}, Rosen \ea analyze the optimality and tightness of the semidefinite relaxation of PGO and introduce the Riemannian staircase optimization to solve the semidefinite relaxation, which, though using the matrix representation, motivates this paper. In \cite{bandeira2017tightness}, Bandeira \ea prove the tightness of semidefinite relaxation of phase synchronization on $SO(2)$ using the complex number representation, which is helpful to our theoretical analysis of  CPL-SLAM.

In graph-based SLAM, poses are special Euclidean groups $SE(d)$ \cite{grisetti2010tutorial,rosen2016se,rosen2014rise,kummerle2011g,briales2017cartan}, which are isomorphic to a semidirect product of real space $\R^d$ and special orthogonal groups $SO(d)$ \cite{chirikjian2011stochastic}. In general, it is possible to identify $SE(d)$ as a pair $(t,\,R)$, in which the translation is represented as a real vector $t$ in $\R^d$ and the rotation is represented as a matrix $R$ in $SO(d)$, and such an identification results in the matrix representation of $SE(d)$ that is commonly used in robotics. However, for planar graph-based SLAM, the matrix representation of $SO(2)$ is redundant, which needs four real numbers, whereas a unit complex number that can be represented with two real numbers is sufficient to capture the topological and geometric structures of $SO(2)$ \cite{selig2004geometric,carlone2016planar}. Furthermore, as is later shown in this paper, the complex number representation of $SO(2)$ and $SE(2)$ brings significant analytical and computational benefits, and renders the resulting CPL-SLAM algorithm a lot more efficient in numerical computation and much more robust to measurement noise. As a result, the CPL-SLAM outperforms existing methods of planar graph-based SLAM in terms of both numerical scalability and theoretical guarantees.

In contrast to the state-of-the-art local search methods in \cite{lu1997globally,duckett2002fast,frese2005multilevel,olson2006fast,grisetti2007tree,dellaert2006square,kaess2008isam,kaess2012isam2,rosen2014rise,kummerle2011g,huang2010far,wang2013structure,carlone2014fast,carlone2014angular}, CPL-SLAM is a lot faster and capable of certifying the correctness of the solutions. As for \cite{rosen2016se,carlone2016planar,mangelson2018guaranteed} that also seek to use convex relaxation to certifiably correctly solve graph-based SLAM, CPL-SLAM has better scalability, {and more explicitly, CPL-SLAM is expected to be several orders of magnitude faster than \cite{carlone2016planar,mangelson2018guaranteed}} and several times faster than \cite{rosen2016se}. Moreover, \cite{carlone2016planar,rosen2016se} are designed for pose-only problems, whereas CPL-SLAM extends the works of \cite{carlone2016planar,rosen2016se} by accepting pose-landmark measurements. Even though \cite{carlone2016planar} uses complex semidefinite relaxation to verify the optimality and tightness of planar pose graph optimization, we present stronger, more complete and more concise theoretical results and more scalable algorithms to solve the complex semidefinite relaxation. Furthermore, the conciseness of the complex number representation makes the semidefinite relaxation in CPL-SLAM much tighter than that in \cite{rosen2016se} using the matrix representation, and thus, CPL-SLAM has greater robustness to measurement noise.

This paper extends the preliminary results of \cite{fan2019cpl_iros} in which we only use the complex number representation to solve planar pose graph optimization without landmarks. In this paper, our proposed CPL-SLAM uses the graph-based SLAM measurement model in \cite{fan2019cpl_iros,briales2017cartan,rosen2016se} and can handle planar graph-based SLAM with landmarks. Similar to \cite{fan2019cpl_iros,rosen2016se,briales2017cartan} for pose graph optimization, CPL-SLAM is a certifiable algorithm that is guaranteed to attain the globally optimal solution to planar graph-based SLAM with landmarks as long as the magnitude of measurement noise is below a certain threshold. Furthermore, even though it is not new to involve landmarks in graph-based SLAM \cite{grisetti2007tree,grisetti2010tutorial,dellaert2012factor}, we propose a novel preconditioner making better use of translation and landmark information in planar graph-based SLAM. As a result, the performance of the truncated conjugate gradient method is improved. In addition, we also provide additional proofs of lemmas and propositions, extensive experimental results on numerous datasets and much more detailed discussions, all of which are not covered in \cite{fan2019cpl_iros}.

In summary, the contributions of this paper are the  following:
\begin{enumerate}
	\item We formulate planar graph-based SLAM with poses of the autonomous agent and positions of observable landmarks using the complex number representation and simplify the resulting estimation problem as an optimization problem on the product of unit complex numbers.
	\item We relax the nonconvex optimization problem as complex semidefinite programming and prove that the complex semidefinite relaxation is tight as long as the magnitude of measurement noise is below a certain threshold.
	\item We recast the complex semidefinite programming as a series of rank-restricted complex semidefinite programming on complex oblique manifolds that can be efficiently solved with the Riemannian staircase optimization \cite{boumal2016non}, and it is almost guaranteed to retrieve the true solution to the complex semidefinite programming if the rank of the Riemannian staircase optimization is appropriately selected.
	\item The resulting CPL-SLAM algorithm is certifiably correct, and more importantly, a lot faster in numerical computation and much more robust to measurement noise than existing state-of-the-art methods \cite{lu1997globally,duckett2002fast,frese2005multilevel,olson2006fast,grisetti2007tree,dellaert2006square,kaess2008isam,kaess2012isam2,rosen2014rise,kummerle2011g,huang2010far,wang2013structure,carlone2014fast,carlone2014angular,rosen2016se,carlone2016planar,dellaert2012factor}.
\end{enumerate}

The rest of this paper is organized as follows. \cref{section:: notation} introduces notations that are used throughout this paper. \cref{section::complex} reviews the complex number representation of $SO(2)$ and $SE(2)$. \cref{section::formulation} formulates planar graph-based SLAM with poses of the autonomous agent and positions of observable landmarks using the complex representation and \cref{section::sdp} relaxes planar graph-based SLAM to complex semidefinite programming. \cref{section::cpl} presents the CPL-SLAM algorithm to solve planar graph-based SLAM. \cref{section::result} presents and discusses comparisons of CPL-SLAM with existing methods \cite{dellaert2012factor,rosen2014rise,rosen2016se} on a series of simulated \textsf{\small Tree} and \textsf{\small City} datasets and a suite of large 2D simulated and real-world SLAM benchmark datasets. The conclusions are made in \cref{section::conclusion}.

\section{Notation}\label{section:: notation}
$\R$ and $\C$ denote the sets of real and complex numbers, respectively; $\R^{m\times n}$ and $\C^{m\times n}$ denote the sets of $m\times n$ real and complex matrices, respectively; $\R^{n}$ and $\C^{n}$ denote the sets of $n\times 1$ real and complex vectors, respectively. $\C_1$ and $\C_1^n$ denote the sets of unit complex numbers and $n\times 1$ vectors over unit complex numbers, respectively. $\P$ denotes the group of $(\C,+)\rtimes (\C_1,\cdot)$ in  which ``$\rtimes$'' denotes the semidirect product of groups \cite{chirikjian2011stochastic} under complex number addition ``$+$'' and multiplication ``$\cdot$''. $\S^n$ and $\H^{n}$ denote the sets of $n\times n$ real symmetric matrices and complex Hermitian matrices, respectively. The notation ``$\i$'' is reserved for the imaginary unit of complex numbers. The notation \mbox{$|\cdot|$} denotes the absolute value of real and complex numbers, and the notation $\overline{(\cdot)}$ denote the conjugate of complex numbers. The superscripts $(\cdot)^\transpose $ and $(\cdot)^H$ denote the transpose and conjugate transpose of a matrix, respectively. For a complex matrix $W$, $[W]_{ij}$ denotes its $(i,j)$-th entry; the notations $\Re(W)$ and $\Im(W)$ denote real matrices such that $W=\Re(W)+\Im(W)\i$; $W\succcurlyeq 0$ means that $W$ is Hermitian and positive semidefinite; $\trace(W)$ denotes the trace of $W$; $\diag(W)$ extracts the diagonal of $W$ into a vector and $\ddiag(W)$ sets all off-diagonal entries of $W$ to zero; the notations $\|W\|_F$ and $\|W\|_2$ denote the Frobenius norm and the induced-$2$ norm, respectively. The notation $\innprod{\cdot}{\cdot}$ denotes the real inner product of matrices. For a vector $v$, the notation $[v]_i$ denotes its $i$-th entry; $\|v\|^2=\|v\|_2^2=\sqrt{\sum_{i}|[v]_i|^2}=\sqrt{v^Hv}$; the notation $\diag(v)$ denotes the diagonal matrix with $\big[\diag(v)\big]_{ii}=v_i$. The notation $\1\in\C^n$ denotes the vector of all-ones. The notation $\0\in\C^n$ denotes the vector of all-zeros. The notation $\I\in\C^{n\times n}$ denotes the identity matrix. The notation $\Zero\in\C^{n\times n}$ denotes the zero matrix. For a hidden parameter $x$ whose value we wish to infer, the notations $\ux$, $\tx$ and $\hx$ denote the true value of $x$, a noisy observation of $x$ and an estimate of $x$, respectively.

\section{The Complex Number Representation of $SO(2)$ and $SE(2)$}\label{section::complex}
In this section, we give a brief review of $SO(2)$ and $SE(2)$, and show that $SO(2)$ and $SE(2)$ can be represented using complex numbers. It should be noted that the complex number representation used in this paper, though presented in a different way, is in fact equivalent to that in \cite{carlone2016planar}.\par

It is known that the set of unit complex numbers $$\C_1\triangleq\{a_1+a_2\i\in \C|a_1^2+a_2^2=1\}$$ forms a group under complex number multiplication ``$\cdot$'' for which the identity is $1$ and the inverse is the conjugate, i.e., for $\rot,\rot'\in \C_1$, we obtain  \cite{selig2004geometric}
$$\rot\cdot \rot'\in \C_1,\quad 1\cdot \rot=\rot\cdot 1=\rot, \quad \rot\cdot\overline{\rot}=\overline{\rot}\cdot \rot=1.$$
In addition, the group of unit complex numbers $(\C_1, \cdot)$ is diffeomorphic and isomorphic to the matrix Lie group $SO(2)$:
$$ 
\begin{aligned}
SO(2)&\triangleq\{\begin{bmatrix}
a_1 & -a_2\\a_2 & a_1
\end{bmatrix}\in \R^{2\times 2}|a_1^2 + a_2^2=1 \}\\
&\triangleq\{R\in\R^{2\times2}| R^\transpose R=\I, \,\det(R)=1\}
\end{aligned}
$$
under matrix multiplication. As a result, $SO(2)$ can be represented using unit complex numbers $\C_1$. More explicitly, if $R\in SO(2)$ is
\begin{equation}\label{eq::R}
R=\begin{bmatrix}
a_1 & -a_2\\
a_2 & a_1\\
\end{bmatrix} =\begin{bmatrix}
\cos\theta & -\sin\theta\\
\sin\theta & \cos\theta
\end{bmatrix},
\end{equation}
the corresponding unit complex number representation $\rot\in \C_1$ is
\begin{equation}\label{eq::multi}
\rot =a_1+a_2\i=e^{\i\theta} =\cos\theta + \sin\theta \i
\end{equation}
in which $e^{\i\theta }=\cos\theta + \sin\theta \i.$ Furthermore, if $b'=\begin{bmatrix}
b'_1 & b'_2
\end{bmatrix}^\transpose \in \R^2$ is rotated by  $R\in SO(2)$ in \cref{eq::R} from $b=\begin{bmatrix}
b_1 & b_2
\end{bmatrix}^\transpose \in \R^2 $, i.e.,
$$b'_1=a_1b_1-a_2b_2=b_1\cos\theta -b_2\sin\theta , $$
$$b'_2=a_1b_2+a_2b_1=b_2\cos\theta +b_1\sin\theta ,$$
we obtain
\begin{subequations}\label{eq::trans}
	\begin{equation}
	\lmk'= \rot\cdot\lmk=\underbrace{a_1b_1-a_2b_2}_{b_1'}+(\underbrace{a_1b_2+a_2b_1}_{b_2'})\i,
	\end{equation}
	or equivalently,
	\begin{equation}
	\begin{aligned}
	\lmk'&= \rot\cdot\lmk= e^{\i\theta}\cdot \lmk\\
			  &=\underbrace{b_1\cos\theta -b_2\sin\theta}_{b_1'}+(\underbrace{b_2\cos\theta +b_1\sin\theta}_{b_2'})\i,
	\end{aligned}
	\end{equation}
\end{subequations}
in which $\rot$ is a unit complex number as that given in \cref{eq::multi}, and
\begin{equation}\label{eq::b}
\lmk=b_1+b_2\i\quad\mathrm{and}\quad\lmk'=b'_1+b'_2\i
\end{equation}
are the complex number representation of $b$ and $b'$, respectively. As a result, rotating a vector can also be described using the complex number representation.

In general, the special Euclidean group $SE(2)$ is the matrix Lie group
\begin{equation}\label{eq::semi}
SE(2)\triangleq\{\begin{bmatrix}
R & t\\
\0 & 1
\end{bmatrix}\in \R^{3\times 3}| R\in SO(2),\, t\in \R^2\},
\end{equation}
whose group multiplication is matrix multiplication. In terms of group theory, $SE(2)$ is also represented as the semidirect product of $(\R^2, +)$ and $SO(2)$:
$$SE(2)\triangleq (\R^2,+)\rtimes SO(2),$$
in which ``$\rtimes$'' denotes the semidirect product of groups under vector addition and matrix multiplication \cite{chirikjian2011stochastic}  and whose group multiplication ``$\circ$'' using the matrix representation of \cref{eq::semi} is defined to be
\begin{equation}\label{eq::gg}
g\circ g' = (Rt'+t, RR'),
\end{equation}
in which $g=(t, R),\,g'=(t',R')\in SE(2)$.\footnote{In group theory, the definition of the semidirect product relies on the choice of the group multiplication rule.} Following the complex number representation of $SO(2)$ and $\R^2$, the representation of $SE(2)$ as \cref{eq::semi} is diffeomorphic and isomorphic to the semidirect product of $(\C,+)$ and $(\C_1,\cdot)$: 
$$\P\triangleq(\C,+)\rtimes (\C_1,\cdot),$$
whose group multiplication ``$\odot$'' is defined to be
\begin{equation}\label{eq::qmulti}
\rho\odot \rho' = (\rot\cdot\tran'+\tran,\rot\cdot \rot')\in \P,
\end{equation}
in which $\rho=(\tran,\rot)$, $\rho'=(\tran', \rot')\in \P$. In \cref{eq::qmulti}, $\rot,\,\rot'\in\C_1$ and $\tran,\,\tran'\in\C$ are the complex number representation of $R,\,R'\in SO(2)$ and $t,\,t'\in\R^2$, respectively, which follow the same representation as that in \cref{eq::multi,eq::b}. It is obvious from \cref{eq::multi,eq::trans} that the group multiplication of $\rho\odot \rho'$ in \cref{eq::qmulti} is equivalent to that of $g\circ g'$ in \cref{eq::gg}.  Furthermore, the identity of $\P$ is $(0,1)\in \P$ and the inverse of $\rho=(\tran,\rot)\in \P$ is
\begin{equation}\label{eq::qinv}
\rho^{-1}=(-\overline{\rot}\cdot\tran,\overline{\rot})\in \P.
\end{equation}
As a result, instead of using the matrix representation, we represent $SE(2)$ with a 2-tuple of complex numbers. In addition, if $b'\in\R^2$ is transformed by $g\in SE(2)$ from $b\in \R^2$, we obtain
\begin{equation}
\nonumber
	\lmk' = \rot\cdot\lmk + \tran, 
\end{equation}
in which $\rho=(\tran,\,\rot)\in \P$ is the complex number representation of $g\in SE(2)$, and $\lmk$ and $\lmk'$ are the complex number representation of $b$ and $b'$, respectively.

For notational convenience, in the rest of paper, we will omit the complex number multiplication ``$\cdot$'' if there is no ambiguity. 

In terms of the computation of group multiplication and transformation only, the complex number representation of $SO(2)$ and $SE(2)$ has the same complexity as the matrix representation. In spite of this, as shown in the following sections, the complex number representation greatly simplifies the analysis for planar graph-based SLAM, and most importantly, the semidefinite relaxation and Riemannian optimization of planar graph-based SLAM using the complex number representation is simpler for problem formulation, more efficient in numerical computation and more robust to measurement noise than that using the matrix representation in \cite{rosen2016se}.

In the following sections, we will use the complex number representation of $SO(2)$ and $SE(2)$ to formulate and solve planar graph-based SLAM.
\section{Problem Formulation and Simplification}\label{section::formulation}
In this section, we formulate planar graph-based SLAM as maximum likelihood estimation, and further simplify it to complex quadratic programming on the product of unit complex numbers. 
\subsection{Problem Formulation}
Planar graph-based SLAM consists of estimating $n$ unknown poses $g_1$, $g_2$, $\cdots$, $g_n\in SE(2)$, in which $g_{(\cdot)}=(t_{(\cdot)}, R_{(\cdot)})$ with $t_{(\cdot)}\in\R^2$ and $R_{(\cdot)}\in SO(2)$, and $n'$ landmark positions $l_1$, $l_2$, $\cdots$, $l_{n'}\in \R^2$ given $m$ noisy pose-pose measurements $\tilde{g}_{ij}\in SE(2)$ of
\begin{equation}
\nonumber
g_{ij}\triangleq g_i^{-1}g_j\in SE(2)
\end{equation}
and $m'$ noisy pose-landmark measurements $\tilde{l}_{ij}\in \R^2$ of
\begin{equation}
\quad l_{ij}\triangleq R_{i}^\transpose(l_j-t_i)\in\R^2.
\end{equation}
From \cref{section::complex}, the problem is equivalent to estimating $n$ 2-tuples of complex numbers $\rho_{1}$, $\rho_2$, $\cdots$, $\rho_n\in \P$, in which $\rho_{(\cdot)}=(\tran_{(\cdot)},\rot_{(\cdot)})\in\P$ with $\tran_{(\cdot)}\in \C$ and $\rot_{(\cdot)}\in \C_1$, and $n'$ complex numbers $\lmk_{1},\,\lmk_{2},\,\cdots,\, \lmk_{n'}\in \C$ given $m$ noisy pose-pose measurements $\tilde{\rho}_{ij}\in\P$ of
$$\rho_{ij}\triangleq \rho_i^{-1}\odot \rho_j\in \P$$ 
and $m'$ noisy pose-landmark measurements $\tilde{\lmk}_{ij}\in\C$ of
$$\quad\lmk_{ij}\triangleq \overline{\rot}_i(s_j-\tran_i)\in\C.$$
The unknown $n$ poses and $n'$ landmark positions and the noisy relative measurements can be described with a directed graph $\overrightarrow{G}=(\VV\cup\VV',\aEE\cup\aEEp)$ in which $i\in\VV\triangleq\{1,\cdots,n\}$ is associated with $g_i$ or $\rho_i$, and $i\in\VV'=\{1,\cdots,n'\}$ is associated with $l_i$ or $\lmk_i$, and $(i,j)\in \aEE\subset\VV\times\VV$ if and only if the pose-pose measurement $\tilde{g}_{ij}$ or $\tilde{\rho}_{ij}$ exists, and $(i,j)\in \aEEp\subset\VV\times\VV'$ if and only if the pose-landmark measurement $\tilde{l}_{ij}$ or $\tilde{\lmk}_{ij}$ exists. If the orientation of edges in $\aEE$ and $\aEEp$ are ignored, we obtain the undirected graph of $\overrightarrow{G}$ that is denoted as $G=(\VV\cup\VV',\,\EE\cup\EE')$. In the rest of this paper, we assume that $\overrightarrow{G}$ is weakly connected and $G$ is (equivalently) connected. In addition, we assume that the noisy relative measurements $\tilde{\rho}_{ij}=(\tilde{\tran}_{ij},\tilde{\rot}_{ij})$ and $\tilde{\lmk}_{ij}$ are random variables that satisfy
\begin{subequations}\label{eq::observation}
\begin{align}
\label{eq::obt}
\tilde{\tran}_{ij}&=\underline{\tran}{}_{ij}+\tran^{\epsilon}_{ij} &\tran_{ij}^{\epsilon}\sim N(0,\tau_{ij}^{-1}),\\
\label{eq::obr}
\tilde{\rot}_{ij}&=\underline{\rot}_{ij}\rot_{ij}^{\epsilon}           &\quad\tilde{\rot}_{ij}^{\epsilon}\sim \mathrm{vMF}(1,\kappa_{ij}),\\
\label{eq::obl}
\tilde{\lmk}_{ij}&=\underline{\lmk}{}_{ij}+\lmk^{\epsilon}_{ij} &\lmk_{ij}^{\epsilon}\sim N(0,\nu_{ij}^{-1}),
\end{align}
\end{subequations}
for all $(i,j)\in \aEE\cup\aEEp$. In \cref{eq::observation}, $\underline{\rho}{}_{ij}=(\underline{\tran}{}_{ij},\underline{\rot}_{ij})$ and $\underline{\lmk}_{ij}$ are the true (latent) values of $\rho_{ij}$ and $\lmk_{ij}$, respectively, $N(\mu,\Sigma)$ denotes the complex normal distribution with mean $\mu\in\C$ and covariance $\Sigma\succcurlyeq 0$, and $\mathrm{vMF}(\rot_0,\kappa)$ denotes the von Mises-Fisher distribution on $\C_1$ with mode $\rot_0\in \C_1$, concentration number $\kappa\geq 0$ and the probability density function of $\mathrm{vMF}(\rot_0,\kappa)$ is \cite{khatri1977mises}
\begin{equation}
\nonumber
f(\rot;\rot_0,\kappa)=\frac{1}{c_d(\kappa)}\exp\left(\kappa(\overline{\rot}_0 \rot+\rot_0\overline{\rot})\right),
\end{equation}
in which $c_d(\kappa)$ is a function of $\kappa$.

If $\tilde{\tran}_{ij}$, $\tilde{\rot}_{ij}$ and $\tilde{\lmk}_{ij}$ are independent from each other, from \cref{eq::trans,eq::qmulti,eq::qinv}, a straightforward algebraic manipulation indicates that the maximum likelihood estimation (MLE) is a least square problem as follows
\begin{multline}\label{eq::LSP}
\tag{MLE}
\min_{\substack{\lmk_i\in \C,\\\tran_i\in\C^n,\,\rot_i\in \C_1^n}}\;\sum_{(i,j)\in\aEE}\left[ \kappa_{ij} |\rot_i\tilde{\rot}_{ij}-\rot_j|^2+ \tau_{ij}|\tran_j-\tran_i\right.\\\left.-\rot_i\tilde{\tran}_{ij}|^2\right]+\sum_{(i,j)\in\overrightarrow{\EE'}}\nu_{ij}|\lmk_j-\tran_i-\rot_i\tilde{\lmk}_{ij}|^2,
\end{multline}
in which $\kappa_{ij}$, $\tau_{ij}$ and $\nu_{ij}$ are as given in \cref{eq::obt,eq::obr,eq::obl}. From \cref{eq::R,eq::multi}, it should be noted that $$|\rot_i\tilde{\rot}_{ij}-\rot_j|^2=\frac{1}{2}\|R_i\widetilde{R}_{ij}-R_j\|_F^2,$$
and it is also trivial to show that
$$|\tran_j-\tran_i-\rot_i\tilde{\tran}_{ij}|^2=\|t_j-t_i-R_{i}\tilde{t}_{ij}\|_F^2,$$
$$|\lmk_j-\tran_i-\rot_i\tilde{\lmk}_{ij}|^2=\|l_j-t_i-R_{i}\tilde{l}_{ij}\|_F^2.$$
As a result, \eqref{eq::LSP} is equivalent to
\begin{multline}\label{eq::mLSP}
\tag{SE-MLE}
\min_{\substack{R_i\in SO(2),\\ p_i,\,l_i\in \R^2}}\;\sum_{(i,j)\in\overrightarrow{\mathcal{E}}} \Big[\frac{\kappa_{ij}}{2} \|R_i\widetilde{R}_{ij}-R_j\|_F^2+ \tau_{ij}\|t_j-t_i-\\R_i\tilde{t}_{ij}\|_F^2\Big]+\sum_{(i,j)\in\overrightarrow{\EE'}}\nu_{ij}\|l_j-t_i-R_il_{ij}\|_F^2.
\end{multline}
Even though there are landmarks present in \eqref{eq::LSP} and \eqref{eq::mLSP}, this does not create a significant distinction from \cite{rosen2016se} in terms of problem formulation. As a matter of fact, if there are no landmarks, i.e., $\VV'=\emptyset$ and $\overrightarrow{\EE'}=\emptyset$, \eqref{eq::mLSP} is almost the same as the formulation of pose graph optimization using the matrix representation in SE-Sync \cite{rosen2016se} except for the weight factors. In addition, \eqref{eq::mLSP} can also be constructed as a specialized case of SE-Sync's \cite{rosen2016se} measurement model if we interpret pose-landmark measurements as pose-pose measurements whose rotational weight factors are zero.

In the next subsection, we will simplify \eqref{eq::LSP} to quadratic programming on the product of unit complex numbers $\C_1^n$.

\subsection{Problem Simplification}

The simplification of \eqref{eq::LSP} is similar to that of \cite[Appendix B]{rosen2016se}, the difference of which is that ours uses the complex number representation while \cite{rosen2016se} uses the matrix representation  and ours has landmarks involved while \cite{rosen2016se} does not.

For notational convenience, we define $\rot_{ji}=\overline{\rot}_{ij}$, $\kappa_{ji}=\kappa_{ij}$ and $\tau_{ji}=\tau_{ij}$, and \eqref{eq::LSP} can be reformulated as
\begin{equation}\label{eq::P}
\tag{P}
\min_{\xi\in \C^{n'}\times\C^n \times \C_1^n}\;\xi^H \widetilde{\Gamma} \xi
\end{equation}
in which 
$$\xi\triangleq\begin{bmatrix}
\lmk_{1}& \cdots & \lmk_{n'} &\tran_1 & \cdots & \tran_n & \rot_1 & \cdots & \rot_n
\end{bmatrix}^\transpose .$$ 
In \eqref{eq::P}, $\widetilde{\Gamma}$ is a $(2n+n')$-by-$(2n+n')$ Hermitian matrix
\begin{equation}\label{eq::problem}
\widetilde{\Gamma}\triangleq\begin{bmatrix}
\Sigma^\lmk & U & \widetilde{N}\\
* & L(W^\tran) +\Sigma^\tran& \widetilde{E} \\
* & * & L(\widetilde{G}^\rot)+\widetilde{\Sigma}^\rot
\end{bmatrix},
\end{equation}
in which $\widetilde{\Sigma}^\lmk\in\H^{n'}$, $\widetilde{U}\in \C^{n'\times n}$, $\widetilde{N}\in \C^{n'\times n}$, $L(W^\tran)\in \H^n$, $\widetilde{\Sigma}^\tran\in \H^n$, $\widetilde{E}\in \C^{n\times n}$, $L(\widetilde{G}^\rot)\in \H^n$ and $\widetilde{\Sigma}^\rot\in \H^n$ are defined as
\allowdisplaybreaks
\begin{align}
\nonumber
[\Sigma^\lmk]_{ij}&\triangleq
\begin{cases}
\sum\limits_{(k,i)\in \aEEp} \nu_{ki},\hphantom{\;\;\quad\quad\;\;\;\;\,} & i=j,\\
0 & \mathrm{otherwise},
\end{cases}\\
\nonumber
[U]_{ij}&\triangleq
\begin{cases}
 -\nu_{ji},\hphantom{\sum\limits_{(k,j)\aEE}\tilde{\tran}\,\quad\quad\;\;\;\;} & (j,i)\in\aEEp,\\
0 & \mathrm{otherwise},
\end{cases}\\
\nonumber
[\widetilde{N}^\lmk]_{ij}&\triangleq
\begin{cases}
-\nu_{ji}\tilde{\lmk}_{ji},\hphantom{\sum\limits_{\in\aEE}\;\;\,\quad\quad\;\;\;\;} & (j,i)\in\aEEp,\\
0 & \mathrm{otherwise},
\end{cases}\\
\nonumber
[L(W^\tran)]_{ij}&\triangleq
\begin{cases}
\sum\limits_{(i,k)\in \mathcal{E}} \tau_{ik},\hphantom{\tilde{\tran}_{ik}\quad\quad\;\;\;\;} & i=j,\\
-\tau_{ij}, & (i,j)\in \mathcal{E},\\
0 & \mathrm{otherwise},
\end{cases}\\
\nonumber
[\Sigma ^\tran]_{ij}&\triangleq
\begin{cases}
\sum\limits_{(i,k)\in \aEEp} \nu_{ik},\hphantom{\;\;\quad\quad\quad\;\;} & i=j,\\
0 & \mathrm{otherwise},
\end{cases}\\
\nonumber
[\widetilde{E}]_{ij}&\triangleq\begin{cases}
\sum\limits_{(i,k)\in\aEE}\tau_{ik}\tilde{\tran}_{ik}+\\
\quad\;+\sum\limits_{(i,k)\in\aEEp}\nu_{ik}\tilde{\lmk}_{ik},\;\;\;\; & i=j,\\
-\tau_{ij}\tilde{\tran}_{ji}, & (j,i)\in\overrightarrow{\mathcal{E}},\\
0 & \mathrm{otherwise},
\end{cases}\\
\nonumber
[L(\widetilde{G}^\rot)]_{ij}&\triangleq\begin{cases}
\sum\limits_{(i,k)\in \mathcal{E}} \kappa_{ik},\hphantom{\tilde{\tran}_{ik}\quad\quad\quad\;} & i=j,\\
-\kappa_{ij}\tilde{\rot}_{ji}, & (i,j)\in\mathcal{E},\\
0 &\mathrm{otherwise},
\end{cases}\\
\nonumber
[\widetilde{\Sigma}^\rot)]_{ij}&\triangleq\begin{cases}
\sum\limits_{(i,k)\in\aEE} \tau_{ik}|\tilde{\tran}_{ik}|^2+\\
\quad\;+\sum\limits_{(i,k)\in\aEEp} \nu_{ik}|\tilde{\lmk}_{ik}|^2, &i=j\\
0 &\mathrm{otherwise},
\end{cases}
\end{align}
respectively.

It is possible to marginalize the translational states and landmarks and reformulate planar graph-based SLAM as an optimization problem on the rotational states only, which has been used in \cite{rosen2016se,carlone2013convergence,khosoussi2015exploiting} to improve the computational efficiency. In a similar way, if rotational states $\rot\triangleq\begin{bmatrix}
\rot_1 & \cdots & \rot_n
\end{bmatrix}^\transpose \in \C_1^n$ are known, \eqref{eq::P} is reduced to unconstrained complex quadratic programming on translational states $\tran\triangleq\begin{bmatrix}
\tran_1 & \cdots & \tran_n
\end{bmatrix}^\transpose \in \C^n$ and landmark positions $\lmk\triangleq\begin{bmatrix}
\lmk_{1} & \cdots & \lmk_{n'}
\end{bmatrix}^\transpose\in \C^w$: 
\begin{equation}\label{eq::qpp}
\min\limits_{\beta\in\C^{n'+n}}\;\beta^H \Lambda \beta + 2\innprod{\beta}{\widetilde{\Theta} \rot} + \underbrace{\rot^H L(\widetilde{G}^\rot)\rot + \rot^H \widetilde{\Sigma}^\rot \rot}_{\text{constant}},
\end{equation}
in which $\beta \triangleq\begin{bmatrix}
\lmk^\transpose & \tran^\transpose
\end{bmatrix}^\transpose\in \C^{n+n'}$, $\widetilde{\Theta}\triangleq\begin{bmatrix}
\widetilde{N}\\
\widetilde{E}
\end{bmatrix}\in \C^{(n+n')\times n}$, and $\Lambda\triangleq\begin{bmatrix}
\Sigma^\lmk & U\\
* & L(W^\tran)+\Sigma^\tran
\end{bmatrix}\in \H^{n+n'}$. It can be shown that according to \cite[Proposition 4.2]{gallier2010schur}\footnote{It should be noted that \cite[Proposition 4.2]{gallier2010schur} was originally derived for real matrices, however, the results can be generalized to complex matrices as well.}, one of the optimal solutions to \cref{eq::qpp} is
\begin{equation}\label{eq::sol_lin}
\beta=-\Lambda^\dagger \widetilde{\Theta} \rot.
\end{equation} 
Substituting \cref{eq::sol_lin} into \eqref{eq::P} and simplifying the resulting equation, we obtain the complex quadratic programming on the product of unit complex numbers $\C_1^n$ as follows
\begin{equation}\label{eq::QP0}
\min_{\rot\in \C_1^n} \rot^H \tM \rot,
\end{equation}
in which $\tM=L(\widetilde{G}^\rot)+\widetilde{\Sigma}^z-\widetilde{\Theta}^H\Lambda^\dagger\widetilde{\Theta}\succeq 0$. 

Furthermore, let $\Omega\in \R^{(m+m')\times(m+m')}$ be the diagonal matrix indexed by $e\in \aEE\cup\aEE'$ and $e'\in \aEE\cup\aEE'$ whose $(e,e')$-element is given by
\begin{equation}\label{eq::Omega}
[\Omega]_{ee'}\triangleq\begin{cases}
\nu_e, & e=e'\text{ and }e\in \aEEp,\\
\tau_{e}, & e=e'\text{ and }e\in \aEE,\\
0,&\text{otherwise},
\end{cases}
\end{equation}
in which $\nu_{e}$ and $\tau_{e}\in \R$ are the precisions of the landmark positional observations and the translational observations as given in Eqs. \eqref{eq::obl} and \eqref{eq::obt}, respectively; and let $\widetilde{T}\in \C^{(m+m')\times n}$ be the matrix indexed by $e\in\aEE\cup\aEEp$ and $k\in \VV\cup\VV'$ whose $(e,\,k)$-element is given by
\begin{equation}\label{eq::T}
\big[\widetilde{T}\big]_{ek}\triangleq\begin{cases}
-\tilde{\lmk}_{kj}, & e=(k,\,j)\in\aEEp,\\
-\tilde{\tran}_{kj}, & e=(k,\,j)\in\aEE,\\
0, &\text{otherwise};
\end{cases}
\end{equation}
and let $A(\aGG)\in\R^{n\times m}$ to the matrix indexed by $k\in \VV\cup\VV'$ and $e\in \aEE\cup\aEEp$ whose $(k,\,e)$-element is given by
\begin{equation}\label{eq::A}
\big[A(\aGG)\big]_{ke}=\begin{cases}
1, & e=(i,\,k)\in \aEE\cup\aEEp,\\
-1, & e=(k,\,j)\in \aEE\cup\aEEp,\\
0, & \text{otherwise}.
\end{cases}
\end{equation}
In addition, without loss of any generality, we also introduce the ordering over $\aEE\cup\aEEp$ and $\VV\cup\VV'$ such that $e'\in\aEEp$ precedes $e\in \aEE$ and $k'\in \VV'$ precedes $k\in \VV$.
As a result of \cref{eq::T,eq::Omega,eq::A}, $\tM=L(\widetilde{G}^\rot)+\widetilde{\Sigma}^\rot-\widetilde{\Theta}^H\Lambda^\dagger\widetilde{\Theta}$ can be rewritten as
\begin{equation}\label{eq::Q}
\tM=L(\widetilde{G}^\rot)+\widetilde{T}^H\Omega^{\frac{1}{2}}\Pi\Omega^{\frac{1}{2}}\widetilde{T},
\end{equation}
in which $\Pi\in\R^{(m+m')\times (m+m')}$ is the matrix of the orthogonal projection operator $\pi:\C^{m+m'}\rightarrow \ker(A(\aGG)\Omega^{\frac{1}{2}})$ onto the kernel of $A(\aGG)\Omega^{\frac{1}{2}}$. Therefore, \cref{eq::QP0} is equivalent to 
\begin{equation}\label{eq::QP}
\tag{QP}
\begin{aligned}
&\hspace{2em}\min_{\rot\in \C_1^n} \trace(\tM \rot\rot^H) ,\\
&\tM=L(\widetilde{G}^\rot)+\widetilde{T}^H\Omega^{\frac{1}{2}}\Pi\Omega^{\frac{1}{2}}\widetilde{T}.
\end{aligned}
\end{equation}
Interested readers can refer to Appendix \hyperref[app::A]{A} for a detailed derivation of \eqref{eq::QP}.

In the next section, we will relax \eqref{eq::QP} to complex semidefinite programming and show that the semidefinite relaxation is tight as long as the noise magnitude is below a certain threshold. 

\section{The Semidefinite Relaxation}\label{section::sdp}
In a similar way to \cite{bandeira2017tightness,boumal2016nonconvex,carlone2016planar}, it is straightforward to relax \eqref{eq::QP} to
\begin{equation}\label{eq::SDP}
\tag{SDP}
\begin{aligned}
&\quad\min_{X\in \H^n} \innprod{\tM}{X}\\
\mathrm{s.t.}\hspace{1em} &X\succeq0, \quad\diag(X)=\1.
\end{aligned}
\end{equation}
It should be noted that if $\hat{X}\in\H^n$ has rank one and solves \eqref{eq::SDP}, then a solution $\hat{\rot}\in \C_1^n$ to \eqref{eq::QP} can be exactly recovered from $\hat{X}$ through singular value decomposition with which we have $\hat{X}=\hat{\rot}\hat{\rot}^H$.  

In this paper, it is without loss of any generality to assume that all the manifolds are Riemannian submanifolds of Euclidean space \cite{absil2009optimization}, whose differential geometric properties, e.g., Riemannian gradients and Riemannian Hessians, are defined accordingly.  

In the rest of section, we will analyze and derive the conditions for the optimality of \eqref{eq::QP} and \eqref{eq::SDP}, and conditions for the tight relaxation of \eqref{eq::SDP}, all the proofs of which can be found in Appendix \hyperref[app::B]{B}.

From \cite{absil2009optimization}, the necessary conditions for the local optimality of \eqref{eq::QP} can be well characterized in terms of the Riemannian gradients and Hessians.

\begin{lemma}\label{lemma::qp}
	If $\hat{\rot}\in \C_1^n$ is a local optimum of \eqref{eq::QP}, then there exists a real diagonal matrix $\hat{\Lambda}\triangleq\Re\{\ddiag(\tM\hat{\rot}\hat{\rot}^H)\}\in \R^{n\times n}$ such that $\hat{S}\triangleq \tM-\hat{\Lambda}\in \H^n$ satisfies the following conditions:
	\begin{enumerate}[label=(\arabic*)]
		\item $\hat{S}\hat{\rot}=\0$;
		\item $\innprod{\dot{\rot}}{\hat{S}\dot{\rot}}\geq 0$ for all $\dot{\rot}\in T_{\hat{\rot}}\C_1^n$.
	\end{enumerate}
	If $\hat{\rot}$ satisfies (1), it is a first-order critical point, and if $\hat{\rot}$ satisfies (1) and (2), it is a second-order critical point.
\end{lemma}
\begin{proof}
	See Appendix \hyperref[app::B1]{B.1}.
\end{proof}

Since \eqref{eq::SDP} is convex and the identity matrix $\I\in \C^{n\times n}$ is strictly feasible, the sufficient and necessary conditions for the global optimality of \eqref{eq::SDP} can be derived in terms of the Karush-Kuhn-Tucker (KKT) conditions.

\begin{lemma}\label{lemma::sdp}
A Hermitian matrix $\hat{X}\in \H^n$ is a global optimum of \eqref{eq::SDP} if and only if there exists $\hat{S}\in \H^n$ such that the following conditions hold:
\begin{enumerate}[label=(\arabic*)]
\item $\diag(\hat{X})=\1$;
\item $\hat{X}\succeq 0$;
\item $\hat{S}\hat{X}=\0$;
\item $\tM-\hat{S}$ is real diagonal;
\item $\hat{S}\succeq 0$.
\end{enumerate}
Furthermore, if $\rank(\hat{S})=n-1$, then $\hat{X}$ has rank one and is the unique global optimum of \eqref{eq::SDP}.
\end{lemma}
\begin{proof}
See Appendix  \hyperref[app::B2]{B.2}.
\end{proof}

As a result of \cref{lemma::qp,lemma::sdp}, we obtain the sufficient conditions for the exact recovery of \eqref{eq::QP} from \eqref{eq::SDP}.

\begin{lemma}\label{lemma::qpsdp}
If $\hat{\rot}\in\C_1^n$ is a first-order critical point of \eqref{eq::QP} and $\hat{S}=\tM-\hat{\Lambda}\succeq 0$ in which $\hat{\Lambda}=\Re\{\ddiag(\tM\hat{\rot}\hat{\rot}^H)\}$, then $\hat{\rot}$ is a global optimum of \eqref{eq::QP} and $\hat{X}=\hat{\rot}\hat{\rot}^H$ is a global optimum of \eqref{eq::SDP}. Moreover, if $\rank(\hat{S})=n-1$, then $\hat{X}$ is the unique optimum of \eqref{eq::SDP}.

\end{lemma}
\begin{proof}
	See Appendix \hyperref[app::B3]{B.3}.
\end{proof}

\cref{lemma::qpsdp} gives sufficient conditions to check whether \eqref{eq::SDP} is a tight relaxation of \eqref{eq::QP}. As a matter of fact, if the measurement noise is not too large, it is guaranteed that \eqref{eq::SDP} is always a tight relaxation of \eqref{eq::QP} as the following proposition states.

\begin{prop}\label{prop::sdp}
Let $\underline{M}\in\H^n$ be the data matrix of the form \cref{eq::Q} that is constructed with the true (latent) pose-pose measurements $\underline{\rho}{}_{ij}=(\underline{\tran}{}_{ij},\,\underline{\rot}{}_{ij})$ and pose-landmark measurements $\underline{\lmk}_{ij}$, then there exists a constant $\gamma=\gamma(\underline{M})>0$ such that if $\|\tM-\underline{M}\|_2< \gamma$, then \eqref{eq::SDP} attains the unique global optimum at $\hat{X}=\hat{\rot}\hat{\rot}^H\in\H^n$, in which $\hat{\rot}\in\C_1^n$ is a global optimum of \eqref{eq::QP}.
\end{prop}
\begin{proof}
	See Appendix \hyperref[app::B4]{B.4}.
\end{proof}

\cref{lemma::qpsdp}  verifies the tightness of the complex semidefinite relaxation and \cref{prop::sdp}  guarantees that the tightness of the complex semidefinite relaxation, which makes \eqref{eq::SDP} certifiably correct for graph-based SLAM. It should be noted that similar results to \cref{lemma::qpsdp} and \cref{prop::sdp} have been presented for synchronization problems on general special Euclidean groups using the matrix representation in \cite{rosen2016se} and for phase synchronization using the complex number representation in \cite{bandeira2017tightness}.

In spite of the tightness of the semidefinite relaxation of planar graph-based SLAM, solving large-scale complex semidefinite programming remains challenging and time-consuming. In the next section, we will further relax \eqref{eq::SDP} as a series of rank-restricted complex semidefinite programming such that \eqref{eq::SDP} can be efficiently solved with the Riemannian staircase optimization.

\section{The CPL-SLAM Algorithm }\label{section::cpl}
In this section, we show that it is possible to recast \eqref{eq::SDP} as Riemannian optimization on complex oblique manifolds, which is one of the most important contributions of this paper. A brief introduction to the complex oblique manifold can be found in Appendix. \hyperref[app::D]{D}, and it is also helpful to read \cite{absil2006joint} that is about the real oblique manifold.

In general, interior point methods to solve \eqref{eq::SDP} take polynomial time, which, however, may still be slow when the polynomial exponent
is large. Instead of solving \eqref{eq::SDP} directly, Boumal \ea found that \eqref{eq::SDP} can be relaxed to a series of rank-restricted complex semidefinite programing \cite{boumal2016non}:
\begin{equation}\label{eq::rsdp}
\tag{$r$-SDP}
\min_{Y\in \mathrm{OB}(r,n)} \trace(\tM YY^H)
\end{equation}
in which
$$\mathrm{OB}(r,n)\triangleq\{Y\in \C^{n\times r}|\ddiag(YY^H)=\I \} $$ 
is the complex oblique manifold. It should be noted that \eqref{eq::rsdp} can be a tight relaxation of \eqref{eq::SDP} if some conditions are met as stated in \cref{prop::rsdp1,prop::rsdp2}, whose proofs are immediate from \cite[Theorem 2]{boumal2016non}.
\begin{prop}\label{prop::rsdp1}
If $\hat{Y}\in  \mathrm{OB}(r,n)$ is rank-deficient and second-order critical for \eqref{eq::rsdp}, then it is globally optimal for \eqref{eq::rsdp} and $\hat{X}=\hat{Y}\hat{Y}^H\in\H^n$ is globally optimal for \eqref{eq::SDP}. 
\end{prop}
\begin{prop}\label{prop::rsdp2}
If $r\geq \lceil\sqrt{n}\,\rceil$, then for almost all $\tM\in \C^{n\times n}$, every first-order critical $\hat{Y}\in \mathrm{OB}(r,n)$ for \eqref{eq::rsdp} is rank-deficient.
\end{prop}

\cref{prop::rsdp1,prop::rsdp2} are referred as the Burer-Monteiro guarantees for smooth semidefinite programming \cite{boumal2016non} that apply to a number of classic estimation problems. From \cref{prop::rsdp1,prop::rsdp2}, it can be concluded that \eqref{eq::SDP} is equivalent to successively solving \eqref{eq::rsdp} with the Riemannian trust region (RTR) method \cite{absil2007trust} for $2\leq r_1<r_2<\cdots<r_k\leq n+1$ until a rank-deficient second-order critical point is found, and such a method to solve semidefinite programming is referred as the Riemannian staircase optimization (\cref{algorithm::stair}) \cite{boumal2016non,boumal2015riemannian}. In \cite{rosen2016se,bandeira2017tightness,boumal2015riemannian}, the Riemannian staircase optimization has been used to solve a number of semidefinite relaxations of synchronization problems. In addition, it is known that the RTR method solves \eqref{eq::rsdp} locally in polynomial time \cite[Proposition 3]{boumal2016non}. In contrast to using interior point methods to solve \eqref{eq::SDP} directly, the Riemannian staircase optimization using the RTR method is empirically orders of magnitude faster in solving large-scale smooth semidefinite programming. 

As shown in \cref{algorithm::rounding}, the solution rounding of an optimum of  $Y^*\in \mathrm{OB}(r,n)$ of \eqref{eq::rsdp} is simply to assign $\hat{\rot}=\begin{bmatrix}
\hat{\rot}_1 & \cdots &\hat{\rot}_n
\end{bmatrix}\in \C^n$ to be the left-singular vector of $Y^*$ that is associated with the greatest singular value, and then normalize each $\rot_i$ to get $\hat{\rot}\in\C_1^n$. The solution rounding algorithm is exact if $\rank(Y^*)=1$. Moreover, it should be noted that the solution rounding algorithm can recover the global optimum $\hat{\rot}^*\in \C_1^n$ from $Y^*$ as long as the exactness of \eqref{eq::SDP} holds and $X=Y^*{Y^*}^H$ solves \eqref{eq::SDP}.

From algorithms of Riemannian staircase optimization (\cref{algorithm::stair}) and solution rounding (\cref{algorithm::rounding}), the proposed CPL-SLAM algorithm for planar graph-based SLAM is as shown in \cref{algorithm::cpl}, which follows a similar procedure to SE-Sync \cite{rosen2016se}. It should be noted that \cref{lemma::qp,lemma::qpsdp,lemma::sdp} can be used to certify the global optimality of the solution, and  \cref{prop::sdp,prop::rsdp1,prop::rsdp2} indicate that the CPL-SLAM algorithm is expected to retrieve the globally optimal solution to  the planar graph-based SLAM as long as the noise magnitude is below a certain threshold. Therefore, it can be concluded that CPL-SLAM is certifiably correct.

\begin{algorithm}[t]
	\caption{The Riemannian staircase optimization (RSO)}
	\label{algorithm::stair}
	\begin{algorithmic}[1]
	\State\textbf{Input}: Integers $2\leq r_0<r_1<\cdots<r_k\leq n+1$; an initial iterate $\rot_0\in\C_1^n$
	\vspace{0.3em}
	\State  $Y_0=\begin{bmatrix}
	\hat{\rot}_0 &  \0
	\end{bmatrix}\in \mathrm{OB}(r_0,n)$
	\vspace{0.3em}
	\For{$i=1\rightarrow k$}
	\State Implement the Riemannian optimization to solve $$Y_i^*=\arg\min\limits_{Y\in \mathrm{OB}(r_i,n)} \trace(\tM YY^H)$$ locally with $Y_i$ as an  initial guess
	\If{$\rank(Y_i^*)<r_i$}
	\State \textbf{return} ${Y}_i^*\in\mathrm{OB}(r_i,n)$ 
	\Else
	\State $Y_{i+1}=\begin{bmatrix}
	\hat{Y}_i & \0
	\end{bmatrix}\in \mathrm{OB}(r_{i+1},n)$
	\EndIf
	\EndFor
	\State \textbf{return} $Y_i^*\in \mathrm{OB}(r_k,n)$
	\end{algorithmic}
\end{algorithm}

\begin{algorithm}[!t]
	\caption{The rounding procedure for solutions of \eqref{eq::rsdp}}
	\label{algorithm::rounding}
	\begin{algorithmic}[1]
		\State\textbf{Input}: An optimum $Y^*\in\mathrm{OB}(r,n)$ to \eqref{eq::rsdp}
		\vspace{0.3em}
		\State  Assign $\hat{\rot}=\begin{bmatrix}
		\hat{\rot}_1 & \cdots & \hat{\rot}_n
		\end{bmatrix}^\transpose \in \C^n$ to be  the left-singular vector of $Y^*$ that is associated with the greatest singular value
		\For{$i=1\rightarrow n$}
		\State $\hat{\rot}_i = \dfrac{\hat{\rot}_i}{|\hat{\rot}_i|} $
		\EndFor
		\State \textbf{return} $\hat{\rot}\in\C_1^n$
	\end{algorithmic}
\end{algorithm}

\begin{algorithm}[!t]
	\caption{The CPL-SLAM algorithm}
	\label{algorithm::cpl}
	\begin{algorithmic}[1]
		\State\textbf{Input}: Integers $2\leq r_0<r_1<\cdots<r_k\leq n+1$; an initial iterate $\rot_0\in\C_1^n$
		\vspace{0.3em}
		\State  Implement \cref{algorithm::stair} to compute an optimum $Y^*\in\mathrm{OB}(r,n)$
		\vspace{0.2em}
		\State  Implement \cref{algorithm::rounding} to compute rotational states $\hat{\rot}\in \C_1^n$
		\State Implement \cref{eq::sol_lin} to compute translational states $\hat{\tran}\in \C^n$
		\State \textbf{return} $\hat{\rot}\in\C_1^n$ and $\hat{\tran}\in\C^n$
	\end{algorithmic}
\end{algorithm}

Similar to SE-Sync \cite{rosen2016se} and Cartan-Sync \cite{briales2017cartan}, CPL-SLAM uses the trust-region method on Riemannian manifolds that relies on the truncated conjugated gradient (TCG) method to evaluate the descent direction  \cite{absil2007trust}. The TCG method iteratively solves linear equations and improves the solution to necessary accuracy within finite iterations, which is usually faster than direct methods. In general, the TCG method needs a preconditioner to accelerate the convergence. Even though the choice of preconditioner for graph-based SLAM without landmarks is immediate \cite{rosen2016se}, there is still a lack of a suitable preconditioner for graph-based SLAM with landmarks. To address this issue, we also propose a preconditioner for graph-based SLAM with landmarks in Appendix \hyperref[app::C]{C}.

Even though the positive semidefinite matrix is not explicitly formed in CPL-SLAM or SE-Sync to solve planar graph-based SLAM, it can be seen that CPL-SLAM using the complex number representation results in semidefinite relaxations of smaller size than \cite{rosen2016se,carlone2016planar}. In the semidefinite relaxation of CPL-SLAM, the  $n\times n$ complex positive semidefinite matrix $X\in \H^n$  can be parameterized with $n^2$ real numbers, whereas the semidefinite relaxation in \cite{rosen2016se} using the matrix representation needs $2n^2+n$ real numbers to parameterize the $2n\times 2n$ real positive semidefinite matrix, and that in \cite{carlone2016planar} needs $4n^2$ real numbers to parameterize the $2n\times 2n$ complex positive semidefinite matrix.

It is obvious that the complex number representation is more concise than the matrix representation, and as a result, CPL-SLAM roughly requires half as much storage space as SE-Sync \cite{rosen2016se}. More importantly, as is discussed in \cref{section::result}, from both theoretical and empirical perspectives, the conciseness of the complex number representation reduces the computational cost a lot and renders the semidefinite relaxation much tighter, and thus, the resulting CPL-SLAM algorithm is much more efficient in numerical computation and much more robust to measurement noise than \cite{rosen2016se}.

In contrast to the works of \cite{rosen2016se,carlone2016planar,mangelson2018guaranteed}, CPL-SLAM is more general and more scalable. As mentioned before, CPL-SLAM is more efficient, tighter and more robust than SE-Sync \cite{rosen2016se}. Even though we use the same complex number representation as \cite{carlone2016planar}, our formulation is simpler and only depends on rotational states $\rot\in\C_1^n$, whereas \cite{carlone2016planar} involves both translational and rotational states $\tran\in\C^n$ and $\rot\in \C_1^n$. Moreover, \cite{carlone2016planar} mainly focuses on the optimality verification of planar pose graph optimization, whereas ours not only works on optimality verification and obtains stronger theoretical results, but also presents more scalable algorithms to solve planar graph-based SLAM. In \cite{mangelson2018guaranteed}, the authors use bounded sum of squares programming to solve planar graph-based SLAM. Even though \cite{mangelson2018guaranteed} always attains the globally optimal solution regardless of the measurement noise, it relies on sparse sum of squares programming, which, to our knowledge, has limited scalability for large-scale problems. As a result, CPL-SLAM can be expected to outperform \cite{mangelson2018guaranteed} by several orders of magnitude in terms of computational time. Last but not least, except for \cite{mangelson2018guaranteed}, the works of \cite{rosen2016se,carlone2016planar} are designed for planar pose graph optimization or angular synchronization, whereas ours considers the planar graph-based SLAM that has both poses and landmarks.

\begin{figure*}[!htbp]
	\centering
	\begin{tabular}{ccc}
		\subfloat[][]{\includegraphics[trim =0mm 0mm 0mm 0mm,width=0.312\textwidth]{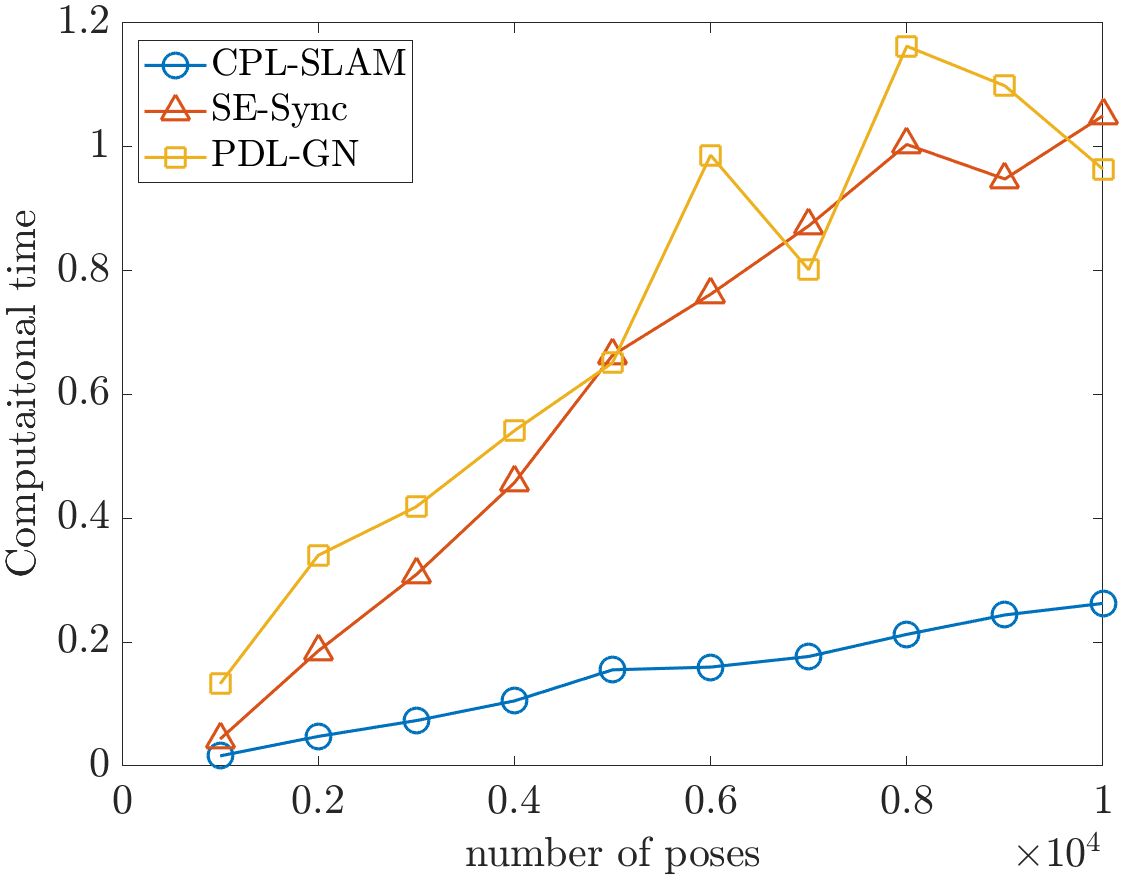}} &
		\subfloat[][]{\includegraphics[trim =0mm 0mm 0mm 0mm,width=0.312\textwidth]{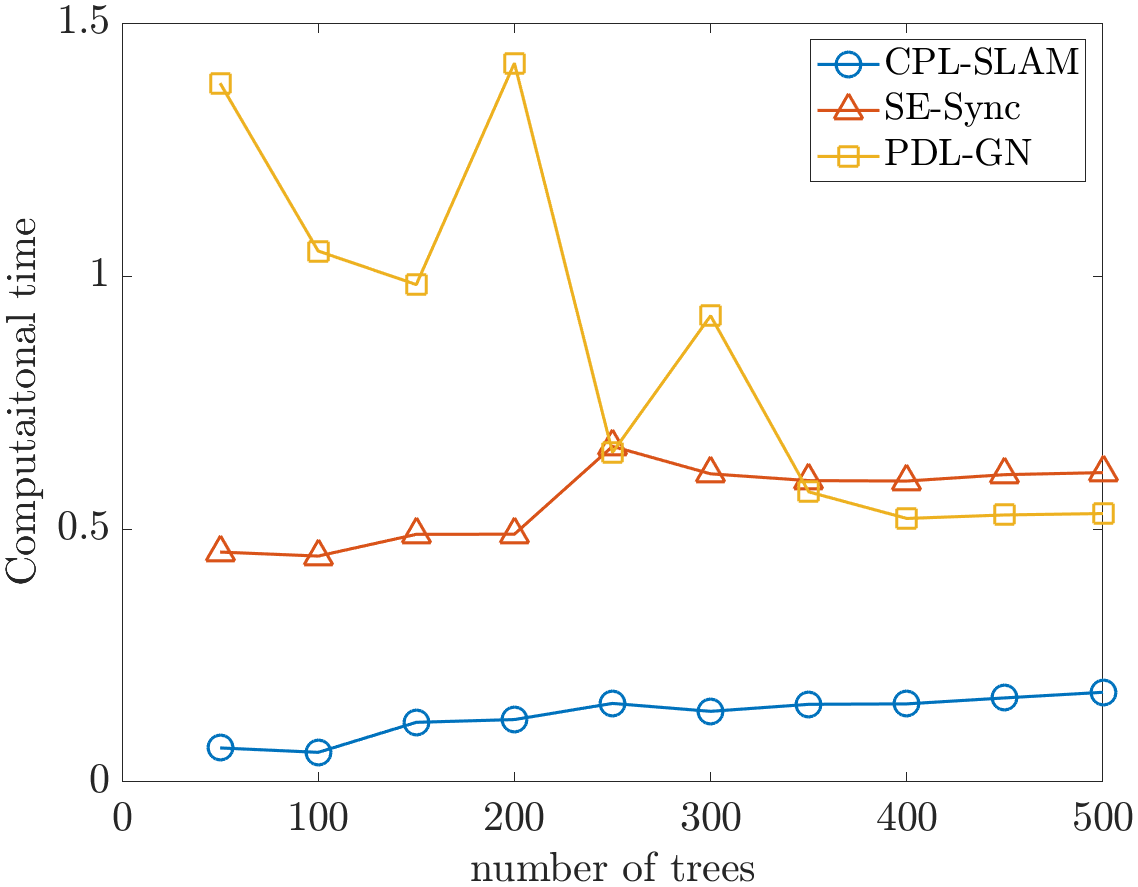}} & 
		\subfloat[][]{\includegraphics[trim =0mm 0mm 0mm 0mm,width=0.312\textwidth]{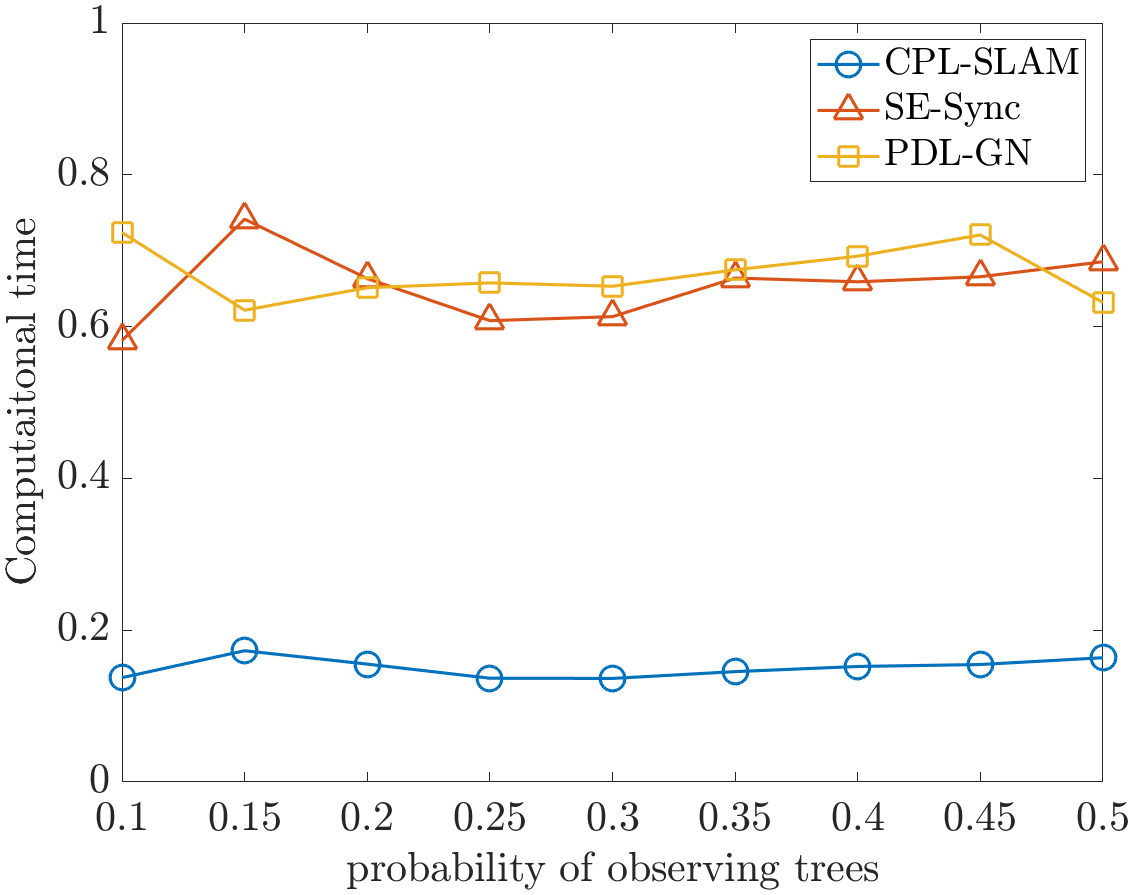}} \\
		\subfloat[][]{\includegraphics[trim =0mm 0mm 0mm 0mm,width=0.312\textwidth]{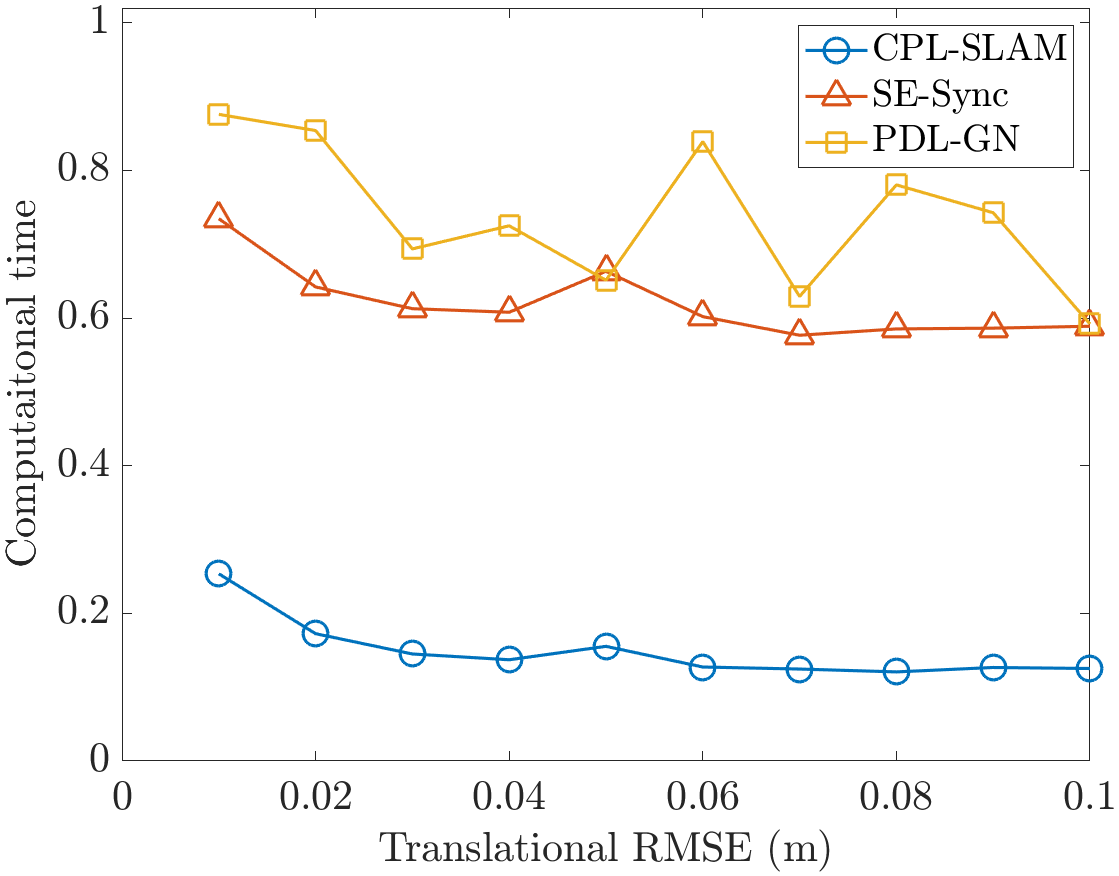}}&
		\subfloat[][]{\includegraphics[trim =0mm 0mm 0mm 0mm,width=0.312\textwidth]{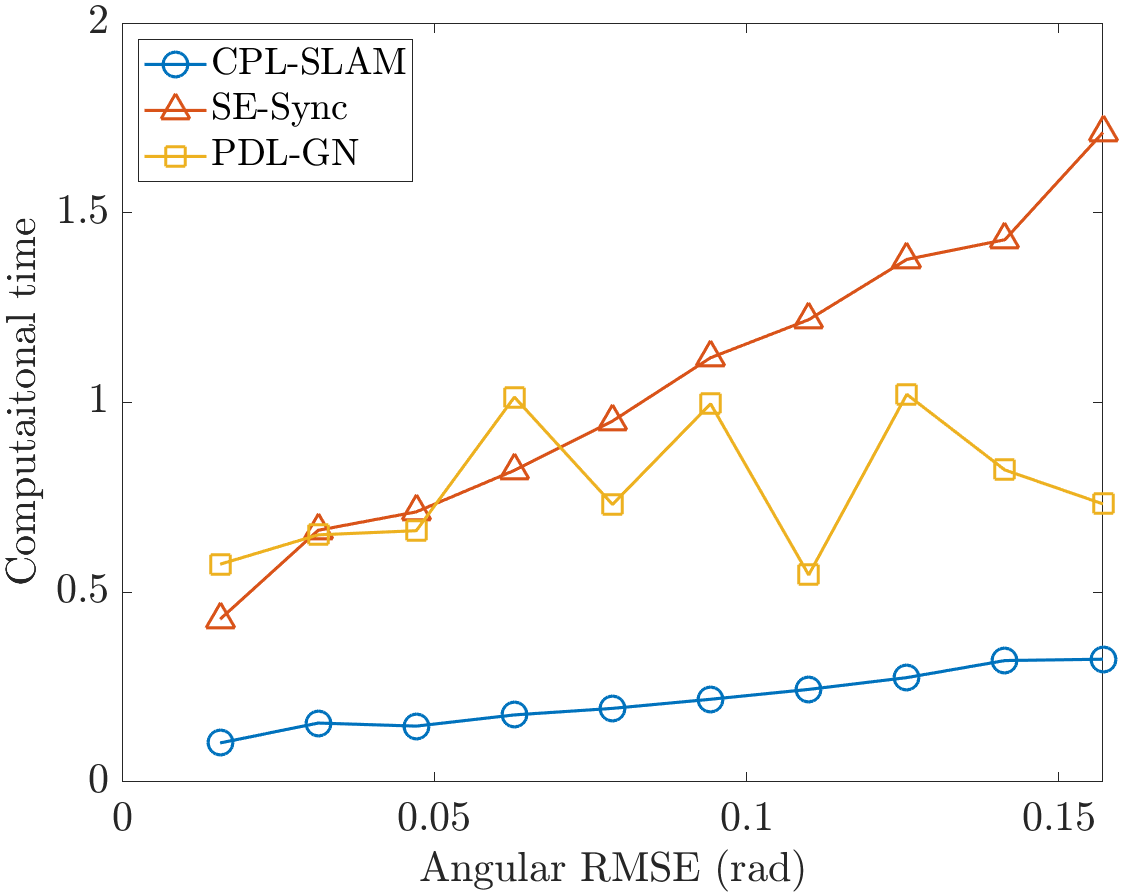}}&
		\subfloat[][]{\includegraphics[trim =0mm 0mm 0mm 0mm,width=0.312\textwidth]{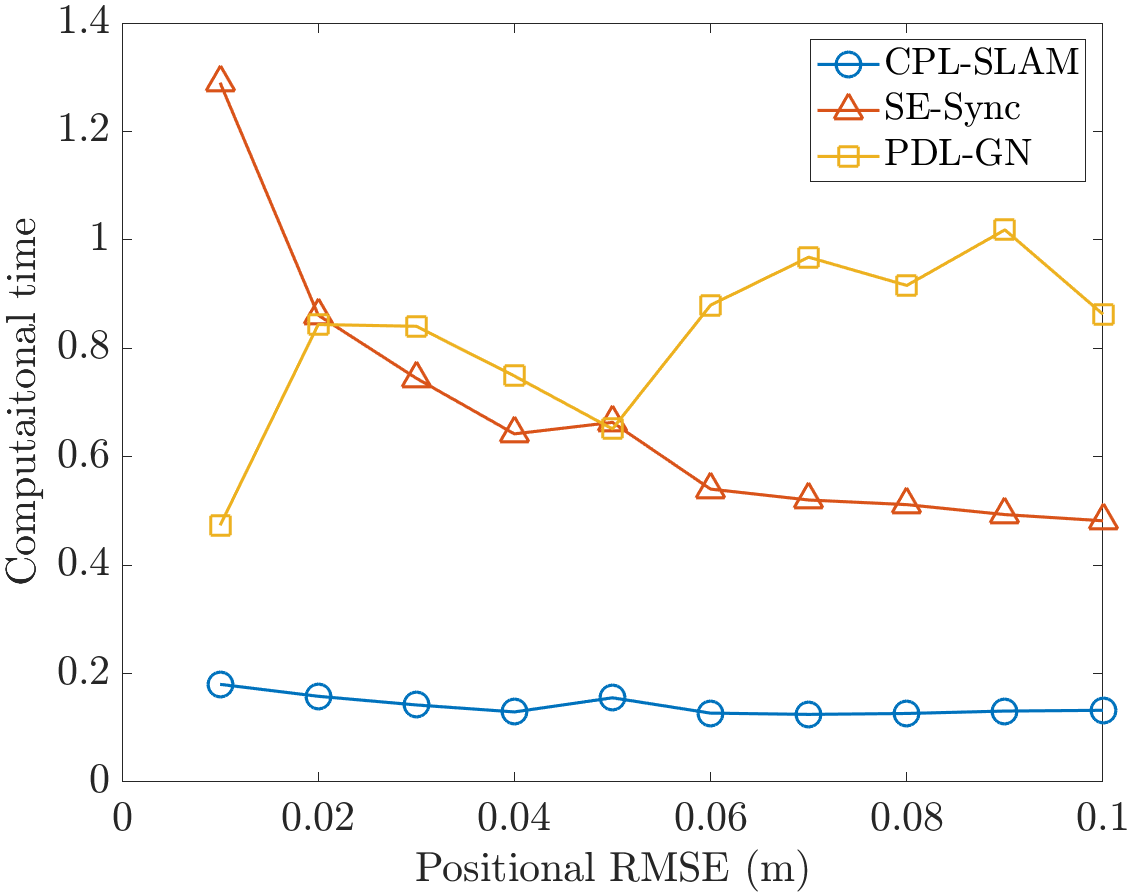}} 
	\end{tabular}
	\caption{The computational time of CPL-SLAM, SE-Sync and PDL-GN on the {\small\textsf{Tree}} datasets with varying each parameter individually while keeping the other parameters to be default values. The chordal initialization is used for all the tests. The results of each varying parameter are the number of poses $n$ in (a), the number of trees $n'$ in (b), the probability of observing trees $p_L$ in (c), translational RMSEs of $\sigma_t$ in (d), angular RMSEs of $\sigma_R$ in (e) and positional RMSEs of $\sigma_l$ in (f). The default values are $n=5000$, $n'=250$, $p_L=0.2$, $\sigma_t=0.05$ m, $\sigma_R=0.015\pi$ rad and $\sigma_l=0.05$ m. For all the {\small\textsf{Tree}} datasets tested, it can be seen that CPL-SLAM is around $4\sim 5$ times faster than SE-Sync and PDL-GN, whereas SE-Sync and PDL-GN are roughly as fast as each other.} 
	\label{fig::tree_time}
	\vspace{-0.25em} 
\end{figure*}

\section{The Results of Dataset Experiments}\label{section::result}
In this section, we implement CPL-SLAM on the simulated {\sf\small Tree} datasets,  simulated {\sf\small City} datasets  and a suite of large-scale 2D SLAM benchmark datasets with and without landmarks \cite{latif2014robust,rosen2016se,carlone2016planar}. We compare CPL-SLAM with the popular state-of-the-art SE-Sync \cite{rosen2016se} and Powell's Dog-Leg method (PDL-GN) \cite{dellaert2012factor,rosen2014rise}. Even though the original algorithms of SE-Sync \cite{rosen2016se} are not designed for problems with landmarks, we extend SE-Sync following a similar procedure as CPL-SLAM. For the linear solvers to compute a descent direction, CPL-SLAM and SE-Sync \cite{rosen2016se} use the indirect and iterative truncated conjugate gradient method, whereas PDL-GN \cite{dellaert2012factor,rosen2014rise} uses the sparse direct method. The C++ code of CPL-SLAM is available at \url{https://github.com/MurpheyLab/CPL-SLAM}.

All the experiments have been performed on a laptop with an Intel i7-8750H CPU and 32GB of RAM running Ubuntu 18.04 and using {\tt g++ 7.8} as C++ compiler. We have done the computation on a single core of CPU. For all the experiments, we choose the initial rank to be $r_{\mathrm{SE}}=3$ and $r_{\mathrm{CPL}}=2$ for SE-Sync and CPL-SLAM, respectively, since we find that $r_{\mathrm{SE}}=3$ and $r_{\mathrm{CPL}}=2$ are in general good enough for SE-Sync and CPL-SLAM to solve planar graph-based SLAM given the noise levels in robotics and computer vision applications.

\subsection{\textsf{\small Tree} Datasets}
In this subsection, we evaluate the performance of CPL-SLAM, SE-Sync and PDL-GN on the simulated {\sf\small Tree} datasets that are similar to {\sf\small tree10000} (\cref{fig::map}k). A \ssf{Tree} dataset is consisted of $25\times 25$ square grids in which each grid has side length of $1$ m, and a robot trajectory of $n$ poses along the rectilinear path of the square grid, and $n'$ trees (landmarks) that are randomly distributed in the centre of some square grids. Odometric pose-pose measurements are available between each pair of sequential poses along the robot trajectory, whereas pose-landmark measurements between poses and trees that are close to each other are available with a probability of $p_L$; the pose-pose measurements $\tilde{\rho}_{ij}=(\tilde{\tran}_{ij},\tilde{\rot}_{ij})$  and pose-landmark measurements $\tilde{\lmk}_{ij}$ are generated from the noise models of \cref{eq::observation}. In our experiments, we investigate the performance of these algorithms by varying each parameter individually and the default values for these parameters are chosen to be $n=5000$, $n'=250$, $p_L=0.2$, $\tilde{\tran}_{ij}$ with an expected translational root-mean-squared error (RMSE) of $\sigma_t = 0.05$ m, $\tilde{\rot}_{ij}$ with an expected angular RMSE of $\sigma_R=0.015\pi$ rad, and $\tilde{s}_{ij}$ with an expected positional RMSE of $\sigma_l=0.05$ m.

For all the {\small\textsf{Tree}} datasets tested, CPL-SLAM, SE-Sync and PDL-GN all converge to the global optima when using the chordal initialization. As is shown in \cref{fig::tree_time}, it can be seen that CPL-SLAM is around $4\sim 5$ times faster than SE-Sync and PDL-GN, whereas SE-Sync and PDL-GN are roughly as fast as each other.

The speed-up of CPL-SLAM over SE-Sync \cite{rosen2016se} in planar graph-based SLAM can be explained from several perspectives. 1) CPL-SLAM is more efficient for the objective and gradient evaluation, e.g., if the rank is $r_{\mathrm{SE}}=3$ and $r_{\mathrm{CPL}}=2$, CPL-SLAM only needs $\frac{1}{2}\sim \frac{2}{3}$ and $\frac{1}{4}\sim \frac{2}{3}$ operations of SE-Sync to evaluate the objective and gradient, respectively. 2) CPL-SLAM is more efficient in terms of the projection or retraction onto the manifold than SE-Sync -- the projection map of CPL-SLAM is just to normalize $n$ vectors, whereas that of SE-Sync has to compute $n$ singular value decompositions, which is much more time consuming. 3) CPL-SLAM is more efficient for chordal initialization and solution rounding. 4) As a result of the conciseness of the complex number representation, the preconditioner used in CPL-SLAM has a better approximation the Hessian matrix than SE-Sync, and thus, has a faster convergence of the truncated conjugate gradient method that the Riemannian trust region method implements to evaluate the descent direction. Therefore, CPL-SLAM should be theoretically more efficient than SE-Sync, which is further confirmed by the results of the experiments. 


Similar to \cite{kummerle2011g,khosoussi2015exploiting,nasiri2018recursive}, PDL-GN uses the Gauss-Newton method and might not perform well if there are large residues of the measurements and strong nonlinearities of the objective function \cite{rosen2014rise,dellaert2012factor}, whereas CPL-SLAM uses the exact Hessian to compute the Newton direction, and thus, is expected to converge faster and have better efficiency. In addition, as mentioned before, when evaluating the descent direction, PDL-GN factorizes sparse matrices to solve linear equations. In contrast, CPL-SLAM makes use of the truncated conjugate gradient method as the linear solver, which might also improve the overall efficiency of CPL-SLAM. On the other hand, since the choice of linear solvers is critical for the efficiency of optimizers, there is a possibility to improve PDL-GN's efficiency if the truncated conjugate gradient method is used.

The performance of CPL-SLAM, SE-Sync and PDL-GN is also evaluated if they are not well initialized. When the odometric initialization is used, it can be seen from \cref{fig::tree_obj} that  CPL-SLAM and SE-Sync converge to the global optima in spite of the poor initial guess, whereas PDL-GN gets stuck at the local optima and has much greater objective values.

As mentioned before, the convergence of CPL-SLAM and SE-Sync to global optima does not rely on initial guess since CPL-SLAM and SE-Sync essentially solve the semidefinite relaxation of graph-based SLAM and are guaranteed to attain the globally optimal solution as long as the magnitude of measurement noise is below a certain threshold. As a comparison, PDL-GN is a local search method whose performance is closely related with quality of initial guess, and thus the global optimum convergence of PDL-GN is usually not guaranteed even with low measurement noise.

\begin{figure*}[!htbp]
	\centering
	\begin{tabular}{ccc}
		\subfloat[][]{\includegraphics[trim =0mm 0mm 0mm 0mm,width=0.312\textwidth]{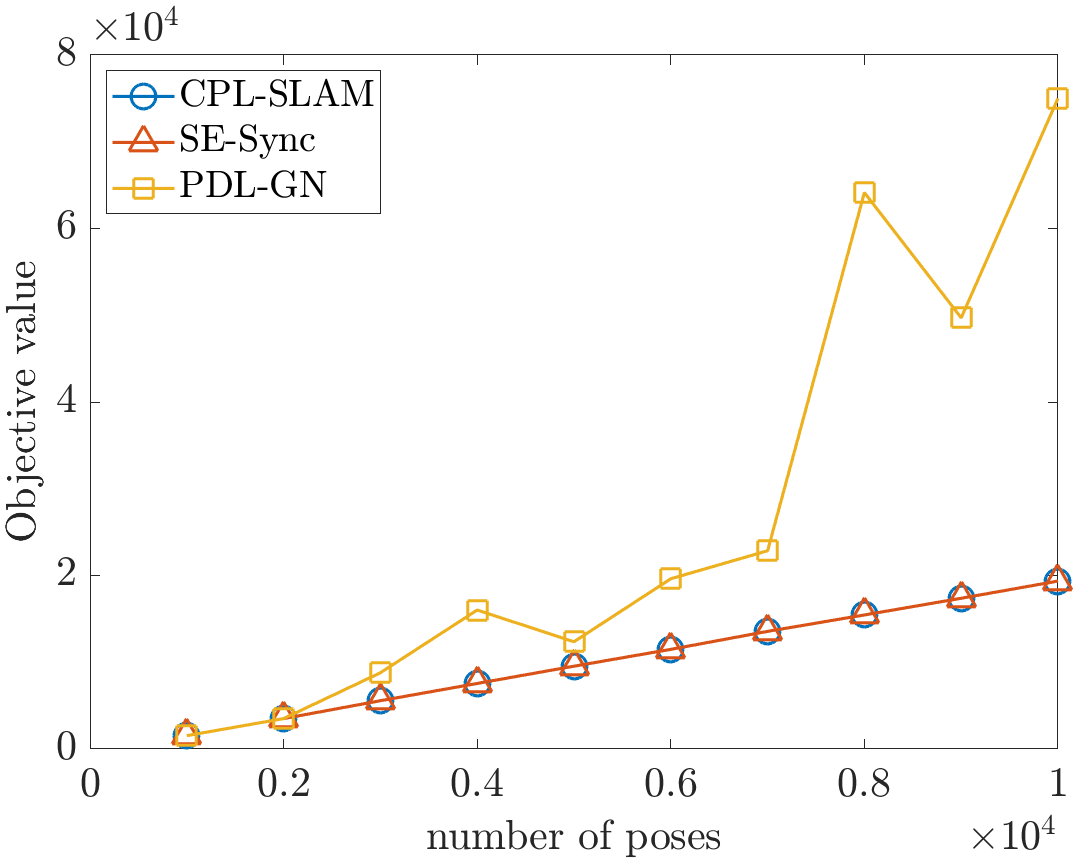}} &
		\subfloat[][]{\includegraphics[trim =0mm 0mm 0mm 0mm,width=0.312\textwidth]{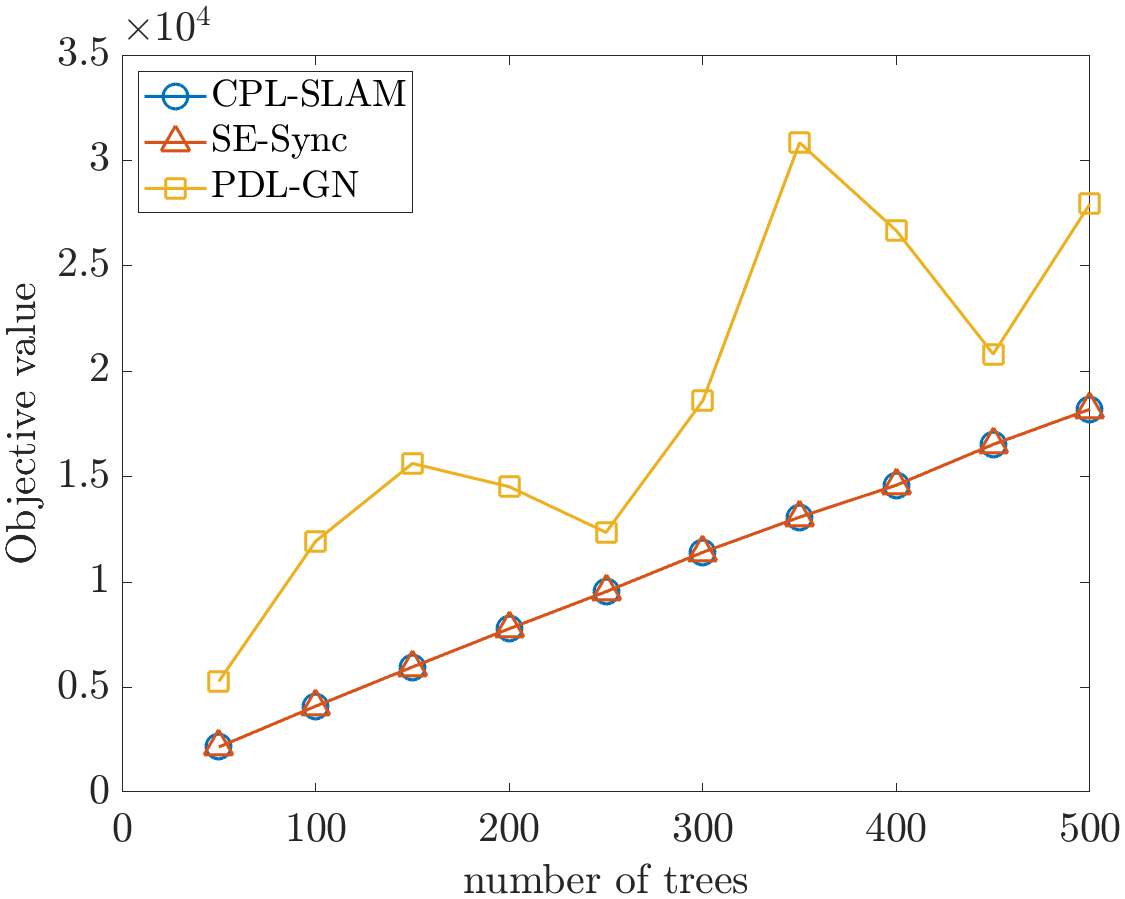}} & 
		\subfloat[][]{\includegraphics[trim =0mm 0mm 0mm 0mm,width=0.312\textwidth]{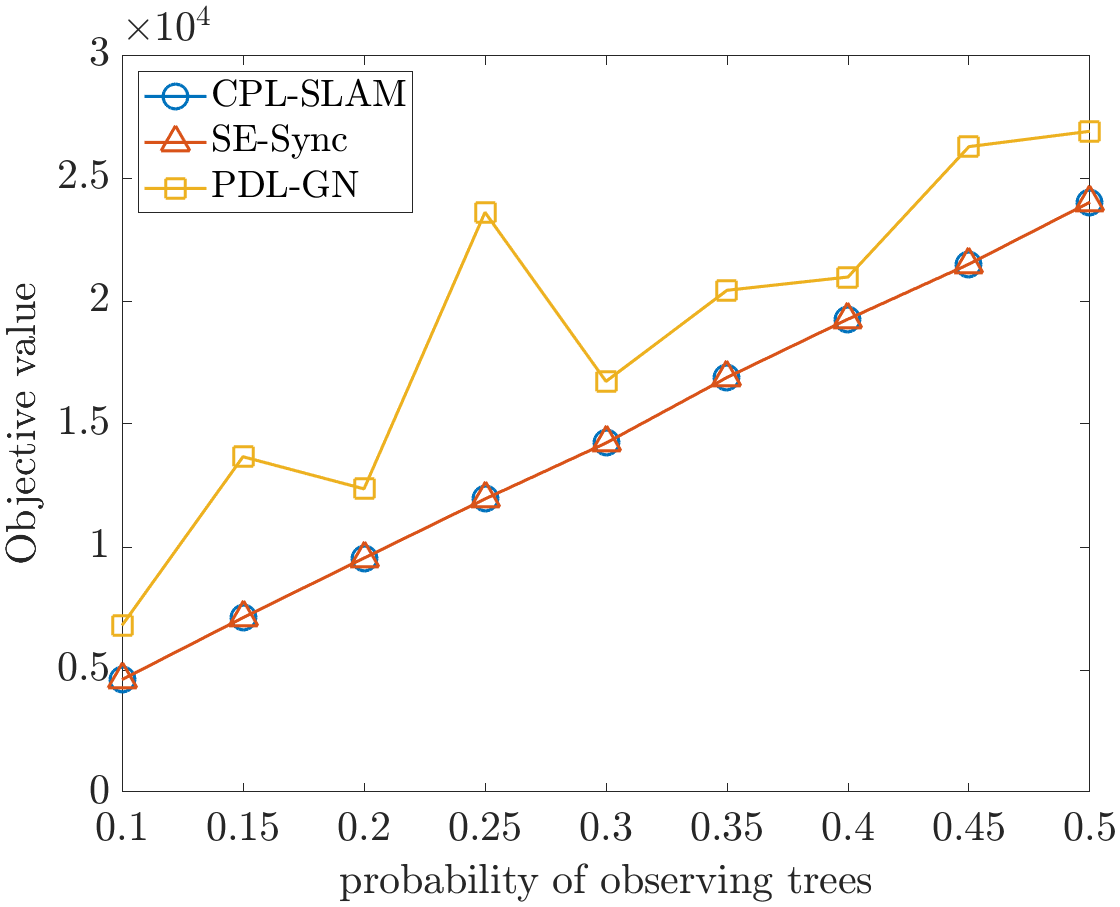}} \\[-0.2em]
		\subfloat[][]{\includegraphics[trim =0mm 0mm 0mm 0mm,width=0.312\textwidth]{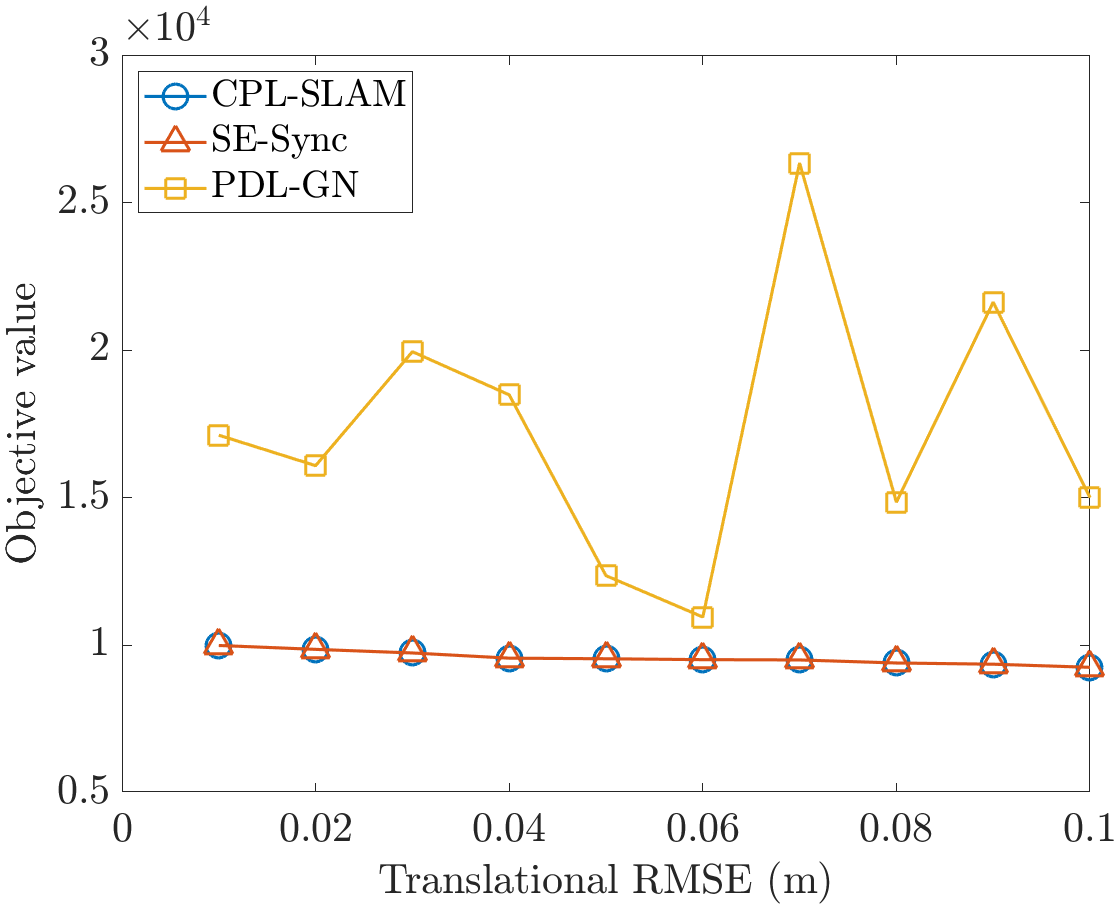}}&
		\subfloat[][]{\includegraphics[trim =0mm 0mm 0mm 0mm,width=0.312\textwidth]{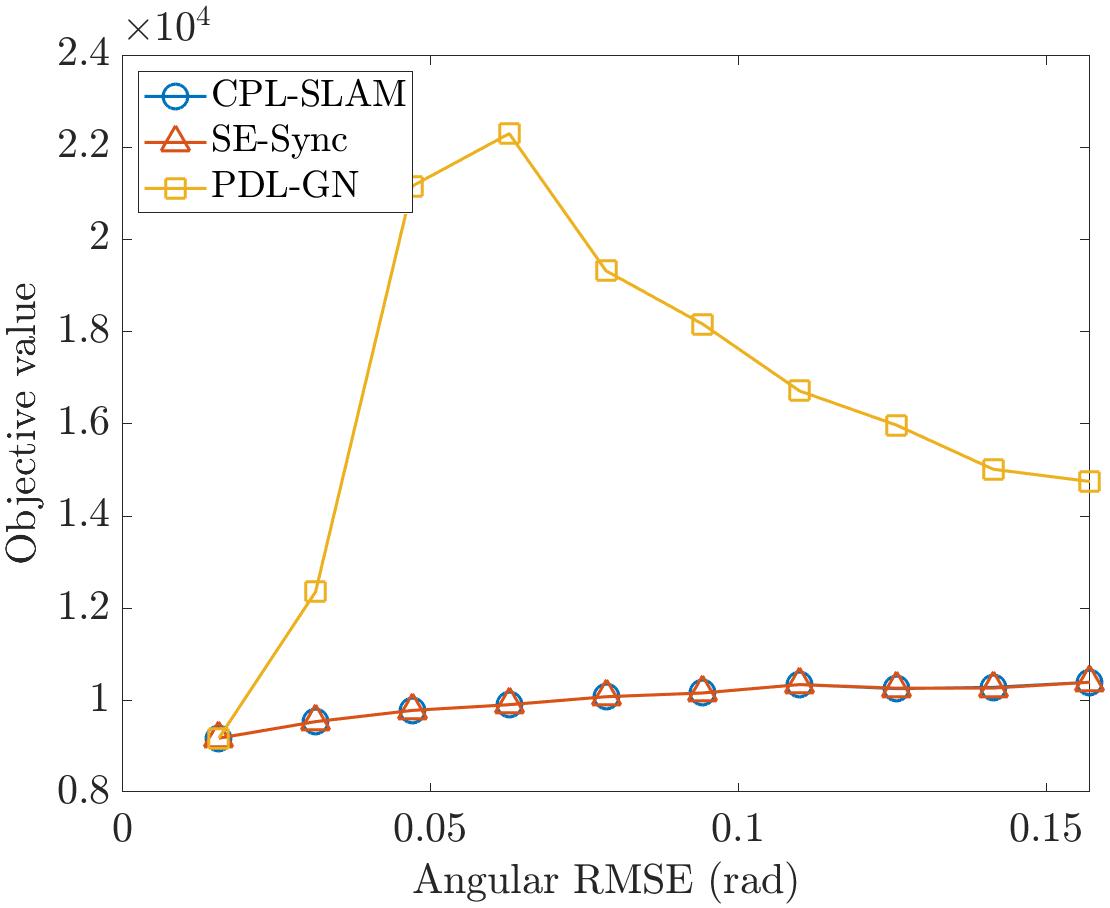}}&
		\subfloat[][]{\includegraphics[trim =0mm 0mm 0mm 0mm,width=0.312\textwidth]{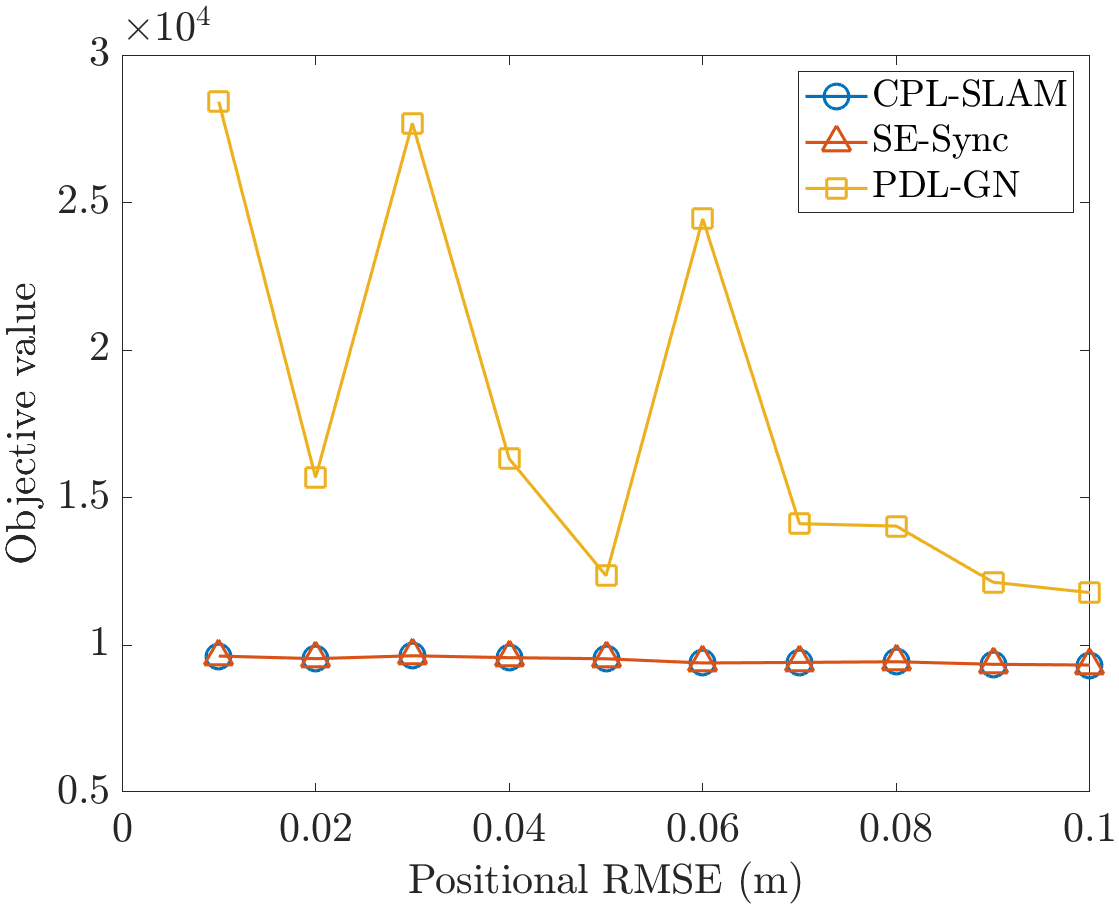}} 
	\end{tabular}
	\caption{The objective of CPL-SLAM, SE-Sync and PDL-GN on the {\small\textsf{Tree}} datasets using the odometric initialization. In the experiments, each parameter is varied individually while the other parameters are kept to be default values. The results of each varying parameter are the number of poses $n$ in (a), the number of trees $n'$ in (b), the probability of observing trees $p_L$ in (c), translational RMSEs of $\sigma_t$ in (d), angular RMSEs of $\sigma_R$ in (e) and positional RMSEs of $\sigma_l$ in (f). The default values are $n=5000$, $n'=250$, $p_L=0.2$, $\sigma_t=0.05$ m, $\sigma_R=0.015\pi$ rad and $\sigma_l=0.05$ m. For all the {\small\textsf{Tree}} datasets tested, it can be seen that CPL-SLAM and SE-Sync converge to global optima despite poor initialization, whereas PDL-GN gets stuck at local optima.} 
	\label{fig::tree_obj}
	\vspace{-0.25em} 
\end{figure*}

\subsection{\textsf{\small City} Datasets }
\vspace{-0.1em}
In this subsection, we evaluate the tightness of CPL-SLAM on a series of simulated {\sf\small City} datasets that are similar to {\sf\small city10000} (\cref{fig::map}b) but with high measurement noise. As a basis for comparison, we also evaluate the tightness of SE-Sync using the matrix representation \cite{rosen2016se}. In general, CPL-SLAM and SE-Sync are said to be tight if the globally optimal solution is exactly recovered from the semidefinite relaxation, or equivalently, there is no suboptimality gap between the rounded solution and the relaxed solution. 

\begin{figure*}[t]
	\centering
	\begin{tabular}{ccc}
		\hspace{-0.25em}\subfloat[][]{\includegraphics[trim =0mm 0mm 0mm 0mm,width=0.32\textwidth]{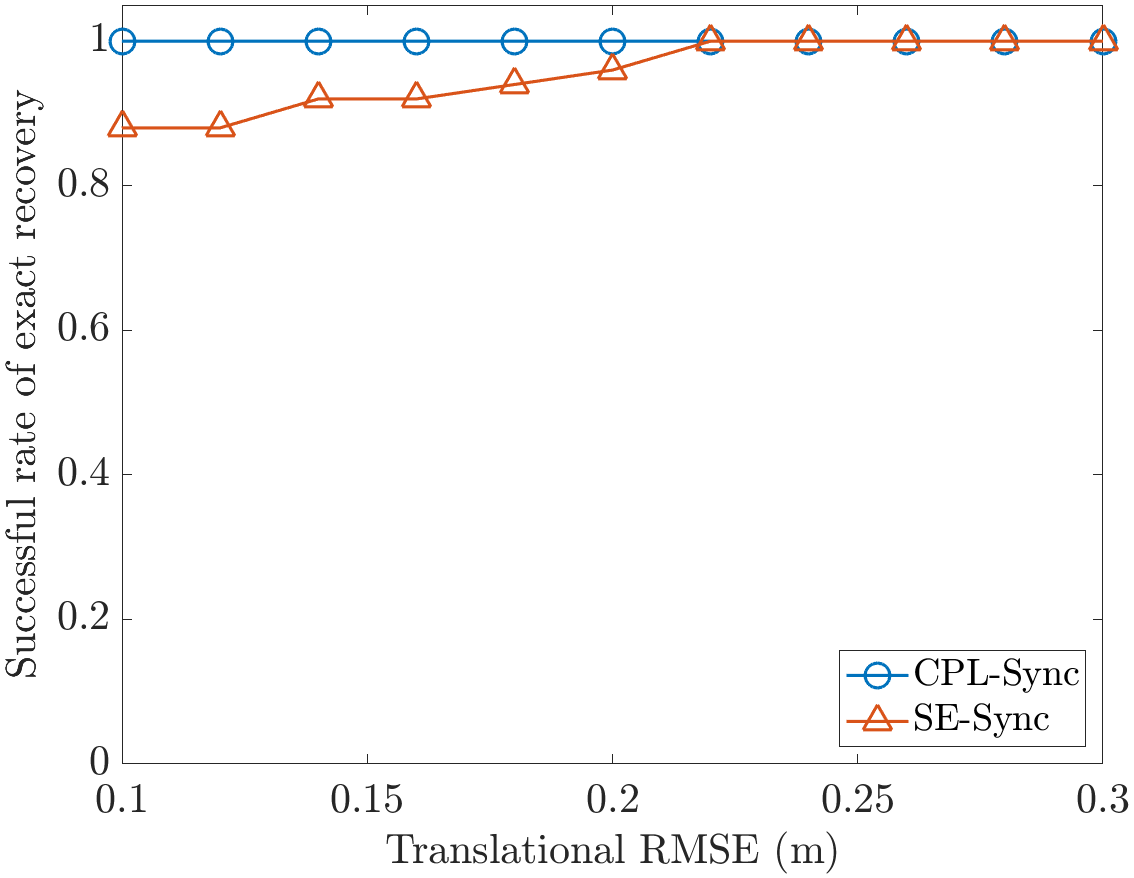}} &
		\hspace{-0.5em}\subfloat[][]{\includegraphics[trim =0mm 0mm 0mm 0mm,width=0.32\textwidth]{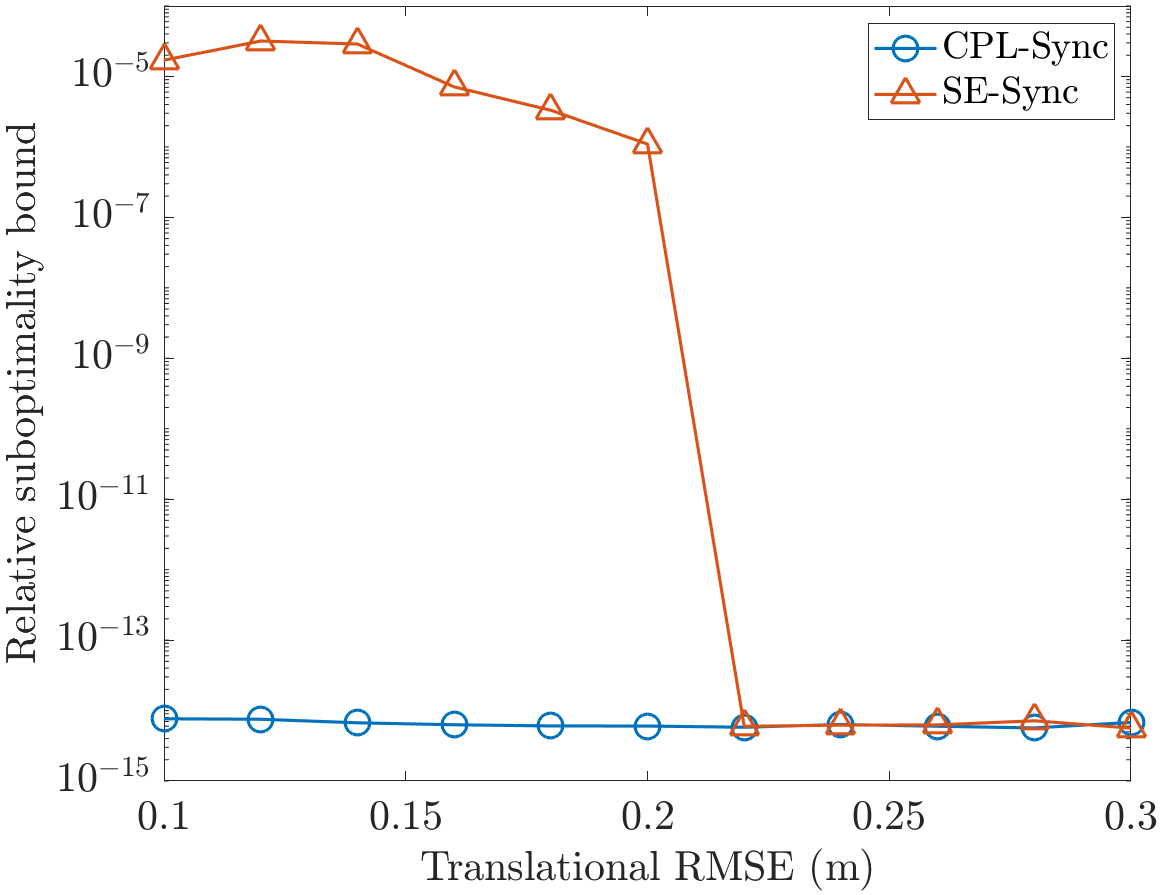}} &
		\hspace{-0.5em}\subfloat[][]{\includegraphics[trim =0mm 0mm 0mm 0mm,width=0.32\textwidth]{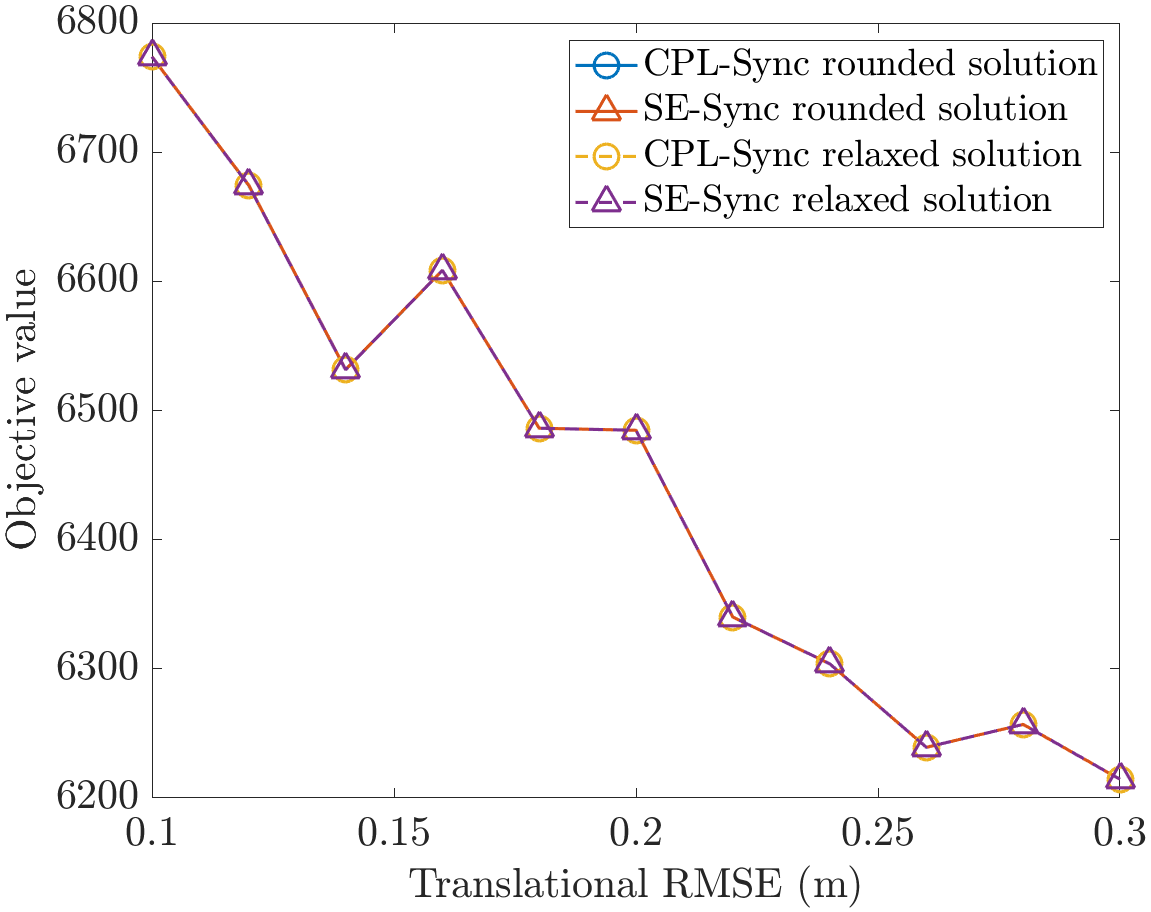}}  
	\end{tabular}
	\caption{The comparisons of CPL-SLAM and SE-Sync on the {\sf\small City} datasets with high translational measurement noise with $n=3000$, $p_C=0.1$, $\kappa_{ij}=40.53$ corresponding to angular RSME of $\sigma_R = 0.05\pi$ rad and varying $\tau_{ij}$ corresponding to different translational RSMEs of $\sigma_t=0.1\sim 0.3$ m. The results are (a) successful rates of exact recovery from the semidefinite relaxation, (b) relative suboptimality bounds between rounded and relaxed solutions, and (c) objective values of rounded and relaxed solutions. For all the datasets with different $\sigma_t$, CPL-SLAM has a tighter semidefinite relaxation and is more robust to translational measurement noise. }\label{fig::noise_t} 
	\vspace{0.85em}
	
	\centering
	\begin{tabular}{ccc}
		\hspace{-0.25em}\subfloat[][]{\includegraphics[trim =0mm 0mm 0mm 0mm,width=0.32\textwidth]{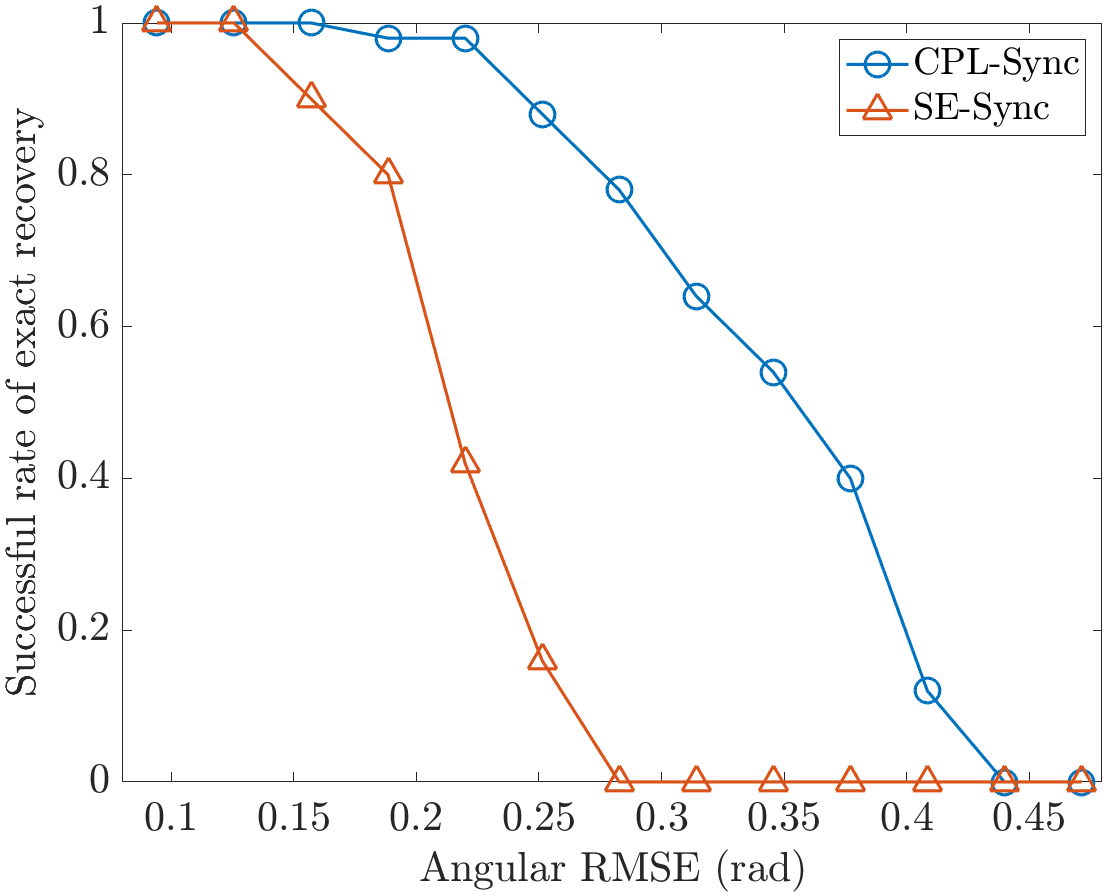}} &
		\hspace{-0.5em}\subfloat[][]{\includegraphics[trim =0mm 0mm 0mm 0mm,width=0.32\textwidth]{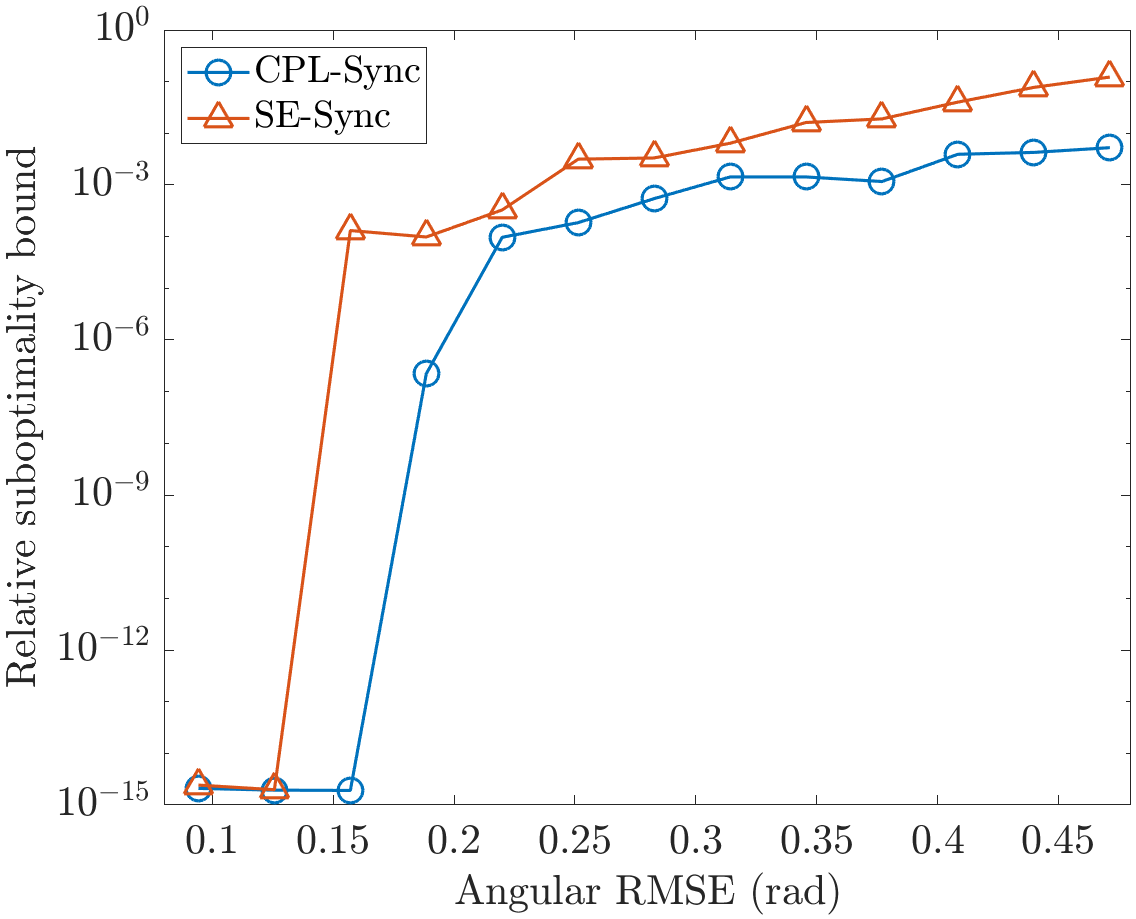}} &
		\hspace{-0.5em}\subfloat[][]{\includegraphics[trim =0mm 0mm 0mm 0mm,width=0.32\textwidth]{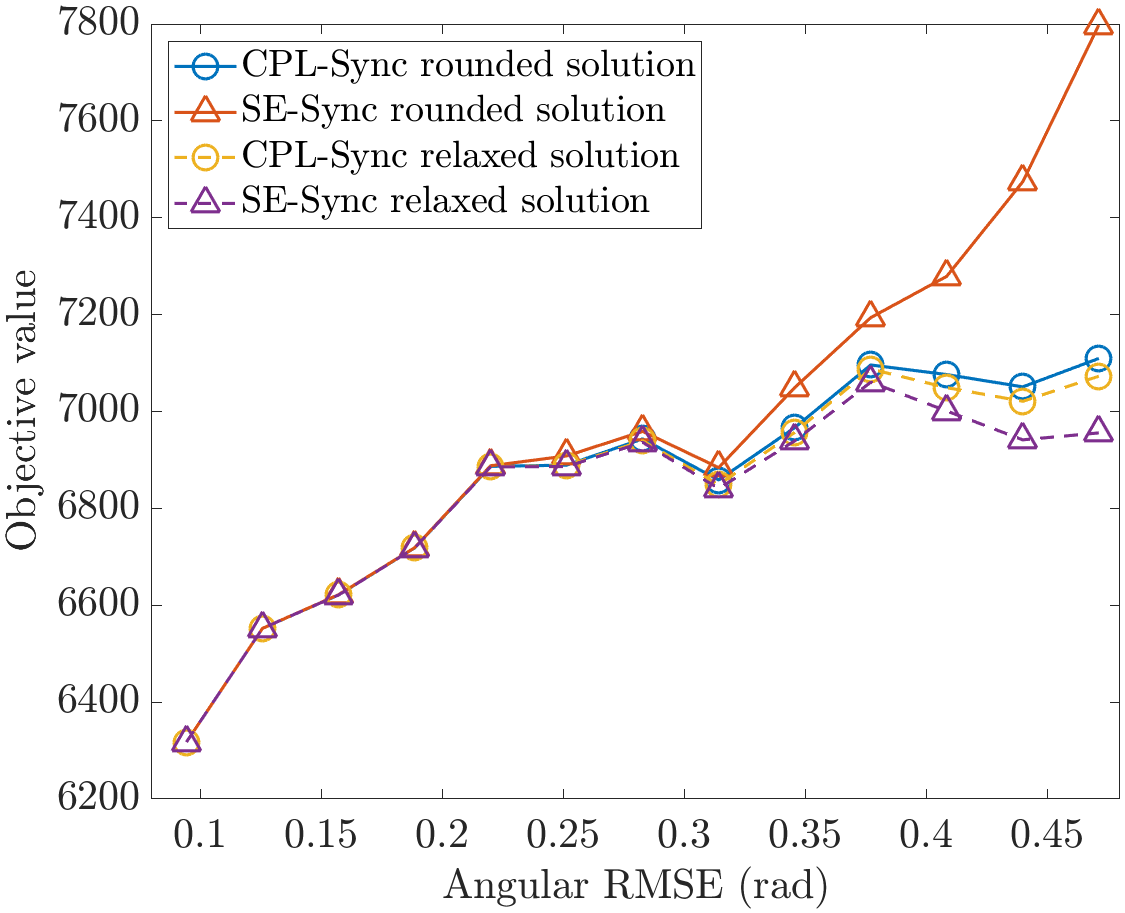}}  
	\end{tabular}
	\caption{The comparisons of CPL-SLAM and SE-Sync on the {\sf\small City} datasets with high rotational measurement noise with $n=3000$, $p_C=0.1$, $\tau_{ij}=88.89$ corresponding to translational RSME of $\sigma_t = 0.15$ m and varying $\kappa_{ij}$ corresponding to different angular RSMEs of $\sigma_R=0.03\pi\sim 0.15\pi$ rad. The results are (a) successful rates of exact recovery from the semidefinite relaxation, (b) relative suboptimality bounds between rounded and relaxed solutions, and (c) objective values of rounded and relaxed solutions. For all the datasets with different $\sigma_R$, CPL-SLAM has a tighter semidefinite relaxation and is more robust to rotational measurement noise. }\label{fig::noise_R} 
\end{figure*}

In our experiments, a {\sf\small City} dataset consists of $25 \times 25$ square grids in which each grid has side length of $1$ m, a robot trajectory of $n=3000$ poses along the rectilinear path of the grid, odometric measurements that are available between sequential poses along the robot trajectory, and loop-closure measurements that are available at random between non-sequential poses with a probability $p_{C}=0.1$. The odometric and loop-closure measurements are generated from noise models of \cref{eq::obt,eq::obr}, and the default translational weight factor is $\tau_{ij}=88.89$ that corresponds to an expected translational RMSE  of $\sigma_t=0.15$ m and the default rotational weight factor is $\kappa_{ij}=40.53$ that corresponds to an expected angular RMSE of $\sigma_R=0.05\pi$ rad. For the datasets, we vary translational and rotational measurement weight factors $\tau_{ij}$ and $\kappa_{ij}$ individually that correspond to translational and angular RMSEs of $\sigma_t=0.1 \sim 0.3$ m and $\sigma_R=0.03\pi \sim 0.15\pi$ rad, respectively,  while keeping the other weight factor as the default value. 

\begin{table*}[t]
	\renewcommand{\arraystretch}{1.3}
	\caption{Results of the 2D SLAM Benchmark datasets}\label{table::comp}
	\vspace{-0.75em}
	\begin{center}
		\begin{tabular}{|c||c|c|c||c||c|c||c|c|}
			\hline			
			\multirow{2}{*}{Dataset}&\multirow{2}{*}{$n+n'$}&\multirow{2}{*}{$m+m'$}&\multirow{2}{*}{$f^*$}& PDL-GN\cite{dellaert2012factor,rosen2014rise} & \multicolumn{2}{c||}{SE-Sync\cite{rosen2016se}} & \multicolumn{2}{c|}{CPL-SLAM [ours]}\\
			\cline{5-9}
			&& && Total time (s) &RTR time (s) & Total time (s) & RTR time (s) & Total time (s) \\
			\hline\hline
			\sf\scriptsize ais2klinik&$15115$&$16727$&$1.885\times 10^2$&$\tentimes{3.2}{0\hphantom{-}}$&$\tentimes{2.6}{0\hphantom{-}}$&$\tentimes{2.7}{0\hphantom{-}}$&$\tentimes{1.0}{0\hphantom{-}}$&$\tentimes{1.2}{0\hphantom{-}}$\\
			\hline
			\sf\scriptsize city10000&$10000$&$20687$&$6.386\times10^2$&$\tentimes{1.8}{0\hphantom{-}}$&$\tentimes{8.6}{-1}$&$\tentimes{1.2}{0\hphantom{-}}$&$\tentimes{5.2}{-1}$&$\tentimes{5.4}{-1}$\\
			\hline
			\sf\scriptsize CSAIL&$1045$&$1172$&$3.170\times 10^1$&$\tentimes{2.6}{-2}$&$\tentimes{5.0}{-3}$&$\tentimes{1.4}{-2}$&$\tentimes{1.0}{-3}$&$\tentimes{5.0}{-3}$\\
			\hline
			\sf\scriptsize intel&$1728$&$2512$&$5.236\times 10^1$&$\tentimes{1.3}{-1}$&$\tentimes{3.8}{-2}$&$\tentimes{6.1}{-2}$&$\tentimes{1.7}{-2}$&$\tentimes{2.6}{-2}$\\
			\hline
			\sf\scriptsize M3500&$3500$&$5453$&$1.939\times 10^2$&$\tentimes{3.3}{-1}$&$\tentimes{1.5}{-1}$&$\tentimes{2.2}{-1}$&$\tentimes{7.4}{-2}$&$\tentimes{9.8}{-2}$\\
			\hline
			\sf\scriptsize M3500-a&$3500$&$5453$&$1.598\times 10^3$&$\tentimes{4.1}{-1}$&$\tentimes{1.6}{-1}$&$\tentimes{2.3}{-1}$&$\tentimes{8.0}{-2}$&$\tentimes{1.0}{-1}$\\
			\hline
			\sf\scriptsize M3500-b&$3500$&$5453$&$3.676\times 10^3$&$\tentimes{1.6}{0\hphantom{-}}$&$\tentimes{5.3}{-1}$&$\tentimes{5.9}{-1}$&$\tentimes{2.6}{-1}$&$\tentimes{2.8}{-1}$\\
			\hline
			\sf\scriptsize M3500-c&$3500$&$5453$&$4.574\times 10^3$&$\tentimes{2.4}{0\hphantom{-}}$ &$\tentimes{7.5}{-1}$&$\tentimes{8.2}{-1}$&$\tentimes{3.7}{-1}$&$\tentimes{4.0}{-1}$\\
			\hline
			\sf\scriptsize FR-079&$989$&$1217$&$2.859\times 10^1$&$\tentimes{3.9}{-2}$&$\tentimes{5.9}{-3}$&$\tentimes{1.8}{-2}$&$\tentimes{1.7}{-3}$&$\tentimes{5.5}{-3}$\\
			\hline
			\sf\scriptsize MIT &$808$&$827$&$6.115\times 10^1$&$\tentimes{7.4}{-2}$&$\tentimes{1.4}{-2}$&$\tentimes{2.1}{-2}$&$\tentimes{4.7}{-3}$&$\tentimes{7.1}{-3}$\\
			\hline
			\sf\scriptsize tree10000&$10100$&$14442$&$6.035\times 10^2$&$\tentimes{6.9}{-1}$&$\tentimes{5.4}{-1}$&$\tentimes{5.9}{-1}$&$\tentimes{1.3}{-1}$&$\tentimes{2.3}{-1}$\\
			\hline
			\sf\scriptsize victoria-park&$7120$&$10608$&$4.660\times 10^2$&$\tentimes{2.1}{0\hphantom{-}}$&$\tentimes{5.9}{-1}$&$\tentimes{6.2}{-1}$&$\tentimes{1.4}{-1}$&$\tentimes{2.0}{-1}$\\
			\hline
		\end{tabular}
	\end{center}
\end{table*}
\begin{figure*}[!htbp]
	\centering
	\begin{tabular}{cc}
		\hspace{0em}\subfloat[][]{\includegraphics[trim =20mm 6mm 1mm 0mm,width=0.42\textwidth]{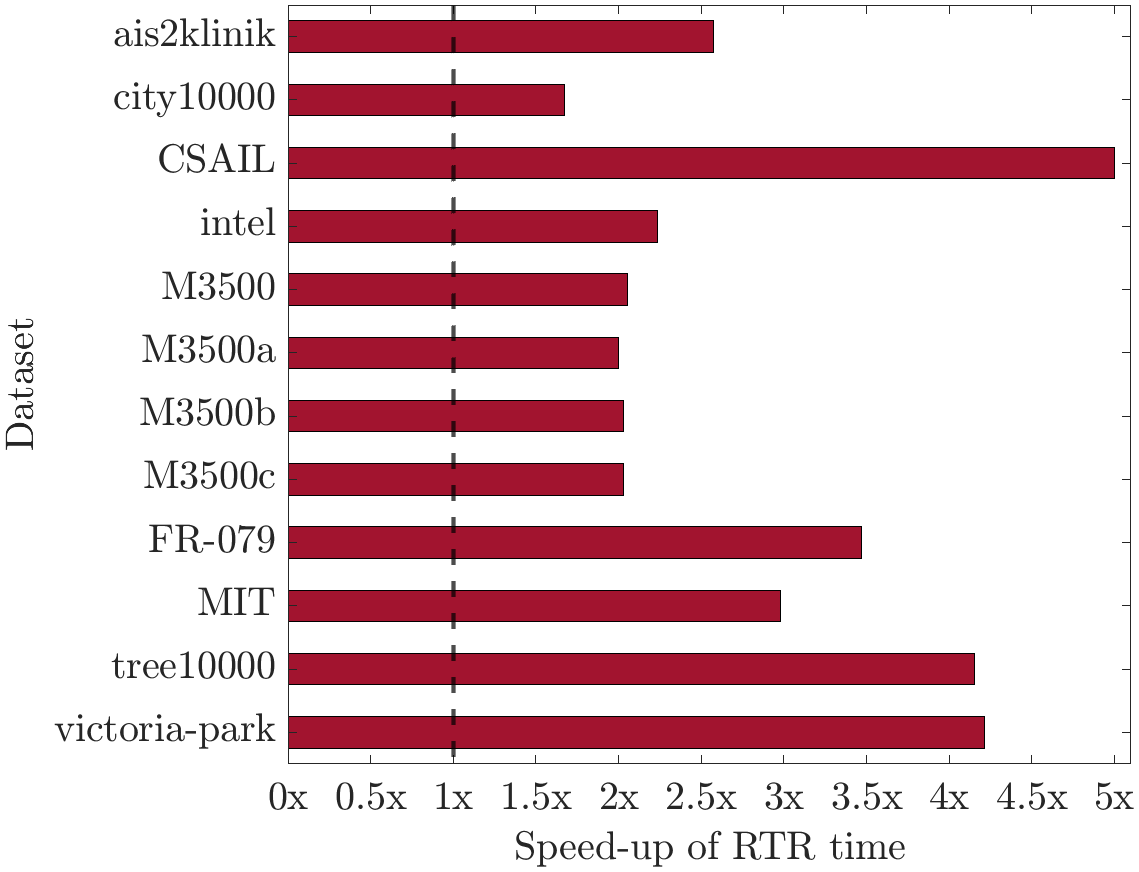}}\hspace{0.75em} &\hspace{1em}
		\subfloat[][]{\includegraphics[trim =20mm 6mm 1mm 0mm,width=0.435\textwidth]{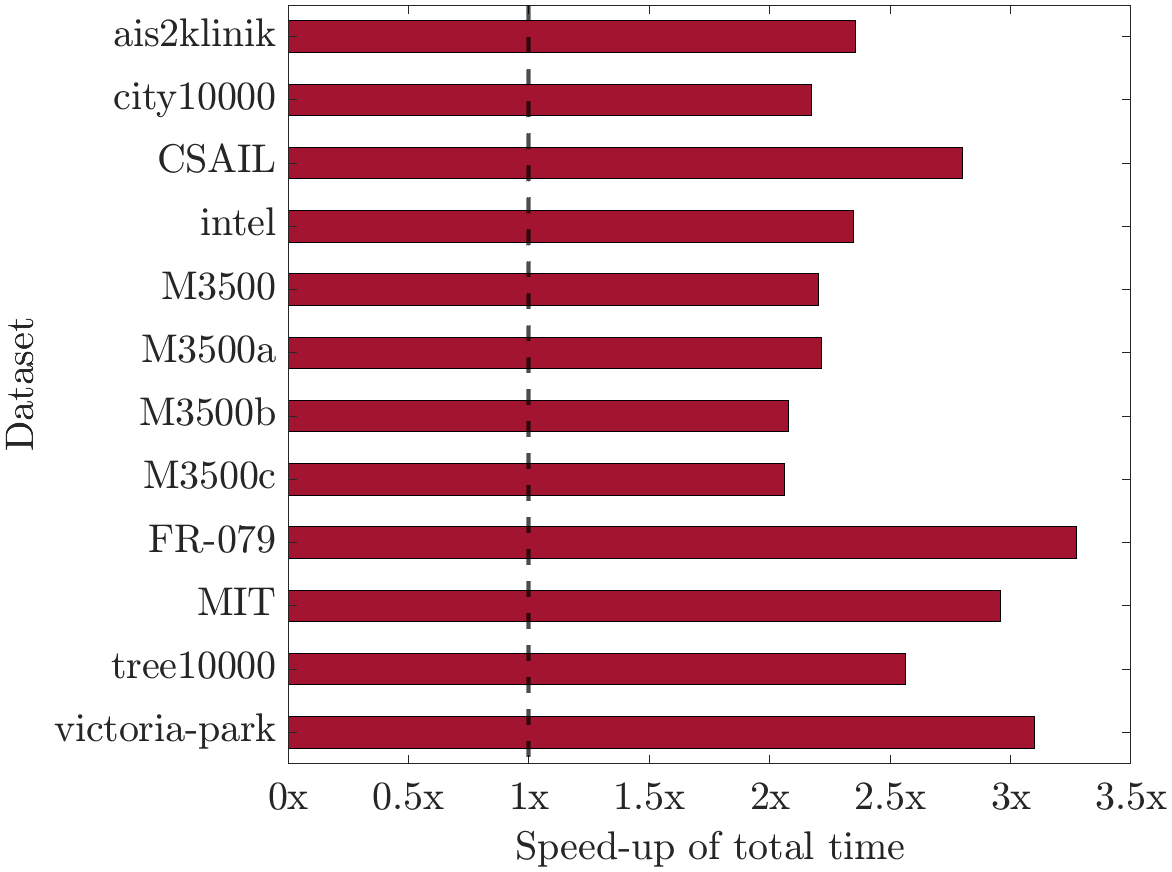}}
	\end{tabular}
	\caption{The speed-up of CPL-SLAM over SE-Sync on 2D SLAM benchmark datasets. The results are (a) the speed-up of RTR time of CPL-SLAM over SE-Sync and (b) the speed-up of total time of CPL-SLAM over SE-Sync. CPL-SLAM is on average $2.87$ and $2.51$ times faster than SE-Sync for RTR time and total time, respectively. }
	\label{fig::time} 
\end{figure*}

The results of CPL-SLAM and SE-Sync on the simulated {\sf\small City} datasets with high translational and rotational measurement noise are in \cref{fig::noise_R,fig::noise_t}, respectively. For each translational and angular RMSE, we calculate the successful rates of exact recovery from the semidefinite relaxation (\cref{fig::noise_t}a and \cref{fig::noise_R}a), the relative suboptimality bounds between rounded and relaxed solutions (\cref{fig::noise_t}b and \cref{fig::noise_R}b), and the objective values of rounded and relaxed solutions (\cref{fig::noise_t}c and \cref{fig::noise_R}c) statistically from $50$ randomly generated {\sf\small City} datasets, in which we assume the globally optimal solution is exactly recovered if the relative suboptimality bound is less than $1\times 10^{-6}$. From \cref{fig::noise_t}, it can be seen that CPL-SLAM holds the tightness on all the datasets with translational RMSEs of $\sigma_t=0.1\sim 0.3$ m, whereas SE-Sync fails on some of the datasets. From \cref{fig::noise_R}, it can be seen that when the angular RMSE is small, i.e., approximately less than $0.15$ rad, both CPL-SLAM and SE-Sync exactly recover the globally optimal solution from the semidefinite relaxation, and as angular RMSE increases and is greater than $0.15$ rad, CPL-SLAM and SE-Sync begin to fail. In spite of this, we find that CPL-SLAM has a much higher successful rate of exact recovery from the semidefinite relaxation (\cref{fig::noise_t}a and \cref{fig::noise_R}a) and orders of magnitude smaller relative suboptimality bounds (\cref{fig::noise_t}b and \cref{fig::noise_R}b). Furthermore, for the objective value, CPL-SLAM has greater lower bound from the relaxed solution but lower upper bound from the rounded solution in scenarios of high measurement noise (\cref{fig::noise_t}c and \cref{fig::noise_R}c). All of these results indicate that CPL-SLAM has a tighter semidefinite relaxation using the complex number representation than SE-Sync using the matrix representation, and thus, is more robust to translational and rotational measurement noise.

In \cref{fig::noise_t}, it is interesting to see that SE-Sync fails on datasets with small translational measurement noise but works on datasets with large translational measurement noise. Even though there is lack of formal analysis, we guess this is because the tightness of the semidefinite relaxation in SE-Sync, in addition to the magnitude of measurement noise, is also related with the ratio $\tau_{ij}/\kappa_{ij}$ of translational weight factors $\tau_{ij}$ and rotational weight factors $\kappa_{ij}$, i.e., when $\tau_{ij}/\kappa_{ij}$ increases, the semidefinite relaxation in SE-Sync tends to be relatively more sensitive to measurement nose.

It is obvious that the improved tightness and robustness of CPL-SLAM over SE-Sync in planar graph-based is associated with the more concise representation of complex numbers over matrices in the semidefinite relaxation, for which a theoretically complete analysis similar to \cite{tron2015inclusion} is left as future work. In spite of this, we present one possible reason that might help explain the improved tightness of CPL-SLAM. The semidefinite matrix resulting from the solution to planar graph-based SLAM using the matrix representation should take the form $X_{R}=\begin{bmatrix}
X_{R_{ij}}
\end{bmatrix}\!\in\! \R^{2n\times 2n}$ in which each $(i,j)$-th block $X_{R_{ij}}$ has the algebraic structure $X_{R_{ij}}\!=\!\begin{bmatrix}
a & -b\\
b & a
\end{bmatrix}\!\in \! \R^{2\times 2}$, and SE-Sync drops such an algebraic structure in the semidefinite relaxation. Even though it is possible for SE-Sync to keep this algebraic structure  by either reformulating the associated data matrix or adding numbers of extra linear constraints, substantial computational efforts are required for both options. In comparison, CPL-SLAM preserves the algebraic structure of $X_{R_{ij}}$ as complex numbers in the semidefinite relaxation without having to reformulate the data matrix or introduce any extra constraints. As said before, the explanations given above are still hypotheses and need to be proved. However, we can still conclude from the \cref{fig::noise_R,fig::noise_t} that the semidefinite relaxation in CPL-SLAM using the complex number representation is tighter than that in SE-Sync using the matrix representation, which further suggests that CPL-SLAM is more robust to measurement noise than SE-Sync.

\begin{figure*}[!t]
	\centering
	\vspace{-4.5mm}
	\begin{tabular}{cccc}
		\hspace{-1.5em}	\subfloat[][{\sf\fontsize{7.35}{8}\selectfont ais2klinik}]{\includegraphics[trim =0mm 0mm 0mm 0mm,width=0.24\textwidth]{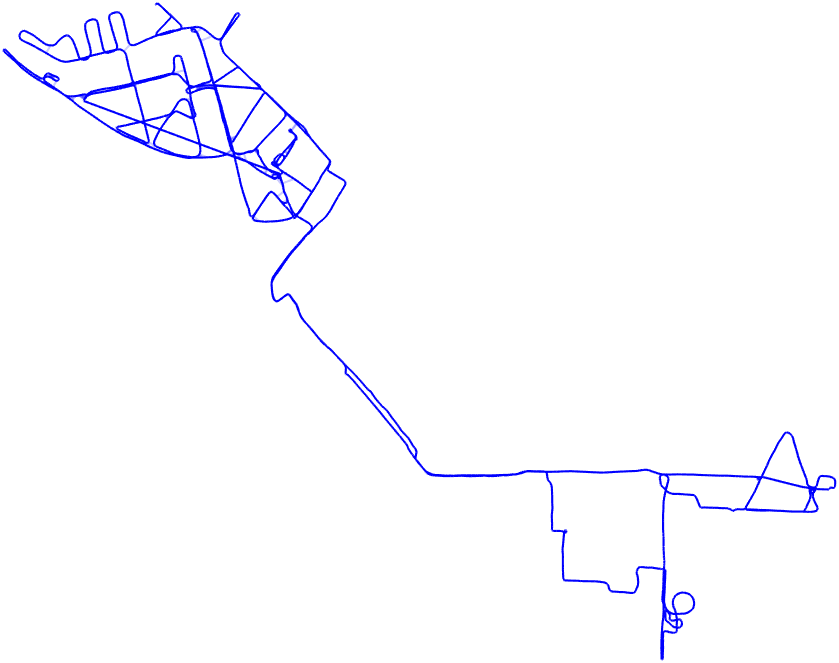}}&
		\hspace{-0.2em}	\subfloat[][\sf{\fontsize{7.35}{8}\selectfont city10000}]{\includegraphics[trim =0mm 0mm 0mm 0mm,width=0.19\textwidth]{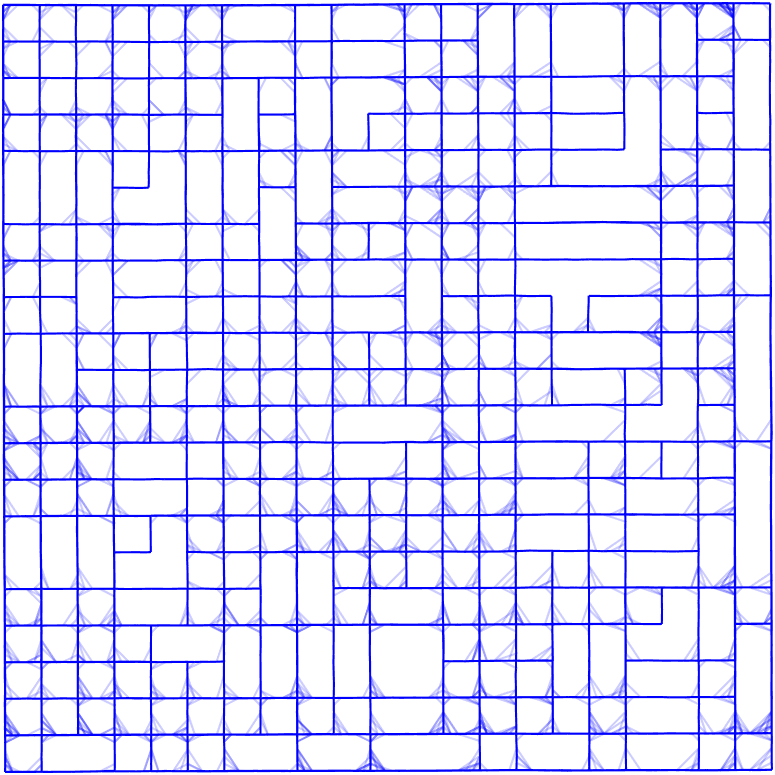}} &
		\hspace{-0.2em}	\subfloat[][\sf{\fontsize{7.35}{8}\selectfont CSAIL}]{\includegraphics[trim =0mm 0mm 0mm 0mm,width=0.23\textwidth]{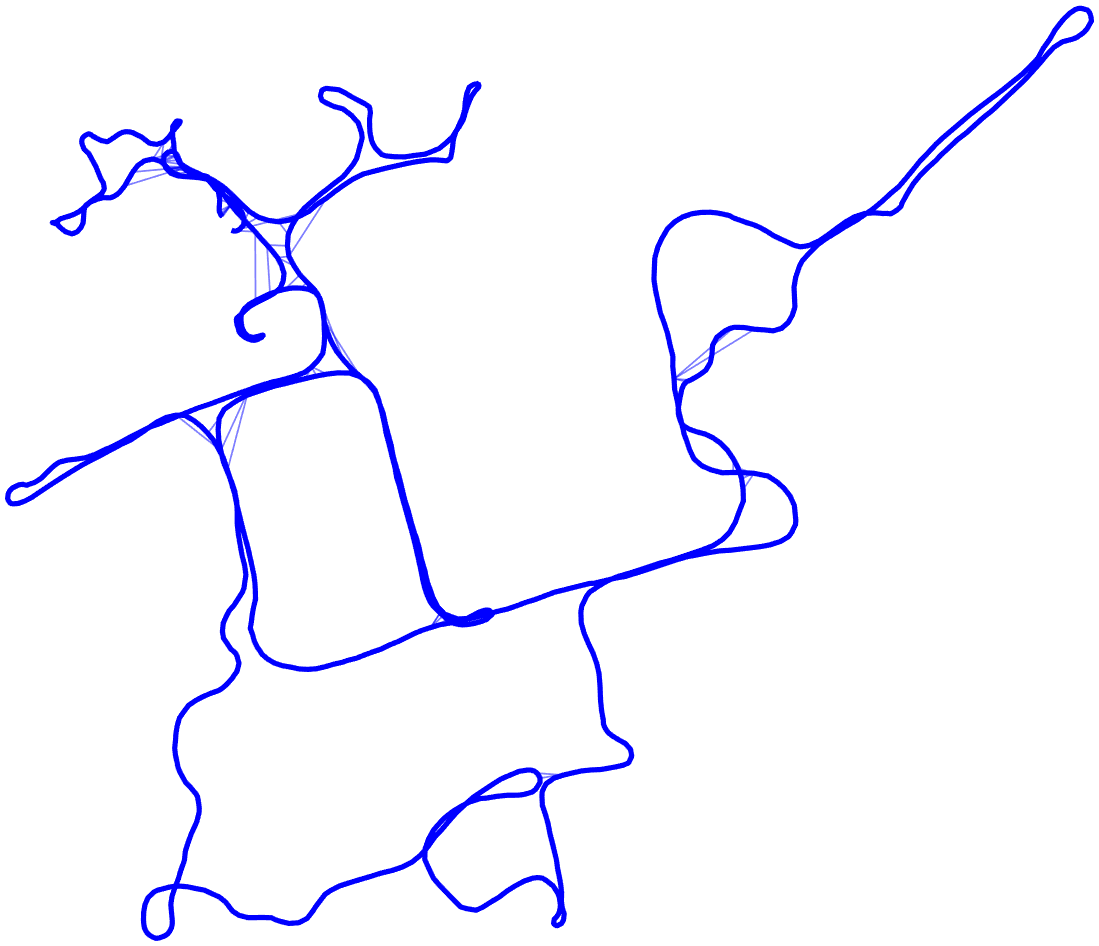}} &
		\hspace{-0.8em}	\subfloat[][\sf{\fontsize{7.35}{8}\selectfont intel}]{\includegraphics[trim =0mm 0mm 0mm 0mm,width=0.21\textwidth]{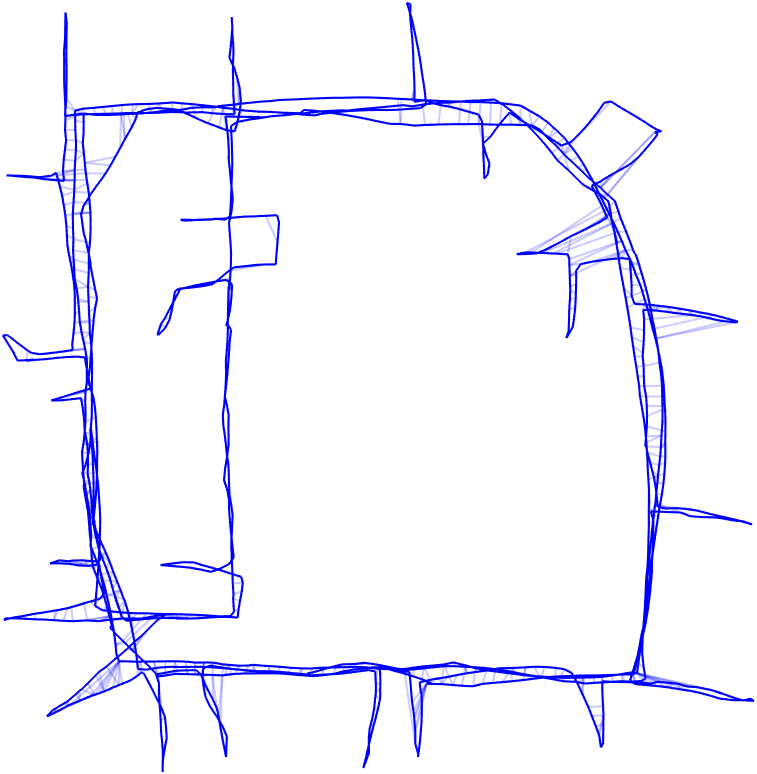}}\\[0.75em]
		\hspace{-0.25em}\subfloat[][\sf\scriptsize M3500]{\includegraphics[trim =0mm 0mm 0mm 0mm,width=0.23\textwidth]{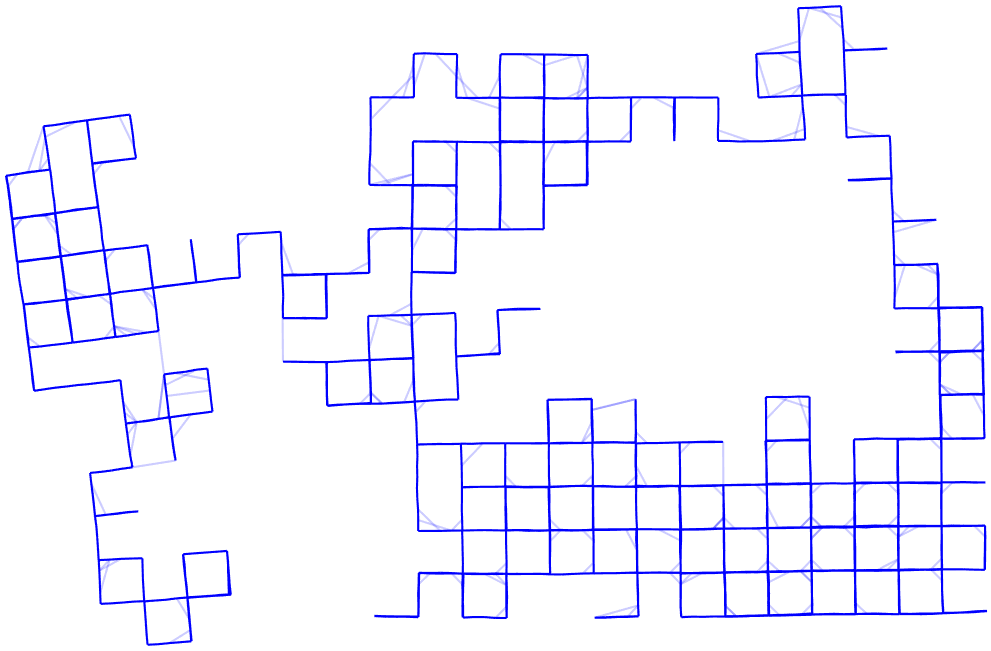}} &
		\hspace{-0.25em}\subfloat[][\sf\scriptsize M3500-a]{\includegraphics[trim =0mm 0mm 0mm 0mm,width=0.23\textwidth]{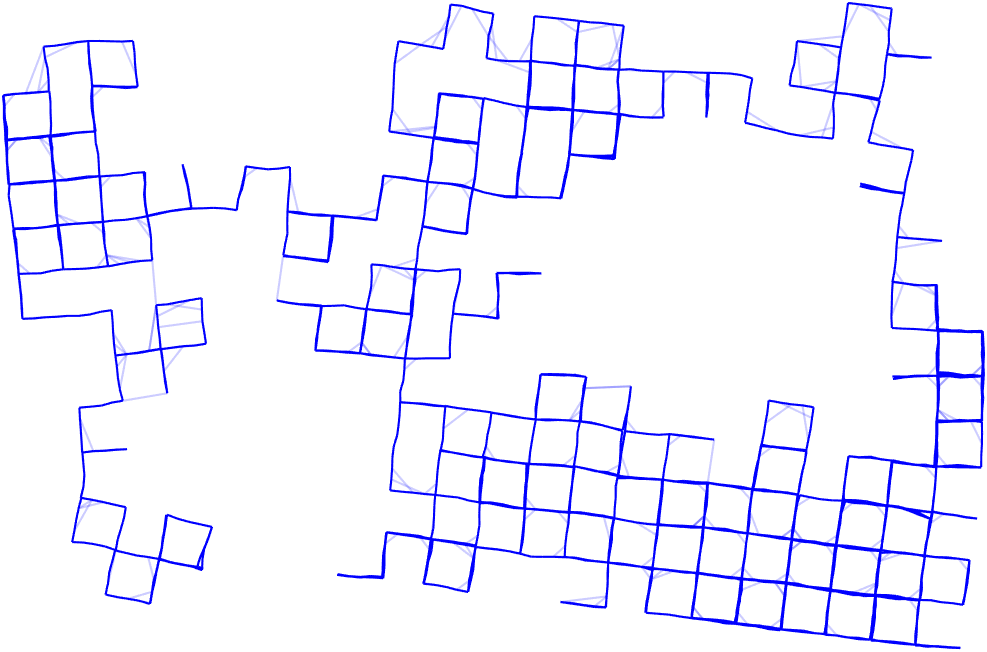}} &
		\hspace{-0.25em}\subfloat[][\sf\scriptsize M3500-b]{\includegraphics[trim =0mm 0mm 0mm 0mm,width=0.23\textwidth]{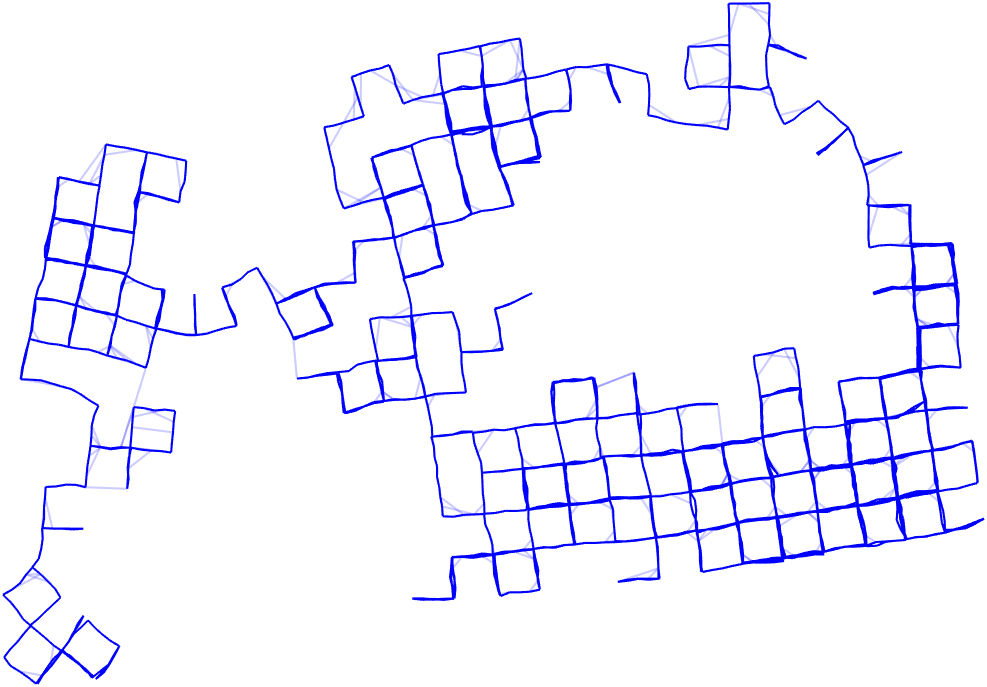}} &
		\hspace{-0.25em}\subfloat[][\sf\scriptsize M3500-c]{\includegraphics[trim =0mm 0mm 0mm 0mm,width=0.23\textwidth]{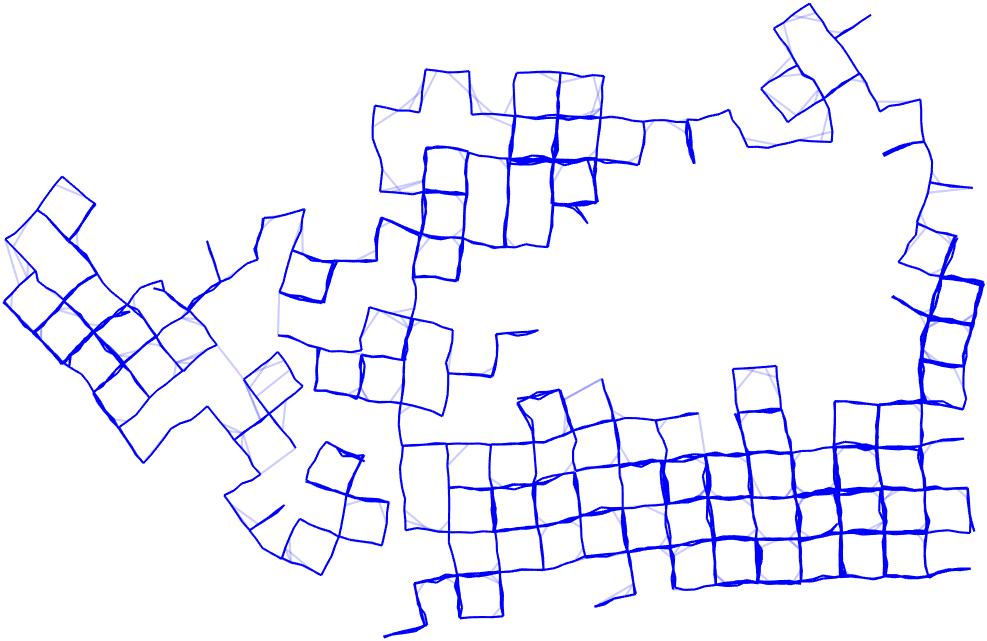}} \\
		\hspace{-1.5em}	\subfloat[][\sf{\fontsize{7.35}{8}\selectfont FR-079}]{\includegraphics[trim =0mm 0mm 0mm 0mm,width=0.195\textwidth]{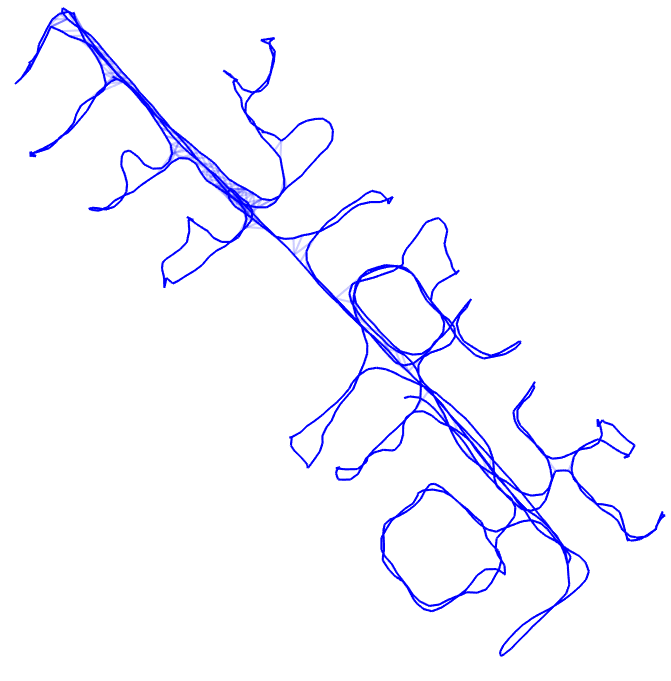}}&
		\hspace{-0.2em}	\subfloat[][\sf{\fontsize{7.35}{8}\selectfont MIT}]{\includegraphics[trim =0mm 0mm 0mm 0mm,width=0.20\textwidth]{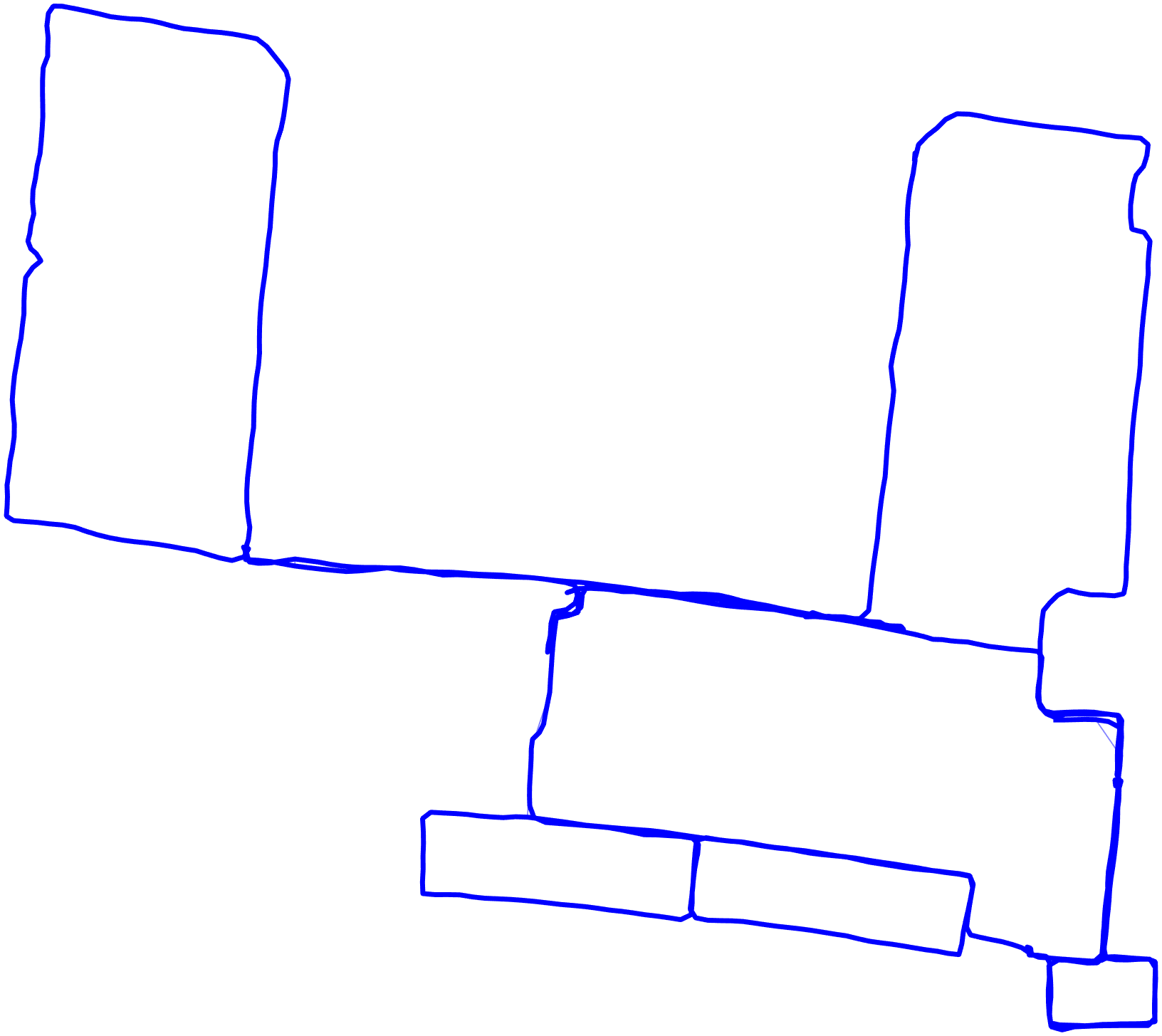}}&
		\hspace{-0.2em}	\subfloat[][\sf{\fontsize{7.35}{8}\selectfont tree10000}]{\includegraphics[trim =0mm 0mm 0mm 0mm,width=0.19\textwidth]{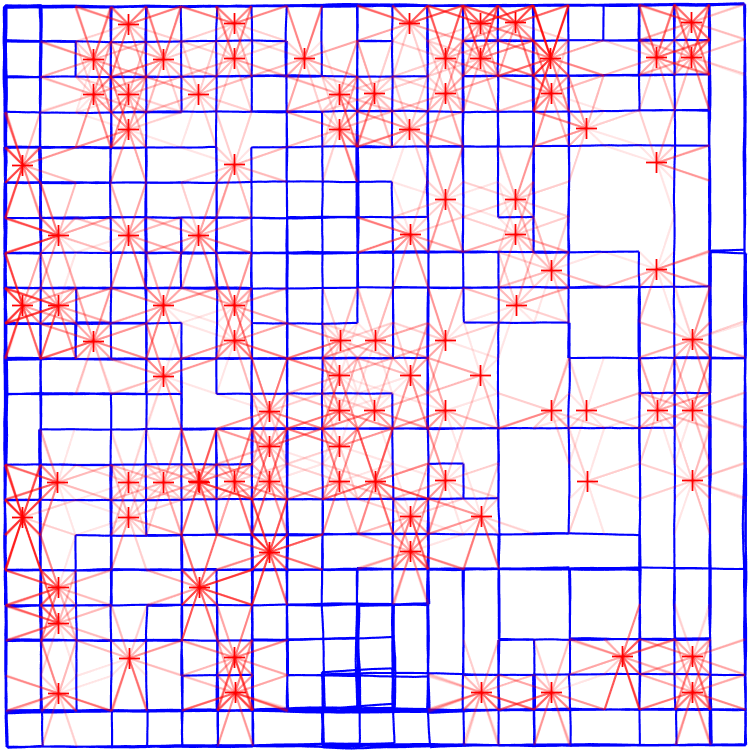}} &
		\hspace{-0.8em} \subfloat[][\sf{\fontsize{7.35}{8}\selectfont victoria-park}]{\includegraphics[trim =0mm 0mm 0mm 0mm,width=0.22\textwidth]{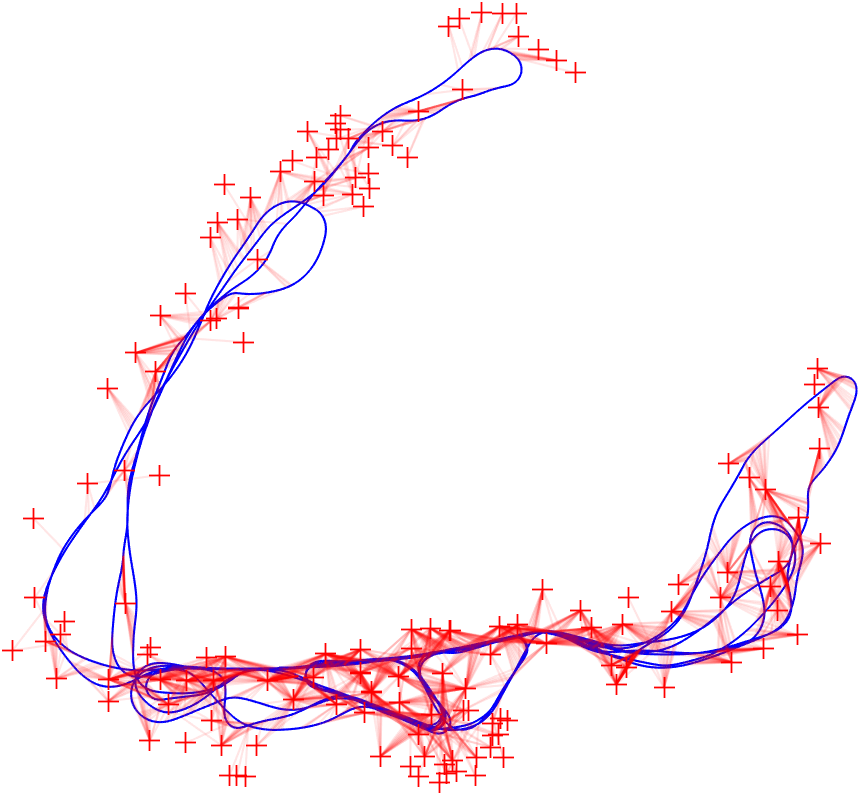}}
	\end{tabular}
	\caption{The globally optimal results of CPL-SLAM on 2D SLAM benchmark datasets. It should be noted that CPL-SLAM still obtains global optima on {\sf\small M3500-a}, {\sf\small M3500-b} and {\sf\small M3500-c} in (f)-(g), which respectively has large extra noise with standard deviations of $0.1$ rad, $0.2$ rad and $0.3$ rad added to the rotational measurements of {\sf\small M3500} in (e). For {\sf\small tree10000} in (k) and {\sf\small victoria-park} in (l) with landmarks, we denote the positions of landmarks with red ``$+$''.}
	\label{fig::map} 
\end{figure*}

\subsection{SLAM Benchmark Datasets}
In this subsection, we implement CPL-SLAM, SE-Sync and PDL-GN on a variety of 2D SLAM benchmark datasets with and without landmarks. In these datasets, {\small\textsf{city10000}}, {\small\textsf{M3500}}, {\small\textsf{M3500-a}}, {\small\textsf{M3500-b}}, {\small\textsf{M3500-c}} and {\small\textsf{tree10000}} are simulated benchmark datasets while the others, i.e., {\small\textsf{ais2klinik}}, {\small\textsf{CSAIL}}, {\small\textsf{intel}}, {\small\textsf{FR-079}}, {\small\textsf{MIT}} and {\small\textsf{victoria-park}}, are real-world datasets. In addition,  {\small\textsf{tree10000}} and  {\small\textsf{victoria-park}} have positions of observed landmarks involved. The chordal initialization \cite{carlone2015initialization} is used for all the benchmark datasets tested.

For all the 2D SLAM benchmark datasets, CPL-SLAM, SE-Sync and PDL-GN converge to the globally optimal solution. The results are shown in \cref{table::comp}, in which $n$ is the number of unknown poses and $n'$ is the number of observed landmarks, $m$ is the number of pose-pose measurements and $m'$ is the number of pose-landmark measurements, $f^*$ is the globally optimal objective value, and the total time accounts for all the time taken to solve graph-based SLAM and the RTR time only accounts for the time taken by the RTR method to solve Riemannian staircase optimization. A specific comparison of SE-Sync and CPL-SLAM is further shown in \cref{fig::time}. From \cref{table::comp,fig::time}, it can be seen that CLP-Sync is significantly faster than both SE-Sync and PDL-GN on all the SLAM benchmark datasets, in which CPL-SLAM outperforms PDL-GN by a factor of 5.53 on average for the overall computation, and outperforms SE-Sync by a factor of $2.87$ and $2.51$ on average for the computation of the RTR method and the overall computation, respectively. In particular, CPL-SLAM obtains a further improved performance of the RTR method over SE-Sync on the datasets with landmarks, and we think it is due to the conciseness of the complex number representation whose resulting preconditioner accelerates the truncated conjugate gradient method that is used in the RTR method to evaluate the descent direction.

The globally optimal results of CPL-SLAM on these 2D SLAM benchmark datasets are as shown in \cref{fig::map}. It should be noted that \textsf{\small M3500-a}, \textsf{\small M3500-b} and \textsf{\small M3500-c} in \cref{fig::map}f-\cref{fig::map}h respectively have extra Gaussian noise with standard deviation $0.1$ rad, $0.2$ rad and $0.3$ rad added to the rotational measurements of \textsf{\small M3500} \cite{carlone2016planar} in \cref{fig::map}e, which indicates that CPL-SLAM can tolerate noisy measurements that are orders of magnitude greater than real-world SLAM applications. For {\sf\small tree10000} in \cref{fig::map}k and {\sf\small victoria-park} in \cref{fig::map}l with landmarks, we denote the positions of landmarks with red ``$+$''.

\section{Conclusion}\label{section::conclusion}
In this paper, we have presented CPL-SLAM that is a certifiably correct algorithm for planar graph-based SLAM using the complex number representation. By leveraging the complex semidefinite programming and Riemannian staircase optimization on complex oblique manifolds, CPL-SLAM is applicable to planar graph-based SLAM with and without landmarks. In addition, even though CPL-SLAM essentially solves the complex semidefinite relaxation,  we prove that CPL-SLAM exactly retrieves the globally optimal solution to planar graph-based SLAM as long as the noise magnitude is below a certain threshold. 

CPL-SLAM is compared with the state-of-the-art methods SE-Sync \cite{rosen2016se} and  Powell's Dog-Leg \cite{dellaert2012factor,rosen2014rise} on the simulated {\small\textsf{Tree}} datasets,  the simulated {\small\textsf{City}} datasets  and numerous large 2D simulated and real-world SLAM benchmark datasets in terms of scalability and robustness. The results of the data experiments indicate that CPL-SLAM is capable of solving planar graph-based SLAM certifiably, and more importantly, is more efficient in numerical computation and more robust to measurement noise. As a result, we expect that CPL-SLAM outperforms existing state-of-the-art methods to planar graph-based SLAM.
 
There is still great potential for improvements of CPL-SLAM in several aspects. A fully distributed extension of CPL-SLAM is definitely beneficial to multi-robot simultaneous localization and mapping. In spite of being able to tolerate large measurement noise, CPL-SLAM still needs to enhance its robustness to measurement outliers. At last, it is currently assumed that the positions of landmarks are fully known in CPL-SLAM, and we hope that in the future  CPL-SLAM can handle range-only and bearing-only measurements of landmarks, which is another important extension.

\section*{Acknowledgment}
This material is based upon work supported by the National Science Foundation under award DCSD-1662233. Any opinions, findings, and conclusions or recommendations expressed in this material are those of the authors and do not necessarily reflect the views of the National Science Foundation.

\section*{Appendix A.\; The Derivation of \eqref{eq::QP}}\label{app::A}
In this section, we derive \eqref{eq::QP} following a similar procedure of \cite[Appendix B]{rosen2016se} even though ours uses the complex number representation and has landmarks involved.

It is straightforward to rewrite \eqref{eq::LSP} as 
\begin{equation}
\min_{\substack{\lmk\in \C^n,\,\\\tran\in \C^n,\,\rot\in \C_1^n}} \left\|B\begin{bmatrix}
\lmk \\
\tran \\
\rot
\end{bmatrix}\right\|_2^2
\end{equation}
in which 
\begin{equation}
\nonumber
B\triangleq\begin{bmatrix}
B_1 & B_2\\
\0  & B_3
\end{bmatrix}\in \C^{(2m+m')\times (2n+n')}.
\end{equation}
Here $B_{1}\in \R^{(m+m')\times (n+n')}$, $B_{2}\in \R^{(m+m')\times n}$ and $B_3\in \C^{m\times n}$ are given as
\begin{subequations}
\begin{equation}\label{eq::b1}
[B_{1}]_{ek}=\begin{cases}
\sqrt{\nu_{ik}},\hphantom{\tilde{\tran}_{kj}-} & e=(i,\,k)\in\aEEp,\\
-\sqrt{\nu_{kj}},\hphantom{\tilde{\tran}_{ik}} & e=(k,\,j)\in\aEEp,\\
\sqrt{\tau_{ik}}, &  e=(i,\,k)\in\aEE,\\
-\sqrt{\tau_{kj}},\hphantom{\tilde{\tran}_{kj}} & e=(k,\,j)\in\aEE,\\
0, & \text{otherwise},
\end{cases}
\end{equation}
\begin{equation}\label{eq::b2}
[B_{2}]_{ek}=\begin{cases}
-\sqrt{\nu_{kj}}\tilde{\lmk}_{kj}, & e=(k,\,j)\in\aEEp,\\
-\sqrt{\tau_{kj}}\tilde{\tran}_{kj}, & e=(k,\,j)\in\aEE,\\
0, & \text{otherwise},
\end{cases}
\end{equation}
and
\begin{equation}\label{eq::b3}
[B_3]_{ek}=\begin{cases}
-\sqrt{\kappa_{kj}}\tilde{\rot}_{kj}, & e=(k,\,j)\in\aEE,\\
\sqrt{\kappa_{ik}}, &  e=(i,\,k)\in\aEE,\\
0, & \text{otherwise},
\end{cases}
\end{equation}
\end{subequations}
respectively. Since \eqref{eq::P} is also equivalent to \eqref{eq::qpp}, it can be concluded that
\begin{subequations}
\begin{align}
&B_1^HB_1=\Lambda,\label{eq::B1}\\
&B_1^HB_2=\widetilde{\Theta},\label{eq::B2}\\
&B_2^HB_2=\widetilde{\Sigma}^\rot,\label{eq::B3}\\
&B_3^HB_3=L(\widetilde{G}^\rot)\label{eq::B4}
\end{align}
\end{subequations}
in which $\Lambda$, $\widetilde{\Theta}$, $\widetilde{\Sigma}$ and $L(\widetilde{G}^\rot)$ are as defined in \eqref{eq::qpp}. If we let $\tM^\sigma \triangleq \widetilde{\Sigma}^\rot-\widetilde{\Theta}^H\Lambda^\dagger\widetilde{\Theta}$, then from \cref{eq::B1,eq::B2,eq::B3,eq::B4}, we obtain
\begin{equation}\label{eq::Qt}
\begin{aligned}
\tM^\sigma &=B_2^HB_2-B_2^HB_1(B_1^HB_1)^\dagger B_1^HB_2\\
&= B_2^H\left(\I - B_1(B_1^HB_1)^\dagger B_1^H\right)B_2,
\end{aligned}
\end{equation}
in which $B_1$, $B_2$ and $B_3$ are defined as \cref{eq::b1,eq::b2,eq::b3}.
It should be noted that we might rewrite $B_1$ and $B_2$ as
\begin{equation}\label{eq::B12}
B_1 = \Omega^{\frac{1}{2}}A^\transpose ,\quad\quad B_2=\Omega^\frac{1}{2}\widetilde{T},
\end{equation}
in which $A\triangleq A(\aGG)$ and $\widetilde{T}$ are given by Eqs. \eqref{eq::A} and \eqref{eq::T}, respectively.
Substituting \cref{eq::B12} into \cref{eq::Qt}, we obtain
\begin{equation}
\begin{aligned}
\tM^\sigma &= B_2^H\left(\I - B_1(B_1^HB_1)^\dagger B_1^H\right)B_2\\
&=\widetilde{T}^H\Omega^\frac{1}{2}\Pi\Omega^\frac{1}{2}\widetilde{T},
\end{aligned}
\end{equation}
in which 
$$\Pi=\I - \Omega^{\frac{1}{2}}A^\transpose \left(A\Omega A^\transpose \right)^\dagger A\Omega^\frac{1}{2}\in\R^{(m+m')\times (m+m')}.$$ 
As a result, it can be concluded that
$$\tM = L(\widetilde{G}^\rot)+\tM^\sigma=L(\widetilde{G}^\rot)+\widetilde{T}^H\Omega^{\frac{1}{2}}\Pi\Omega^{\frac{1}{2}}\widetilde{T}.$$ 
Furthermore, it is known that $X^\transpose (XX^\transpose )^\dagger = X^\dagger$ for any matrix $X$, then we further obtain
\begin{equation}
\begin{aligned}
\Pi&=\I - \Omega^{\frac{1}{2}}A^\transpose \left(A\Omega A^\transpose \right)^\dagger A\Omega^\frac{1}{2}\\
&=\I -\left(A\Omega^\frac{1}{2}\right)^\dagger A\Omega^\frac{1}{2},
\end{aligned}
\end{equation}
which according to \cite[Chapter 5.13]{meyer2000matrix} is the matrix of orthogonal projection operator $\pi:\C^{m+m'}\rightarrow \ker(A(\aGG)\Omega^{\frac{1}{2}})$ onto the kernel space of $A\Omega^\frac{1}{2}$. In addition, similar to \cite[Appendix B.2]{rosen2016se},  it is possible to further decompose $\Pi$ in terms of sparse matrices and their inverse for efficient computation even though $\Pi$ is in general a dense matrix. 

\section*{Appendix B.\; Proofs of the Lemmas and Propositions in Section V}\label{app::B}
In this section, we present proofs of the lemmas and propositions in \cref{section::sdp}. These proofs draw heavily on \cite{absil2009optimization} and follows a similar procedure to that of \cite[Appendix C]{rosen2016se} and \cite[Section 4.3]{bandeira2017tightness}.
\subsection*{B.1. Proof of \cref{lemma::qp}}\label{app::B1}
It is known that the unconstrained Euclidean gradient of $F(\rot)\triangleq \rot^H\tM \rot$ is $\nabla F(\rot)=2\tM\rot$, and thus, if we let $S(\rot)\triangleq \tM-\Re\{\ddiag{(\tM\rot\rot^H)}\} $, the Riemannian gradient is
\begin{equation}\label{eq::gradf}
\begin{aligned}
\mathrm{grad}\,F(\rot)&=\mathrm{proj}_x(\nabla F(\rot))\\
&=2(\tM-\Re\{\ddiag{(\tM\rot\rot^H)}\})\rot\\
&=2 S(\rot)\rot,
\end{aligned}
\end{equation}
in which the linear projection operator $\mathrm{proj}_z:\C^n\rightarrow T_z\C_1^n$ is defined to be
$$\mathrm{proj}_zu=u-\Re\{\ddiag(uz^H)\}z.$$
In addition, it should be noted that we have assumed that $\C_1^n$ is a Riemannian submanifold of Euclidean space,  then the Riemannian Hessian is
\begin{equation}\label{eq::hess}
\!\!\mathrm{Hess}\, F(\rot)[\dot{\rot}]=\mathrm{proj}_z \mathrm{D}\,\mathrm{grad}\, F(\rot)[\dot{\rot}]\\=\mathrm{proj}_z 2S(\rot)\dot{\rot},
\end{equation}
in which $\mathrm{D}\,\mathrm{grad}\, F(\rot)[\dot{\rot}]$ is the direction derivative of $\mathrm{grad}\, F(\rot)$ along direction $\dot{\rot}$.
From \cref{eq::hess}, we obtain
\begin{equation}
\nonumber
\innprod{\mathrm{Hess}\, F(\rot)[\dot{\rot}]}{\dot{\rot}}=2\innprod{S(\rot)\dot{\rot}}{\dot{\rot}}.
\end{equation}
Moreover, according to \cite[Chapter 5]{absil2009optimization}, if $\exp_z: T_z \C_1^n\rightarrow \C_1^n$ is the exponential map at $\rot\in\C_1^n$, we obtain
\begin{equation}
\nonumber
\left.\frac{\d}{\d t} F\circ \mathrm{exp}_z(t\dot{\rot})\right|_{t=0}=\innprod{\mathrm{grad}\,F(\rot)}{\dot{\rot}}
\end{equation} 
and
\begin{equation}
\nonumber
\left.\frac{\d^2}{\d t^2} F\circ \mathrm{exp}_z(t\dot{\rot})\right|_{t=0}=\innprod{\mathrm{Hess}\,F(\rot)[\dot{\rot}]}{\dot{\rot}}.
\end{equation}
Therefore, if $\hat{\rot}\in\C_1^n$ is a local optimum for \cref{eq::QP} and $\hat{S}=S(\hat{\rot})$, it is required that $\hat{S}\hat{\rot} =0$ and $\innprod{\dot{\rot}}{\hat{S}\dot{\rot}}\geq 0$ for all $\dot{\rot}\in T_x \C_1^n$, which completes the proof.

\subsection*{B.2. Proof of \cref{lemma::sdp}}\label{app::B2}
It should be noted that (1) to (5) in \cref{lemma::sdp} are KKT conditions of \eqref{eq::SDP}, which proves the necessity. Since the identity matrix $\I\in\C^{n\times n}$ is strictly feasible to \cref{lemma::sdp}, the Slater's condition is satisfied, which proves the sufficiency. In addition, it should be noted that the Slater's condition also holds for the dual of \eqref{eq::SDP}. If $\rank(\hat{S})=n-1$, according to \cite[Theorem 6]{alizadeh1997complementarity},  $\hat{S}$ is dual nondegenerate. Moreover, by complementary slackness, $\hat{S}$ is also optimal for the dual of \eqref{eq::SDP}, which, as a result of \cite[Theorem 10]{alizadeh1997complementarity}, implies that $\hat{X}$ is unique. If $\rank(\hat{S})=n-1$,   it can be concluded that $\hat{X}$ has rank one from $\hat{S}\hat{X}=\0$.

\subsection*{B.3. Proof of \cref{lemma::qpsdp}}\label{app::B3}
Since $\hat{\rot}\in\C_1^n$ is a first-order critical point and $\hat{S}\succeq 0$, we conclude that $\hat{\rot}$ is a second-order critical point from \cref{lemma::qp}. Also it can be checked that $\hat{X}=\hat{\rot}^H\hat{\rot}\in\H^n$ satisfies (1) to (5) in \cref{lemma::sdp}, thus, $\hat{\rot}$ solves \eqref{eq::QP}, and $\hat{X}$ solves \eqref{eq::SDP} and is the unique global optimum for \eqref{eq::SDP} if $\rank(\hat{S})=n-1$. 

\subsection*{B.4. Proof of \cref{prop::sdp}}\label{app::B4}
In order to prove \cref{prop::sdp}, we need \cref{prop::Q,prop::E} as follows.
\begin{prop}\label{prop::Q}
If $\underline{M}\in \H^n$ is data matrix of the form \cref{eq::Q} that is constructed with the true (latent) relative measurements, and $\underline{\rot}\in \C_1^n$ is the true (latent) value of rotational states $\rot$, then $\underline{M}\,\underline{\rot}=\0$ and $\lambda_2(\underline{M})>0$. 
\end{prop}
\begin{proof}
For consistency, we assume that \eqref{eq::P} and \eqref{eq::QP} are formulated with the true (latent) relative measurements. Let $\underline{\lmk}\in \C^{n'}$ and $\underline{\tran}\in \C^n$ be the true (latent) value of landmark positions and translational states $\tran$, respectively, then $\underline{\xi}=\begin{bmatrix}
\underline{\lmk}^\transpose  & \underline{\tran}^\transpose  & \underline{\rot}^\transpose 
\end{bmatrix}^\transpose \in \C^{n'}\times \C^n\times \C_1^n$ solves \eqref{eq::P}, and the optimal objective value is $0$. Since \eqref{eq::QP} is equivalent to \eqref{eq::P}, it can be concluded that $\underline{\rot}\in \C_1^n$ solves \eqref{eq::QP}, and the optimal objective value of \eqref{eq::QP} is $0$ as well. Furthermore, since $\underline{M}\succeq 0$, we obtain $\underline{M}\,\underline{\rot}=\0$. Let $\Xi\triangleq\diag\{\underline{\rot}{}_1,\,\cdots,\,\underline{\rot}_n\}\in \C^{n\times n}$ and $L(W^\rot)\in \R^{n\times n}$ be the Laplacian such that
\begin{equation}
\nonumber
[L(W^\rot)]_{ij}\triangleq
\begin{cases}
\sum\limits_{(i,k)\in \mathcal{E}} \kappa_{ik},\hphantom{\tilde{\tran}_{ik}} & i=j,\\
-\kappa_{ij}, & (i,j)\in \mathcal{E},\\
0 & \mathrm{otherwise},
\end{cases}
\end{equation}
we obtain $L(\underline{G}^\rot)= \Xi L(W^\rot)\Xi^H$. It should be noted that $G$ is assumed to be connected, and as a result, $\lambda_2(L(\underline{G}^\rot))>0$ and $L(\underline{G}^\rot) \underline{\rot}=\0 $. Furthermore, it is by the definition of $\underline{M}$ or $\tM$ in \cref{eq::QP0} that
\begin{equation}\label{eq::Ml}
\underline{M} = L(\underline{G}^\rot) + \underline{M}^\sigma,
\end{equation} 
in which $\underline{M}^\sigma =\underline{\Sigma}^\rot-\underline{\Theta}^H\Lambda^\dagger\underline{\Theta}$. From \cref{eq::B1,eq::B2,eq::B3} and \eqref{eq::Qt}, we obtain that $\underline{M}^\sigma$ is the Schur complement of 
$$\begin{bmatrix}
B_1^HB_1 & B_1^HB_2\\
B_2^HB_1 & B_2^HB_2
\end{bmatrix}=
\begin{bmatrix}
B_1^H\\
B_2^H
\end{bmatrix}\begin{bmatrix}
B_1&
B_2
\end{bmatrix}\succeq 0,$$ 
which suggests that $\underline{M}^\sigma\succeq 0$ and $\lambda_1(\underline{M}^\sigma)\geq 0$. As a result of \cref{eq::Ml}, $\lambda_2(L(\underline{G}^\rot))>0$ and $\lambda_1(\underline{M}^\sigma)\geq 0$, we obtain 
$$\lambda_2(\underline{M})\geq \lambda_2(L(\underline{G}^\rot))+\lambda_1(\underline{M}^\sigma)>0,$$
which completes the proof. 
\end{proof}

\begin{prop}\label{prop::E}
	If $\underline{\rot}\in\C_1^n$ is the true (latent) value of $\rot\in \C_1^n$, and $\hat{\rot}$ solves \eqref{eq::QP}, and $d(\underline{\rot},\,\hat{\rot})\triangleq\min\limits_{\theta\in \R}\|\hat{\rot}-e^{\i\theta}\underline{\rot}\|$, then we obtain
\begin{equation}\label{eq::dnoise}
d(\underline{\rot},\,\hat{\rot})\leq 2\sqrt{\frac{n\|\tM-\underline{M}\|_2}{\lambda_2(\underline{M})}}
\end{equation}
\end{prop}
\begin{proof}
If we define $\Delta M\triangleq\tM-\underline{M}\in \H^n$ to be the perturbation matrix, then
\begin{equation}\label{eq::e1}
\underline{\rot}^H\tM\underline{\rot} = \underline{\rot}^H\underline{M}\,\underline{\rot} + \underline{\rot}^H\Delta M\underline{\rot} = \underline{\rot}^H\Delta M\underline{\rot}
\leq n\|\Delta M\|_2,
\end{equation}
in which, according to \cref{prop::Q}, $\underline{\rot}^H\underline{M}\,\underline{\rot}=0$. In addition, it should be noted that
\begin{equation}\label{eq::e3}
\underline{\rot}^H\tM\underline{\rot}\geq \hat{\rot}^H\tM\hat{\rot}
\end{equation}
and
\begin{equation}\label{eq::e2}
\hat{\rot}^H\tM\hat{\rot}=\hat{\rot}^H\underline{M}\hat{\rot}+\hat{\rot}^H\Delta {M}\hat{\rot}\geq \hat{\rot}^H\underline{M}\hat{\rot}- n\|\Delta M\|_2.
\end{equation}
From \cref{eq::e1,eq::e2,eq::e3}, we obtain 
\begin{equation}\label{eq::e4}
2n\|\Delta M\|_2\geq \hat{\rot}^H\underline{M}\hat{\rot}
\end{equation}
From \cref{prop::Q}, we obtain
\begin{equation}\label{eq::use}
\begin{aligned}
\hat{\rot}^H\underline{M}\hat{\rot} &= (\hat{\rot}-\frac{1}{n}\underline{\rot}^H\hat{\rot}\underline{\rot})^H\underline{M}(\hat{\rot}-\frac{1}{n}\underline{\rot}^H\hat{\rot}\underline{\rot})\\
&\geq \lambda_2(\underline{M})\|\hat{\rot}-\frac{1}{n}\underline{\rot}^H\hat{\rot}\underline{\rot}\|^2,
\end{aligned}
\end{equation}
in which the equality ``$=$'' uses $\underline{M}\,\underline{\rot}=\0$ and the inequality ``$\geq$'' uses  $\lambda_2(\underline{M})>\lambda_1(\underline{M})=0$ and $\underline{z}^H(\hat{\rot}-\frac{1}{n}\underline{\rot}^H\hat{\rot}\underline{\rot})=0.$
Furthermore, an algebraic manipulation indicates that
\begin{equation}\label{eq::norm}
\|\hat{\rot}-\frac{1}{n}\underline{\rot}^H\hat{\rot}\underline{\rot}\|^2=n-\frac{1}{n}|\underline{z}^H\hat{z}|^2.
\end{equation} 
From \cref{eq::norm,eq::use}, we obtain
\begin{equation}\label{eq::e5}
\begin{aligned}
\hat{\rot}^H\underline{M}\hat{\rot} &\geq \lambda_2(\underline{M})\|\hat{\rot}-\frac{1}{n}\underline{\rot}^H\hat{\rot}\underline{\rot}\|^2\\
& = \frac{1}{n}\lambda_2(\underline{M})(n^2-|\underline{\rot}^H\hat{\rot}|^2)\\
& \geq \lambda_2(\underline{M})(n-|\underline{\rot}^H\hat{\rot}|),
\end{aligned}
\end{equation}
in which the last inequality ``$\geq$'' uses the Cauchy-Schwarz inequality
$$|\underline{z}^H \hat{z}|\leq ||\underline{z}||\cdot\|\hat{z}\|=n.$$
Substituting \cref{eq::e5} into \cref{eq::e4} and simplifying the resulting equation, we obtain 
\begin{equation}\label{eq::e6}
n-|\underline{\rot}^H\hat{\rot}| \leq \frac{2n\|\Delta{M}\|_2}{\lambda_2(\underline{M})}.
\end{equation}
In addition, from \cite[Eq. (4.1)]{bandeira2017tightness}, it is known that $d(\underline{\rot},\,\hat{\rot})=\sqrt{2n-2|\underline{\rot}^H\hat{\rot}|}$, and then from \cref{eq::e6}, we further obtain \cref{eq::dnoise}, which completes the proof. 
\end{proof}

To prove \cref{prop::sdp}, we first decompose $\hat{S}=\tM-\Re({\ddiag(\tM\hat{\rot}^H\hat{\rot})})$ as follows:
\begin{equation}
\nonumber
\begin{aligned}
\hat{S}=&\tM-\Re({\ddiag(\tM\hat{\rot}^H\hat{\rot})})\\
=&\underline{M}+\Delta M-\\
&\hspace{3.8em}\Re\left\{\ddiag\left((\underline{M}+\Delta M)(\underline{\rot}+\Delta \rot)(\underline{\rot}+\Delta \rot)^H\right)\right\}\\
=&\underline{M}+\Delta M-\Re\big\{\ddiag(\underline{M}\Delta \rot \rot^H+\underline{M}\Delta \rot \Delta \rot^H+\\
&\hspace{2em}\underbrace{\hspace{9em}\Delta M(\underline{\rot}+\Delta \rot)(\underline{\rot}+\Delta \rot)^H)\big\}}_{\Delta S},
\end{aligned}
\end{equation}
in which $\underline{\rot}\in\C_1^n$ is the true (latent) value of $\rot\in \C_1^n$ such that $\underline{M}\,\underline{\rot}=\0$, and $\hat{\rot}$ solves \eqref{eq::QP}, and $\Delta \rot\triangleq \hat{\rot}-\underline{\rot}$. In addition, we assume $\|\hat{\rot}-\underline{\rot}\|=d(\hat{\rot},\,\underline{\rot})\triangleq \min\limits_{\theta\in \R}\|\hat{\rot}-e^{\i\theta}\underline{\rot}\|$. It is obvious that $\|\Delta S\|_2\rightarrow 0$ as long as $\|\Delta M\|_2\rightarrow 0$ and $\|\Delta \rot\|\rightarrow 0$, and by \cref{prop::E}, we obtain $\|\Delta \rot\|\rightarrow 0$ as long as $\|\Delta M\|_2\rightarrow 0$. As a result, from continuity, there exists some $\gamma > 0$ such that $\|\Delta S\|_2< \lambda_2(\underline{M})$ as long as $\|\Delta M\|_2 < \gamma$. Then we obtain
\begin{equation}
\nonumber
\lambda_i(\hat{S})\geq \lambda_i(\underline{M})-\|\Delta S\|_2 > \lambda_i(\underline{M})-\lambda_2(\underline{M})\geq 0
\end{equation}
for all $i\geq 2$, which implies that $\hat{S}$ at least has $n-1$ positive eigenvalues. In addition, by \cref{lemma::qp}, we obtain $\hat{S}\hat{\rot}=\0$, from which it can be concluded that $\hat{S}\succeq 0$ and $\rank(\hat{S})=n-1$. Furthermore, \cref{lemma::qpsdp} guarantees that $\hat{X}=\hat{\rot}\hat{\rot}^H\in\H^n$ is the unique optimum of \eqref{eq::SDP} if $\hat{S}\succeq 0$ and $\rank(\hat{S})=n-1$.

\section*{Appendix C.\; The Preconditioner For The Truncated Conjugate Gradient Method in CPL-SLAM}\label{app::C}

In this section, we propose a preconditioner for the truncated conjugate gradient method in CPL-SLAM that targets for planar graph-based SLAM with landmarks.

Similar to \cite{rosen2016se,briales2017cartan,rosen2017computational}, instead of factorizing the Riemannian Hessian matrix $\mathrm{Hess}\,F(\rot)\in \C^{n\times n}$ in \cref{eq::hess} to explicitly evaluate the descent direction, the Riemannian trust region (RTR) method in CPL-SLAM leverages the truncated conjugated gradient (TCG) method to approximate the descent direction with necessary accuracy. Though the TCG method is guaranteed to converge to the true solution within finite iterations, the rates of the convergence is closely related with the preconditioner $\mathrm{Precon}(\mathrm{Hess}\,F(\rot))\in \C^{n\times n}$ that is used to approximate $\mathrm{Hess}\,F(\rot)$ and iteratively solve
\begin{equation}
\nonumber
\mathrm{Precon}(\mathrm{Hess}\,F(\rot))\cdot a = b
\end{equation}
to evaluate the descent direction, in which $a,\,b\in \C^n$. 

In graph-based SLAM, several preconditioners have been proposed for the TCG method \cite{briales2017cartan,rosen2017computational}. For CPL-SLAM, an immediate choice of the preconditioner $\mathrm{Precon}(\mathrm{Hess}\,F(\rot))$ is $L(\widetilde{G}^\rot)+\widetilde{\Sigma}^\rot$, which is the submatrix of $\widetilde{\Gamma}$ in \cref{eq::problem} that corresponds to the rotational states,\footnote{A similar preconditioner is also used in SE-Sync \cite{rosen2017computational}.} and  such a preconditioner  works well for planar graph-based SLAM without landmarks. However, the preconditioner of $L(\widetilde{G}^\rot)+\widetilde{\Sigma}^\rot$ suffers slow convergence for planar graph-based SLAM with landmarks since the submatrix $L(\widetilde{G}^\rot)+\widetilde{\Sigma}^\rot$ loses information of pose-landmark measurements and results in a bad approximation of  $\mathrm{Hess}\,F(\rot)$.

In contrast to $L(\widetilde{G}^\rot)+\widetilde{\Sigma}^\rot$ that only captures the information of pose-pose measurements, the matrix
\begin{equation}
\tM=L(\widetilde{G}^\rot)+\widetilde{T}^H\Omega^{\frac{1}{2}}\Pi\Omega^{\frac{1}{2}}\widetilde{T}\in \C^{n\times n}
\end{equation}
in \eqref{eq::QP} implicitly but properly keeps the information of both pose-pose measurements and pose-landmark measurements. However, there is no exact expression of $\tM$ and we need to evaluate the equation above to factorize $\tM$. Furthermore, since $\tM$ can be a dense matrix, the resulting factorization of $\tM$ might be inefficient to solve 
\begin{equation}\label{eq::ma}
\tM\cdot a = b,
\end{equation}
which affects the performance of the TCG method. As a result, we need some other methods rather than evaluate and factorize $\tM$ explicitly to solve \cref{eq::ma}. 

It should be noted that the solution to 
\begin{equation}\label{eq::opt0}
\min_{x\in \C^n}\frac{1}{2} \innprod{x}{\tM x} - \innprod{x}{b},
\end{equation}
is also a solution to \cref{eq::ma}. From \cref{eq::problem,eq::qpp,eq::sol_lin,eq::QP0}, it is straightforward to show that \cref{eq::opt0} is equivalent to
\begin{equation}\label{eq::opt1}
\min_{x\in \C^n,\,x'\in\C^{n+n'}}\frac{1}{2} \innprod{\begin{bmatrix}
x'\\
x
	\end{bmatrix}}{\widetilde{\Gamma} \begin{bmatrix}
	x'\\
	x
	\end{bmatrix}} - \innprod{\begin{bmatrix}
	x'\\
	x
	\end{bmatrix}}{\begin{bmatrix}
	\0\\
	b
	\end{bmatrix}},
\end{equation}
and the solution $\begin{bmatrix}
a'\\
a
\end{bmatrix}\in \C^{2n+n'}$  to \cref{eq::opt1} can be computed in closed form as
\begin{equation}\label{eq::sol0}
\widetilde{\Gamma}\begin{bmatrix}
a'\\
a
\end{bmatrix}=\begin{bmatrix}
\0\\
b
\end{bmatrix},
\end{equation}
or equivalently,
\begin{equation}\label{eq::sol1}
\begin{bmatrix}
a'\\
a
\end{bmatrix}=\widetilde{\Gamma}^\dagger \begin{bmatrix}
\0\\
b
\end{bmatrix}.
\end{equation}
It is by definition that $a\in \C^n$ in \cref{eq::sol0,eq::sol1} is also a solution to \cref{eq::opt0,eq::ma}.
Therefore, we might factorize $\widetilde{\Gamma}$ and solve \cref{eq::sol0,eq::sol1} instead so as to solve \cref{eq::ma}. In \cref{eq::QP0,eq::ma,eq::opt0}, we construct the dense data matrix  $\tM$ using \cref{eq::sol_lin} to reduce the dimension of the optimization problem with no information loss. In \cref{eq::sol0,eq::sol1}, on the other hand, we essentially reverse the operation of \cref{eq::sol_lin} to recover the sparse matrix $\widetilde{\Gamma}$  by augmenting the dimension. Since $\widetilde{\Gamma}$ is a sparse matrix whose exact expression requires no extra computation,  it is more efficient to exploit the sparsity of $\widetilde{\Gamma}$ to solve \cref{eq::sol0,eq::sol1} than factorize the dense data matrix $\tM$ to solve \cref{eq::ma,eq::opt0}.

\section*{Appendix D.\; The Complex Oblique Manifold}\label{app::D}
In this section, we give a brief review of the Riemannian geometric concepts and operators of the complex oblique manifold that are used in this paper. A detailed introduction to Riemannian geometry and optimization can be found in \cite{absil2009optimization}.

In Riemannian geometry, the complex oblique manifold
\begin{equation}
\nonumber
\mathrm{OB}(r,n)\triangleq\{Y\in \C^{n\times r}|\ddiag(YY^H)=\I \}
\end{equation}
is a smooth and compact complex matrix manifold, whose tangent space at $Y\in \mathrm{OB}(r,n)$ is
$$T_Y \mathrm{OB}(r,n)\triangleq \{U\in \C^{n\times r}| \Re\{\ddiag(UY^H)\}=\Zero\}.$$
For any $U\in \C^{n\times r}$, we define the projection operator $\mathrm{proj}_Y:\C^{n\times r}\rightarrow T_Y \mathrm{OB}(r,n)$ to be
\begin{equation}
\nonumber
\mathrm{proj}_YU \triangleq U - \Re\{\ddiag(UY^H)\}Y.
\end{equation}7
For a smooth function $F: \mathrm{OB}(r,n)\rightarrow \R$, it is by definition that the Riemannian gradient for $F(Y)$  is
\begin{equation}
\nonumber
\mathrm{grad}\, F(Y)\triangleq \nabla F(Y)-\Re\{\ddiag(\nabla F(Y)Y^H)\}Y.
\end{equation}
Since we have assumed that the complex oblique manifold is a Riemannian submanifold of Euclidean space, then for any tangent vector $\dot{Y}\in T_Y \mathrm{OB}(r,n)$, the Riemannian Hessian for $F(Y)$ can be computed as
\begin{equation}
\nonumber
\mathrm{Hess}\, F(Y)[\dot{Y}]\triangleq \mathrm{proj}_Y\, \D\mathrm{grad}\, F(Y)[\dot{Y}], 
\end{equation}
in which $\mathrm{D}\,\mathrm{grad}\, F(\rot)[\dot{\rot}]$ is the direction derivative of $\mathrm{grad}\, F(\rot)$ along direction $\dot{\rot}$.
\bibliographystyle{IEEEtran}
\bibliography{mybib}

\begin{thebibliography}{10}
\providecommand{\url}[1]{#1}
\csname url@samestyle\endcsname
\providecommand{\newblock}{\relax}
\providecommand{\bibinfo}[2]{#2}
\providecommand{\BIBentrySTDinterwordspacing}{\spaceskip=0pt\relax}
\providecommand{\BIBentryALTinterwordstretchfactor}{4}
\providecommand{\BIBentryALTinterwordspacing}{\spaceskip=\fontdimen2\font plus
\BIBentryALTinterwordstretchfactor\fontdimen3\font minus
  \fontdimen4\font\relax}
\providecommand{\BIBforeignlanguage}[2]{{%
\expandafter\ifx\csname l@#1\endcsname\relax
\typeout{** WARNING: IEEEtran.bst: No hyphenation pattern has been}%
\typeout{** loaded for the language `#1'. Using the pattern for}%
\typeout{** the default language instead.}%
\else
\language=\csname l@#1\endcsname
\fi
#2}}
\providecommand{\BIBdecl}{\relax}
\BIBdecl

\bibitem{Cadena16tro-SLAMsurvey}
C.~Cadena, L.~Carlone, H.~Carrillo, Y.~Latif, D.~Scaramuzza, J.~Neira, I.~Reid,
  and J.~J. Leonard, ``Past, present, and future of simultaneous localization
  and mapping: Toward the robust-perception age,'' \emph{IEEE Transactions on
  Robotics}, 2016.

\bibitem{thrun2005probabilistic}
S.~Thrun, W.~Burgard, and D.~Fox, \emph{Probabilistic Robotics}.\hskip 1em plus
  0.5em minus 0.4em\relax MIT press, 2005.

\bibitem{taketomi2017visual}
T.~Taketomi, H.~Uchiyama, and S.~Ikeda, ``Visual {SLAM} algorithms: {A} survey
  from 2010 to 2016,'' \emph{IPSJ Transactions on Computer Vision and
  Applications}, vol.~9, no.~1, p.~16, 2017.

\bibitem{geiger2012we}
A.~Geiger, P.~Lenz, and R.~Urtasun, ``Are we ready for autonomous driving? {The
  KITTI} vision benchmark suite,'' in \emph{2012 IEEE Conference on Computer
  Vision and Pattern Recognition}.\hskip 1em plus 0.5em minus 0.4em\relax IEEE,
  2012, pp. 3354--3361.

\bibitem{kleiner2006rfid}
A.~Kleiner, J.~Prediger, and B.~Nebel, ``Rfid technology-based exploration and
  slam for search and rescue,'' in \emph{2006 IEEE/RSJ International Conference
  on Intelligent Robots and Systems}.\hskip 1em plus 0.5em minus 0.4em\relax
  IEEE, 2006, pp. 4054--4059.

\bibitem{kinsey2006survey}
J.~C. Kinsey, R.~M. Eustice, and L.~L. Whitcomb, ``A survey of underwater
  vehicle navigation: Recent advances and new challenges,'' in \emph{IFAC
  Conference of Manoeuvering and Control of Marine Craft}, vol.~88, 2006, pp.
  1--12.

\bibitem{dong2017agriculture}
J.~{Dong}, J.~G. {Burnham}, B.~{Boots}, G.~{Rains}, and F.~{Dellaert}, ``4d
  crop monitoring: Spatio-temporal reconstruction for agriculture,'' in
  \emph{2017 IEEE International Conference on Robotics and Automation (ICRA)},
  2017.

\bibitem{vysotska2016exploiting}
O.~Vysotska and C.~Stachniss, ``Exploiting building information from publicly
  available maps in graph-based {SLAM},'' in \emph{2016 IEEE/RSJ International
  Conference on Intelligent Robots and Systems (IROS)}.\hskip 1em plus 0.5em
  minus 0.4em\relax IEEE, 2016, pp. 4511--4516.

\bibitem{polvi2016slidar}
J.~Polvi, T.~Taketomi, G.~Yamamoto, A.~Dey, C.~Sandor, and H.~Kato, ``{SlidAR}:
  {A} 3d positioning method for {SLAM}-based handheld augmented reality,''
  \emph{Computers \& Graphics}, vol.~55, pp. 33--43, 2016.

\bibitem{grisetti2010tutorial}
G.~Grisetti, R.~Kummerle, C.~Stachniss, and W.~Burgard, ``A tutorial on
  graph-based {SLAM},'' \emph{IEEE Intelligent Transportation Systems
  Magazine}, 2010.

\bibitem{bertsekas1999nonlinear}
D.~P. Bertsekas, \emph{Nonlinear Programming}.\hskip 1em plus 0.5em minus
  0.4em\relax Athena Scientific, 1999.

\bibitem{lu1997globally}
F.~Lu and E.~Milios, ``Globally consistent range scan alignment for environment
  mapping,'' \emph{Autonomous robots}, vol.~4, no.~4, pp. 333--349, 1997.

\bibitem{duckett2002fast}
T.~Duckett, S.~Marsland, and J.~Shapiro, ``Fast, on-line learning of globally
  consistent maps,'' \emph{Autonomous Robots}, vol.~12, no.~3, pp. 287--300,
  2002.

\bibitem{frese2005multilevel}
U.~Frese, P.~Larsson, and T.~Duckett, ``A multilevel relaxation algorithm for
  simultaneous localization and mapping,'' \emph{IEEE Transactions on
  Robotics}, vol.~21, no.~2, pp. 196--207, 2005.

\bibitem{olson2006fast}
E.~Olson, J.~Leonard, and S.~Teller, ``Fast iterative alignment of pose graphs
  with poor initial estimates,'' in \emph{Proceedings 2006 IEEE International
  Conference on Robotics and Automation, 2006. ICRA 2006.}\hskip 1em plus 0.5em
  minus 0.4em\relax IEEE, 2006, pp. 2262--2269.

\bibitem{grisetti2007tree}
G.~Grisetti, C.~Stachniss, S.~Grzonka, and W.~Burgard, ``A tree
  parameterization for efficiently computing maximum likelihood maps using
  gradient descent.'' in \emph{Robotics: Science and Systems}, vol.~3, 2007,
  p.~9.

\bibitem{fan2019proximal}
T.~Fan and T.~Murphey, ``Generalized proximal methods for pose graph
  optimization,'' in \emph{International Symposium on Robotics Research
  (ISRR)}, 2019.

\bibitem{dellaert2006square}
F.~Dellaert and M.~Kaess, ``Square root sam: Simultaneous localization and
  mapping via square root information smoothing,'' \emph{The International
  Journal of Robotics Research}, vol.~25, no.~12, pp. 1181--1203, 2006.

\bibitem{kaess2008isam}
M.~Kaess, A.~Ranganathan, and F.~Dellaert, ``isam: Incremental smoothing and
  mapping,'' \emph{IEEE Transactions on Robotics}, vol.~24, no.~6, pp.
  1365--1378, 2008.

\bibitem{kaess2012isam2}
M.~Kaess, H.~Johannsson, R.~Roberts, V.~Ila, J.~J. Leonard, and F.~Dellaert,
  ``{iSAM2}: {Incremental} smoothing and mapping using the bayes tree,''
  \emph{The International Journal of Robotics Research}.

\bibitem{rosen2014rise}
D.~M. Rosen, M.~Kaess, and J.~J. Leonard, ``Rise: An incremental trust-region
  method for robust online sparse least-squares estimation,'' \emph{IEEE
  Transactions on Robotics}, 2014.

\bibitem{huang2010far}
S.~Huang, Y.~Lai, U.~Frese, and G.~Dissanayake, ``How far is slam from a linear
  least squares problem?'' in \emph{2010 IEEE/RSJ International Conference on
  Intelligent Robots and Systems}.\hskip 1em plus 0.5em minus 0.4em\relax IEEE,
  2010, pp. 3011--3016.

\bibitem{wang2013structure}
H.~Wang, G.~Hu, S.~Huang, and G.~Dissanayake, ``On the structure of
  nonlinearities in pose graph slam,'' \emph{Robotics: Science and Systems
  VIII}, p. 425, 2013.

\bibitem{kummerle2011g}
R.~K{\"u}mmerle, G.~Grisetti, H.~Strasdat, K.~Konolige, and W.~Burgard,
  ``$\mathrm{g}^2\mathrm{o}$: {A} general framework for graph optimization,''
  in \emph{2011 IEEE International Conference on Robotics and Automation}.

\bibitem{carlone2014angular}
L.~Carlone and A.~Censi, ``From angular manifolds to the integer lattice:
  Guaranteed orientation estimation with application to pose graph
  optimization,'' \emph{IEEE Transactions on Robotics}, vol.~30, no.~2, pp.
  475--492, 2014.

\bibitem{carlone2014fast}
L.~Carlone, R.~Aragues, J.~A. Castellanos, and B.~Bona, ``A fast and accurate
  approximation for planar pose graph optimization,'' \emph{The International
  Journal of Robotics Research}, vol.~33, no.~7, pp. 965--987, 2014.

\bibitem{khosoussi2015exploiting}
K.~Khosoussi, S.~Huang, and G.~Dissanyake, ``Exploiting the separable structure
  of slam,'' in \emph{Robotics: Science and systems}, 2015.

\bibitem{liu2012convex}
M.~Liu, S.~Huang, G.~Dissanayake, and H.~Wang, ``A convex optimization based
  approach for pose slam problems,'' in \emph{2012 IEEE/RSJ International
  Conference on Intelligent Robots and Systems}, 2012.

\bibitem{carlone2015duality}
L.~Carlone and F.~Dellaert, ``Duality-based verification techniques for {2D
  SLAM},'' in \emph{2015 IEEE International Conference on Robotics and
  Automation (ICRA)}, 2015.

\bibitem{carlone2016planar}
L.~Carlone, G.~C. Calafiore, C.~Tommolillo, and F.~Dellaert, ``Planar pose
  graph optimization: Duality, optimal solutions, and verification,''
  \emph{IEEE Transactions on Robotics}, 2016.

\bibitem{carlone2015lagrangian}
L.~Carlone, D.~M. Rosen, G.~Calafiore, J.~J. Leonard, and F.~Dellaert,
  ``Lagrangian duality in {3D SLAM}: Verification techniques and optimal
  solutions,'' in \emph{IEEE/RSJ International Conference on Intelligent Robots
  and Systems (IROS)}, 2015.

\bibitem{briales2016fast}
J.~Briales and J.~Gonzalez-Jimenez, ``Fast global optimality verification in 3d
  slam,'' in \emph{IEEE/RSJ International Conference on Intelligent Robots and
  Systems (IROS)}, 2016.

\bibitem{rosen2016se}
D.~M. Rosen, L.~Carlone, A.~S. Bandeira, and J.~J. Leonard, ``{SE-Sync}: {A
  certifiably correct algorithm for synchronization over the special Euclidean
  group},'' \emph{The International Journal of Robotics Research}, 2019.

\bibitem{briales2017cartan}
J.~Briales and J.~Gonzalez-Jimenez, ``Cartan-sync: Fast and global
  {SE}(d)-synchronization,'' \emph{IEEE Robotics and Automation Letters}, 2017.

\bibitem{mangelson2018guaranteed}
J.~G. Mangelson, J.~Liu, R.~M. Eustice, and R.~Vasudevan, ``Guaranteed globally
  optimal planar pose graph and landmark slam via sparse-bounded
  sums-of-squares programming,'' \emph{IEEE International Conference on
  Robotics and Automation (ICRA)}, 2019.

\bibitem{singer2011angular}
A.~Singer, ``Angular synchronization by eigenvectors and semidefinite
  programming,'' \emph{Applied and computational harmonic analysis}, vol.~30,
  no.~1, pp. 20--36, 2011.

\bibitem{singer2011three}
A.~Singer and Y.~Shkolnisky, ``Three-dimensional structure determination from
  common lines in cryo-em by eigenvectors and semidefinite programming,''
  \emph{SIAM journal on imaging sciences}, vol.~4, no.~2, pp. 543--572, 2011.

\bibitem{bandeira2017tightness}
A.~S. Bandeira, N.~Boumal, and A.~Singer, ``Tightness of the maximum likelihood
  semidefinite relaxation for angular synchronization,'' \emph{Mathematical
  Programming}, vol. 163, no. 1-2, pp. 145--167, 2017.

\bibitem{boumal2016nonconvex}
N.~Boumal, ``Nonconvex phase synchronization,'' \emph{SIAM Journal on
  Optimization}, vol.~26, no.~4, pp. 2355--2377, 2016.

\bibitem{eriksson2018rotation}
A.~Eriksson, C.~Olsson, F.~Kahl, and T.-J. Chin, ``Rotation averaging and
  strong duality,'' in \emph{Proceedings of the IEEE Conference on Computer
  Vision and Pattern Recognition}, 2018, pp. 127--135.

\bibitem{betke1997mobile}
M.~Betke and L.~Gurvits, ``Mobile robot localization using landmarks,''
  \emph{IEEE transactions on robotics and automation}, vol.~13, no.~2, pp.
  251--263, 1997.

\bibitem{bandeira2016note}
A.~S. Bandeira, ``A note on probably certifiably correct algorithms,''
  \emph{Comptes Rendus Mathematique}, vol. 354, no.~3, pp. 329--333, 2016.

\bibitem{boumal2016non}
N.~Boumal, V.~Voroninski, and A.~Bandeira, ``The non-convex {Burer-Monteiro}
  approach works on smooth semidefinite programs,'' in \emph{Advances in Neural
  Information Processing Systems}, 2016.

\bibitem{chirikjian2011stochastic}
G.~S. Chirikjian, \emph{Stochastic Models, Information Theory, and Lie Groups,
  Volume 2: Analytic Methods and Modern Applications}.\hskip 1em plus 0.5em
  minus 0.4em\relax Springer Science \& Business Media, 2011, vol.~2.

\bibitem{selig2004geometric}
J.~M. Selig, \emph{Geometric fundamentals of robotics}.\hskip 1em plus 0.5em
  minus 0.4em\relax Springer Science \& Business Media, 2004.

\bibitem{fan2019cpl_iros}
T.~Fan, H.~Wang, M.~Rubenstein, and T.~Murphey, ``Efficient and guaranteed
  planar pose graph optimization using the complex number representation,'' in
  \emph{IEEE/RSJ International Conference on Intelligent Robots and Systems
  (IROS)}, 2019.

\bibitem{dellaert2012factor}
F.~Dellaert, ``Factor graphs and {GTSAM}: A hands-on introduction,'' Georgia
  Institute of Technology, Tech. Rep., 2012.

\bibitem{khatri1977mises}
C.~Khatri and K.~Mardia, ``The von {Mises-Fisher} matrix distribution in
  orientation statistics,'' \emph{Journal of the Royal Statistical Society.
  Series B (Methodological)}, pp. 95--106, 1977.

\bibitem{carlone2013convergence}
L.~Carlone, ``A convergence analysis for pose graph optimization via
  {Gauss-Newton} methods,'' in \emph{2013 IEEE International Conference on
  Robotics and Automation}, 2013.

\bibitem{gallier2010schur}
J.~Gallier, ``The {Schur} complement and symmetric positive semidefinite (and
  definite) matrices,'' 2010.

\bibitem{absil2009optimization}
P.-A. Absil, R.~Mahony, and R.~Sepulchre, \emph{Optimization algorithms on
  matrix manifolds}.\hskip 1em plus 0.5em minus 0.4em\relax Princeton
  University Press, 2009.

\bibitem{absil2006joint}
P.-A. Absil and K.~A. Gallivan, ``Joint diagonalization on the oblique manifold
  for independent component analysis,'' in \emph{IEEE International Conference
  on Acoustics Speech and Signal Processing Proceedings}, 2006.

\bibitem{absil2007trust}
P.-A. Absil, C.~G. Baker, and K.~A. Gallivan, ``Trust-region methods on
  riemannian manifolds,'' \emph{Foundations of Computational Mathematics},
  vol.~7, no.~3, pp. 303--330, 2007.

\bibitem{boumal2015riemannian}
N.~Boumal, ``A {R}iemannian low-rank method for optimization over semidefinite
  matrices with block-diagonal constraints,'' \emph{arXiv preprint
  arXiv:1506.00575}, 2015.

\bibitem{latif2014robust}
Y.~Latif, C.~Cadena, and J.~Neira, ``Robust graph slam back-ends: A comparative
  analysis.''

\bibitem{nasiri2018recursive}
S.~M. Nasiri, R.~Hosseini, and H.~Moradi, ``A recursive least square method for
  3d pose graph optimization problem,'' \emph{arXiv preprint arXiv:1806.00281},
  2018.

\bibitem{tron2015inclusion}
R.~Tron, D.~M. Rosen, and L.~Carlone, ``On the inclusion of determinant
  constraints in {L}agrangian duality for 3d slam,'' \emph{Robot. Sci. Syst.
  Work.“The Probl. Mob. Sensors}, 2015.

\bibitem{carlone2015initialization}
L.~Carlone, R.~Tron, K.~Daniilidis, and F.~Dellaert, ``Initialization
  techniques for {3D SLAM}: a survey on rotation estimation and its use in pose
  graph optimization,'' in \emph{IEEE International Conference on Robotics and
  Automation (ICRA)}, 2015.

\bibitem{meyer2000matrix}
C.~D. Meyer, \emph{Matrix analysis and applied linear algebra}.\hskip 1em plus
  0.5em minus 0.4em\relax {SIAM}, 2000, vol.~71.

\bibitem{alizadeh1997complementarity}
F.~Alizadeh, J.-P.~A. Haeberly, and M.~L. Overton, ``Complementarity and
  nondegeneracy in semidefinite programming,'' \emph{Mathematical programming},
  1997.

\bibitem{rosen2017computational}
D.~Rosen and L.~Carlone, ``Computational enhancements for certifiably correct
  slam,'' in \emph{IEEE/RSJ International Conference on Intelligent Robots and
  Systems (IROS)}, 2017.

\end{thebibliography}

\begin{IEEEbiography}[{\includegraphics[width=1in,height=1.25in,clip,keepaspectratio]{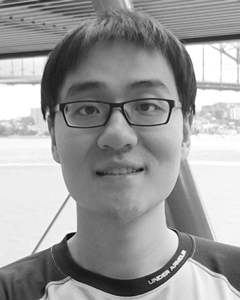}}]{Taosha Fan}
received his B.E. degree in automotive engineering from Tongji University, Shanghai, China, in 2013, and M.S. degrees in mechanical engineering and mathematics from Johns Hopkins University, Baltimore, MD, USA, in 2015. He is currently a Ph.D. candidate in mechanical engineering at Northwestern University, Evanston, IL, USA. His research interests lie at the intersection of robotic control, simulation and estimation.
\end{IEEEbiography}

\begin{IEEEbiography}[{\includegraphics[width=1in,height=1.25in,clip,keepaspectratio]{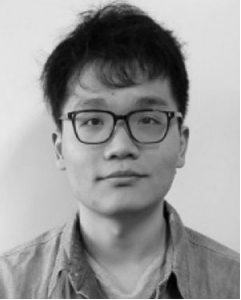}}]{Hanlin Wang}
received the B.E. degree in mechanical engineering from Shanghai Jiao Tong University, Shanghai, China, in 2015. He is currently working toward the Ph.D. degree in computer science with the Department of Computer Science, Northwestern University, Evanston, IL, USA. His research interests include swarm systems, planning and scheduling, and distributed computation systems.
\end{IEEEbiography}

\begin{IEEEbiography}[{\includegraphics[width=1in,height=1.25in,clip,keepaspectratio]{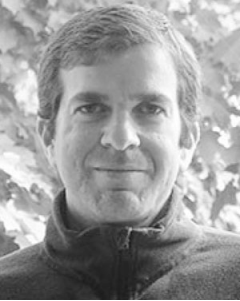}}]{Michael Rubenstein}
received the Ph.D. degree in computer science from the University of Southern California, Los Angeles, CA, USA, in 2009. He is currently an Assistant Professor working on swarm robotics and control with the Department of Computer Science as well as the Department of Mechanical Engineering, Northwestern University, Evanston, IL, USA. His research interests include robot swarms and multirobot systems.
\end{IEEEbiography}

\begin{IEEEbiography}[{\includegraphics[width=1in,height=1.25in,clip,keepaspectratio]{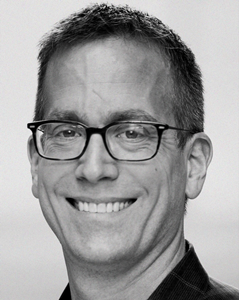}}]{Todd Murphey}
received his B.S. degree in mathematics from the University of Arizona and the Ph.D. degree in Control and Dynamical Systems from the California Institute of Technology.
He is a Professor of Mechanical Engineering at Northwestern University. His laboratory is part of the Neuroscience and Robotics Laboratory, and his research interests include robotics, control, computational methods for biomechanical systems, and computational neuroscience.
Honors include the National Science Foundation CAREER award in 2006, membership in the 2014-2015 DARPA/IDA Defense Science Study Group, and Northwestern’s Professorship of Teaching
Excellence. He was a Senior Editor of the IEEE Transactions on Robotics.
\end{IEEEbiography}
\end{document}